\documentclass[twoside]{article}

\usepackage[accepted]{aistats2021}
\usepackage{amsmath,amsfonts,bm}

\def\eqref#1{equation~\ref{#1}}
\def\1{\bm{1}}

\def\eps{{\epsilon}}

\DeclareMathAlphabet{\mathsfit}{\encodingdefault}{\sfdefault}{m}{sl}
\SetMathAlphabet{\mathsfit}{bold}{\encodingdefault}{\sfdefault}{bx}{n}

\newcommand{\E}{\mathbb{E}}

\newcommand{\R}{\mathbb{R}}

\DeclareMathOperator*{\argmax}{arg\,max}
\DeclareMathOperator*{\argmin}{arg\,min}

\usepackage{hyperref}
\usepackage{url}
\usepackage{color}
\usepackage{soul}
\usepackage{amsmath,amssymb,amsthm, mathtools}
\usepackage{subfigure}
\usepackage[final]{graphicx}
\usepackage{xcolor}
\usepackage{natbib}
\usepackage{mathrsfs}
\usepackage{algorithm, algorithmic}

\newcommand\blfootnote[1]{%
  \begingroup
  \renewcommand\thefootnote{}\footnote{#1}%
  \addtocounter{footnote}{-1}%
  \endgroup
}
\renewcommand{\cal}{\mathcal}
\newcommand\cA{{\mathcal A}}

\newcommand{\cC}{{\cal C}}
\newcommand{\cX}{{\cal X}}

\newcommand{\cD}{{\cal D}}

\newcommand{\cG}{{\cal G}}
\newcommand\cH{{\mathcal H}}

\newcommand{\cN}{{\cal N}}

\newcommand{\cP}{{\cal P}}

\newcommand{\cR}{{\mathcal R}}

\renewcommand{\cal}{\mathcal}

\newcommand{\nunlab}{\tilde{n}}
\newcommand{\intest}{\hat{\theta}_{\textup{intermediate}}}
\newcommand{\estim}{\hat{\theta}}
\newcommand{\finalest}{\estim_{\textup{final}}}
\newcommand{\labest}{\intest}
\newcommand{\semisupest}{\finalest}
\newcommand{\xun}{\tilde{x}} 

\newcommand{\yun}{\tilde{y}}

\newcommand{\defeq}{\coloneqq}

\newcommand{\bbeta}{{\bm{\beta}}}
\newcommand{\bth}{{\bm{\theta}}}
\newcommand{\supind}[1]{^{\left({#1}\right)}}

\usepackage{makecell}

\newcommand{\bE}{\mathbb{E}}

\newcommand{\bP}{\mathbb{P}}

\newcommand{\bR}{{\mathbb R}}

\renewcommand{\leq}{\leqslant}
\renewcommand{\geq}{\geqslant}
\newtheorem{theorem}{Theorem}
\newtheorem{proposition}{Proposition}
\newtheorem{lemma}{Lemma}

\newtheorem{remark}{Remark}

\newtheorem{definition}{Definition}

\DeclareMathOperator{\supp}{supp}

\newcommand{\Pro}{\mathbb{P}}
\newcommand{\bB}{{\mathbb B}}

\begin{document}



\twocolumn[

\aistatstitle{Improving Adversarial Robustness via Unlabeled Out-of-Domain Data}

\aistatsauthor{ Zhun Deng $^{*}$  \And Linjun Zhang $^{*}$ \And  Amirata Ghorbani \And James Zou }

\aistatsaddress{ Harvard University \And  Rutgers University \And Stanford University\And Stanford University } ]

\begin{abstract}
Data augmentation by incorporating cheap unlabeled data from multiple domains is a powerful way to improve prediction especially when there is limited labeled data. In this work, we investigate how adversarial robustness can be enhanced by leveraging out-of-domain unlabeled data. We demonstrate that for broad classes of distributions and classifiers, there exists a sample complexity gap between standard and robust classification. We quantify the extent to which this gap can be bridged by leveraging unlabeled samples from a shifted domain by providing both upper and lower bounds. Moreover, we show settings where we achieve better adversarial robustness when the unlabeled data come from a shifted domain rather than the same domain as the labeled data.  We also investigate how to leverage out-of-domain data when some structural information, such as sparsity, is shared between labeled and unlabeled domains.  Experimentally, we augment object recognition datasets (CIFAR-10, CINIC-10, and SVHN) with easy-to-obtain and unlabeled out-of-domain data and demonstrate substantial improvement in the model's robustness against $\ell_\infty$ adversarial attacks on the original domain. \blfootnote{$^*$ Equal contribution.}
\end{abstract}

\section{Introduction}
Robustness to adversarial attacks has been a major focus in machine learning security \citep{biggio2018wild, dalvi2004adversarial, lowd2005adversarial}, and has been intensively studied in the past few years \citep{goodfellow2014explaining,carlini2017towards,nguyen2015deep}. However, the theoretical understanding of adversarial robustness is still far from being satisfactory.  Recently \cite{schmidt2018adversarially} have demonstrated sample complexity may be one of the obstacles in achieving high robustness under standard learning, which is a large challenge since in many real-world applications, labeled examples are few and expensive. To address this challenge, recent works \citep{carmon2019unlabeled,stanforth2019labels} showed that adversarial robustness can be improved by leveraging unlabeled data that come from the same distribution/domain as the original labeled training samples. Nevertheless, that is still limited due to the difficulty to make sure that the unlabeled data are exactly from the same distribution as the labeled data. For example, gathering a large number of unlabeled images that follow the same distribution as CIFAR-10 is challenging, since one would have to carefully match the same lighting conditions, backgrounds, etc.  Meanwhile, out-of-domain unlabeled data can be much easier and cheaper to collect. For instance, we used Bing search engine to query a small number of keywords and, within hours, generated a new 500k dataset of noisy CIFAR-10 categories; we call this Cheap-10 (available at \url{ https://tinyurl.com/mere5j0x}). Despite being fast and easy to collect, we show that using Cheap-10 can substantially improve the adversarial robustness of the original CIFAR-10 classifier.  

\paragraph{Our contributions} In this paper, we investigate how such widely-available, out-of-domain unlabeled data could improve robustness in the original domain.
%\textit{How and when out-of-domain unlabeled data could improve robustness for the original domain?}
 We analyze the behavior of standard and robust classification under a flexible generative model with Gaussian seeds and a non-linear classifier class. Our model and classifier classes can be viewed as an extension of the Gaussian model and linear classifier class proposed in \cite{schmidt2018adversarially}. We show in this more general setting, the sample complexity gap between standard and robust classification still exists. That is, to achieve the same amount of accuracy, the sample complexity of robust training is significantly larger than that of standard training. We also demonstrate the necessity of this gap by providing a minimax type lower bound result.
 %{\red The necessity of that gap was shown by firstly providing the number of samples sufficient to obtain small standard and robust error respectively, where the later one is much larger. Then, a minimax type lower bound for robust learning states the least number of samples required for obtaining non-trivial robust accuracy, which nearly matches the sample size to obtain small standard error.}
 %\linjun{Zhun, does this writing look good to you?}
Luckily, we show that using unlabeled out-of-domain data can substantially improve robust accuracy as long as the unlabeled domain is not too different from the original domain or if they share some (unknown) structural information, such as similar sparse features. %Specifically, when the out-of-domain data is not too different from the labeled data from the original domain, we use the semi-supervised learning \cite{rosenberg2005semi} to obtain \textit{pseudo-labels} for unlabeled out-of-domain data, and therefore incorporate the out-of-domain data to obtain a robust classifier; when the distribution of out-of-domain data is not similar to the original domain data but share same structural information -- same unknown sparsity with same support, then we can use the out of domain data to obtain \zhun{has problem}
%We quantify the relationship between the degree of difference between the shifted and original domains and the classification performance by providing both upper and lower bounds. 
%Further, we provide a lower bound result showing that if the out-of-domain data is from a very different distribution, incorporating these data would hurt the adversarial robustness.
Interestingly, we further show settings where using out-of-domain unlabeled data can produce even better robust accuracy than using in-domain unlabeled data. 

We support our theory with experiments on three benchmark image recognition tasks, CIFAR-10, CINIC-10 and SVHN, for empirical $\ell_\infty$ robustness and certified $\ell_2$ robustness. In both CIFAR-10 and CINIC-10, adding our easily generated Cheap-10 unlabeled data produces substantially higher robust accuracy than using just the CIFAR-10 or CINIC-10 data.%, and it achieves performance close to that of adding a much more expensive and carefully curated set of $500k$ unlabeled images that come from the same domain as CIFAR-10. 
%in addition to using an existing set of $260k$ out-of-domain data, and our new Cheap-10 dataset of $500k$ images as a baseline for easy collection of unlabeled data. We show that both of these datasets increase the robustness of a trained model from the baseline and the result is very close to that of using a set of $500k$ inside-domain unlabeled images. 
 On SVHN, we systematically characterize the tradeoff between the amount of noise in the unlabeled data and the robustness gain from adding such data. 
%we synthetically push the unlabeled data out of domain and show that the resulting robustness gain from adding unlabeled data is persistent up to some level.

%===================
\paragraph{Related works.}
%==============
%\textbf{General adversarial learning.}
After the successful implementation of white-box and black-box adversarial examples~\citep{goodfellow2014explaining,biggio2018wild,moosavi2016deepfool}, several heuristic defense methods were introduced and broken one after another~\citep{zantedeschi2017efficient,athalye2018obfuscated,guo2017countering,biggio2013evasion,carlini2017adversarial,athalye2017synthesizing}. A line of work has focused on certified robustness~\citep{cohen2019certified,lecuyer2019certified,raghunathan2018certified,liu2020enhancing,chiang2020certified} which has appealing guarantees but has relatively limited empirical performance. Most recent efforts on training empirically robust models is based on adversarial training~\citep{madry2017towards, zhang2019theoretically,kurakin2016adversarial,hendrycks2019using}. 
%==============
%\textbf{Obstacles to gain robustness.}
Theoretically, some works justify the efficiency of adversarial training \citep{deng2020towards}, and to explain why it is difficult to achieve satisfactory performance in robust learning, some works try to explain the obstacles to gain robustness in a perspective of computation cost \citep{bubeck2018adversarial,degwekar2019computational}. Meanwhile, other works demonstrate how to quantify the trade-off of adversarial robustness and standard accuracy \citep{deng2020interpreting} and data augmentation such as Mixup could mitigate the trade-off \citep{zhang2020does}.  In addition, work such as \cite{schmidt2018adversarially} try to explain the obstacle by showing the sample complexity of robust learning can be significantly larger than that of standard learning. They investigated the Gaussian model, which is a special case of our Gaussian generative model. %In addition, there are works such as \cite{kumar2020understanding} trying to study self-training with gradual shift, but not for adversarial robustness.

%==============
%\textbf {Semi-supervised learning for adversarial robustness.}
Some recent works \citep{carmon2019unlabeled,stanforth2019labels} propose using semi-supervised learning method, which has a rich literature \citep{laine2016temporal,miyato2018virtual,sajjadi2016regularization}, to bridge that sample gap. Their theoretical results all assume the unlabeled data are drawn from the same marginal distribution as the labeled data. We show that to bridge the sample complexity gap, it is sufficient to have well-behaved unlabeled out-of-domain data. We substantially extend the previous results to more general models and classifier classes and also make the first step to quantify when and how unlabeled data coming from a shifted distribution can help in improving adversarial robustness. Experimentally, previous works augmented CIFAR-10 with Tiny images, which is curated and very similar to CIFAR-10. We introduce a new dataset Cheap-10 and obtain comparable results and demonstrate the power of incorporating out-of-domain data.
Other related works include \cite{zhai2019adversarially}, which demonstrates a PCA-based procedure to incorporate unlabeled data to gain robustness and \cite{najafi2019robustness}, who consider combining distributional robust optimization and semi-supervised learning. 

\section{Set-up}
Consider the classification task of mapping the input $x\in\cX\subseteq \bR^{s_1}$ to the label $y\in\{\pm 1\}$. We have $n$  labeled training data from an original domain $\cD$, with a joint distribution $\cP_{x,y}$ over $(x,y)$ pairs and marginal distribution $\cP_x$ over $x$. Meanwhile, we have another  $\tilde{n}$ unlabeled samples from a different domain $\tilde{\cD}$, with a distribution $\tilde{\cP}_x$ over $x$.

In this work, we focus on studying the possible advantages and limitations by performing semi-supervised learning with data from $\cD$ and $\tilde{\cD}$ to train a classifier for the domain $\cD$. %\james{Make this more precise.} 
Specifically, we apply the pseudo-labeling approach used in \cite{carmon2019unlabeled} as follows. First, we perform supervised learning on the labeled data from domain $\cD$ to obtain a classifier $f_0$.  We then 
apply this classifier on $\tilde{\cD}$ and generate pseudo-labels for the unlabeled data: $\{(x, f_0(x))| x \in \tilde{D}\}$, which are further used to train a final model. The classification error metrics we consider are defined as the following. 
%========
\begin{definition}[(Robust) classification error]
Let $\cP$ be a distribution over $\cX \times \{\pm 1\}$. The classification error $\beta$ of a classifier $g:\bR^{s_1}\mapsto \{\pm 1\}$ is defined as $\beta_g = \bP_{(x,y)\sim \cP} (g(x)\neq y)$ and robust classification error $\beta^\cR_g = \bP_{(x,y)\sim \cP} (\exists u \in \cC(x): g(u)\neq y)$, for some constraint set $\cC$.
\end{definition}
Throughout the paper, we consider the constraint set $\cC$ to be the $\ell_p$-ball $\bB_p(x,\varepsilon):=\{u\in\cX|\|u-x\|_p\leq\varepsilon\}$ with $p=\infty$. %We use $\|\cdot\|_\infty$ for $\ell_\infty$-norm, besides we use $\|\cdot\|_2$ or $\|\cdot\|$ exchangeably for $\ell_2$-norm. 
In addition, we consider a certain type of data generating process for domain $\cD$ ---  the Gaussian generative model, which is frequently used in generative models in machine learning. This model is more general than the one analyzed in \cite{schmidt2018adversarially}, which only considered symmetric Gaussian mixtures. Our Gaussian generative model takes a sample from a Gaussian mixture as input, and then pass it through a nonlinear (possibly high-dimensional) mapping.  %\james{Make clear diff.}
%The Gaussian seeds are frequently used in generative models in machine learning. %We also consider further extensions such as sub-gaussian models in Section \ref{subsec:further extension}.  
%=============
\paragraph{Gaussian generative model.} For a function $\rho:\bR^{s_1}\mapsto \bR^{s_2}$, given $z\in\bR^{s_1}$, the samples from $\cD$ are drawn i.i.d. from a distribution over $(x,y)\in\bR^{s_2} \times\{\pm 1\}$, such that
\begin{equation}
x=\rho(yz)
\end{equation}
where $z\sim \cN(\mu,\sigma^2I_{s_1}), ~y\sim Bern(\frac{1}{2})$ for $\mu\in \bR^{s_1}$, $\sigma\in\bR$. %We discuss more general models in Section \ref{sec:conclusion}.
\begin{remark}
The Gaussian generative model is very flexible and includes many of the recent machine learning models. For example, many common deep generative models such as VAE and GAN are Gaussian generative models: in their case, the input is a Gaussian sample $z$ and $\rho$ is parametrized by a neural network. Therefore our results  are quite generally applicable.
%\zhun{We remark here that our investigation of the Gaussian generative model is closely related to the image dataset used in practice. Recall in the Generative Adversarial Networks (GANS) \citep{goodfellow2014generative}, the new images are generated following this Gaussian generative model. Thus, comparing with previous works such as \cite{schmidt2018adversarially,carmon2019unlabeled}, our data generating process is more closed to the real data generating process of images.}
\end{remark}
%=============
\paragraph{Classifier class.} The classifier class we consider in this paper is in the following form:
\begin{equation}
 \cal G=\big\{g |g(x)=sgn\big(w^\top(\vartheta (x) - b)\big), (w,b)\in \bR^d\times \bR ^d\big\},
\end{equation}
where $\vartheta$ is a basis function and $\vartheta: \bR^{s_2}\mapsto \bR^d$. We remark here that this classifier class is more general than the linear classifier class considered in \cite{schmidt2018adversarially}. For a broad class of kernels, by Mercer's theorem, the corresponding kernel classification belongs to $\cG$ with a certain basis function $\vartheta$. Throughout the paper, we use $\theta=(w,b)$ to denote the parameters.
%=============
\begin{remark}
The Gaussian generative model and the classifier class we considered in this paper forms a hierarchy structure, where a random seed $z\in\bR^{s_1}$ is mapped by a generative function $\rho$ to the input space $x\in\bR^{s_2}$, and it is further mapped to $\bR^d$ by $\vartheta(x)$ when implementing classification.
%If we further operating $\vartheta$ on $x$ when implementing classification, we get $\vartheta(x)\in \bR^d$. In real applications, we usually have $s_1<s_2<d$. 
\end{remark}

\paragraph{Notations and terminology.} We let $\phi=\vartheta \circ \rho$ and denote $f_{w,b}(x)=sgn(w^\top(\vartheta(x)-b))$. In Section~\ref{sec:TR}, the results will be mainly described in terms of $\phi$. Besides, let $\beta(w,b) = \bP_{(x,y)\sim \cP} (f_{w,b}(x)\neq y)$ and $\beta^\cR(w,b) = \bP_{(x,y)\sim \cP} (\exists u \in \cC(x): f_{w,b}(u)\neq y)$ for a constraint set $\cC$. In particular, we use $\beta^{\varepsilon,\infty}$ when $\cC$ is the $\ell_\infty$-ball with radius $\varepsilon$. 
%We use $\lesssim$ for $\leq$ up to a constant, i.e. $a\lesssim b$ means $a\leq const\cdot b$, and use $\gtrsim$ for $\geq$ up to a constant. We use $\asymp$ for $=$ up to a constant. 
Meanwhile, we use $\|\cdot\|_{\psi_2}$ for sub-gaussian norm\footnote{Due to the limit of space, we present the rigorous definition of the sub-gaussian norm in the {appendix}.}. We call the conditional distribution of $x$ on $y=1$ as positive distribution while for $y=-1$ as negative distribution. For distribution $\cP_1$ and $\cP_2$ over $x$, we call a distribution $\cP$ over $x$ is a \textit{uniform mixture} of $\cP_1$ and $\cP_2$ if it equals to $\cP_1$ and $\cP_2$ with probability $1/2$ respectively. For a sequence of random variables $\{X_n\}$ and a sequences of positive numbers $\{a_n\}$, we write $X_n=O_{\bP}(a_n)$ if there exists a constant $C$, such that $\bP(X_n\le Ca_n)\to 1$ when $n\to\infty$. 
For real-valued sequences $\{a_n\}$ and $\{b_n\}$, we write $a_n \lesssim b_n$ if $a_n \leq cb_n$ for some universal constant $c \in (0, \infty)$, and $a_n \gtrsim b_n$ if $a_n \geq c'b_n$ for some universal constant $c' \in (0, \infty)$. We say $a_n \asymp b_n$ if $a_n \lesssim b_n$ and $a_n \gtrsim b_n$. In this paper, $c, C, c_0, c_1, c_2, \cdots, $ refer to universal constants, and their specific values may vary from place to place.

%==========================================================
\section{Theoretical Results}\label{sec:TR}
We demonstrate for Gaussian generative models that combining unlabeled data from a reasonably well-behaved shifted domain leads to a classifier with better robust accuracy on the original domain $\cD$ compared to the achievable robust accuracy using only the labeled data from $\cD$. We further analyze the tradeoff between how different the shifted domain can be from $\cD$ before the unlabeled data hurts the robust accuracy on $\cD$.  Finally, we show that if the data from a shifted domain share certain unknown sparsity structure with the data from original domain, performing semi-supervised learning also helps in obtaining a classifier of higher robust accuracy on the original domain. 

%=============
\textbf{Assumptions.} Throughout this section, our theories are based on the following assumptions unless we state otherwise explicitly. 1). $\vartheta(\cdot)$ is $L_1$-Lipchitz continuous in $\ell_2$-norm, i.e. $\|\vartheta(a)-\vartheta(b)\|\leq L_1\|a-b\|$, and $L_1'$-Lipchitz continuous in $\ell_\infty$-norm; 2). $\rho(\cdot)$ is $L_2$-Lipchitz continuous in $\ell_2$-norm and $L_2'$-Lipchitz continuous in $\ell_\infty$-norm; 3). $\|\bE\phi(z)-\bE\phi(-z)\|=2\alpha\sqrt{d}$ for $z\sim \cN(\mu,\sigma^2I_{s_1})$ and some constant $\alpha>0$. The last condition on the magnitude of the separation is added for the simplicity of presentation. Such a magnitude choice is also used in \cite{schmidt2018adversarially,carmon2019unlabeled}. % \james{Add an explanation/motivation for 3.}
%\linjun{I change the notation $L_3$ to $L_2'$. Do we use the original $L_3$ in this main text?}

%======================================
\subsection{Supervised learning in Gaussian generative models}\label{subsec:sg}
We first consider the supervised setting where only the labeled data are used. In this setting, we prove the following two theorems demonstrating the   sample complexity gap when one considers standard error and robust error respectively. {Analogous results for Gaussian mixture models was shown in \cite{schmidt2018adversarially}; our results cover the more general Gaussian generative model setting. }

%phenomenon first brought up in \cite{schmidt2018adversarially} also holds for much broader classes of data models and classifiers in high dimensional settings where $d$ is large.  \james{Clarify difference from \cite{schmidt2018adversarially}.}

%==========
\textit{Supervised learning algorithm:} in this section, for the simplicity of presentation, we use $2n$ to denote the size of labeled training data.
For Gaussian generative models, we focus on the following method. We first estimate $w$ and $b$ by $\hat{w}=1/n\sum_{i=1}^n y_i\vartheta(x_i)$ and $\hat{b}=1/n\sum_{i=n+1}^{2n}\vartheta(x_i)$. The final classifier is then constructed as $f_{\hat w,\hat b}(x)=sgn(\hat w^\top(\vartheta(x)-\hat b))$. Here half of the labeled data, $n$, is used to fit $\hat{w}$ and the other half used to fit $\hat{b}$, so that their estimation errors are independent, which simplifies the analysis.  The following theorem shows that this method achieves high standard accuracy. 
%==========
\begin{theorem}[Standard accuracy]\label{thm:standardaccuracy}
 For a Gaussian generative model with $\sigma\lesssim  d^{1/4}$, the method described above obtains a classifier $f_{\hat w,\hat b}$ such that for $d$ sufficiently large, with high probability,  the classification error $\beta(\hat{w},\hat{b})$ is at most $1\%$ even with $n=1$.
\end{theorem}

Meanwhile, we have the following lower bound to show the essentiality of the increased sample complexity if we are interested in the robust error.
%==========
\begin{theorem}[Sample complexity gap for robust accuracy]\label{thm:lowerbound}
Let $\cA_n$ be any learning algorithm, i.e. a function from $n$ samples to a binary classifier $g_n$. Let $\sigma \asymp d^{1/4}$, $\varepsilon\geq 0$, and $\mu\in\bR^{s_1}$ be drawn from a prior distribution $\cN(0,I_{s_1})$. We draw $2n$ samples from $(\mu,\sigma)$-Gaussian generative model. Then, the expected robust classification error $\beta^{\varepsilon,\infty}_{g_n}$ is at least $(1-1/d)/2$ if
$$n\lesssim\frac{\varepsilon^2\sqrt{d}}{\log d}.$$
\end{theorem}

Taken together, these two Theorems demonstrate that a substantial larger number of labeled samples (from the same domain) are necessary in order to achieve a decent robust accuracy in that domain.

%======================================
%\vspace{-0.1cm}
\subsection{Improving learning via out-of-domain data}\label{subsec:improve}
We next investigate how to improve the  robust accuracy of a classifier via incorporating unlabeled out-of-domain data. 
%==========
\paragraph{Semi-supervised learning on out-of-domain data.}
Let us denote the samples from the shifted domain as $\{\tilde{x}_i\}_{i=1}^{2\tilde{n}}$, which is incorporated via the following semi-supervised learning algorithm.%, where $\tilde{x}_i\in\bR^{s_2}$. 

\textit{Semi-supervised learning algorithm:} we use $\hat{w}$ and $\hat{b}$ obtained in supervised learning to label $\{\tilde{x}_i\}_{i=1}^{2\tilde{n}}$ via $\hat{g}(x)=sgn\big(\hat{w}^\top(\vartheta (x) - \hat{b}))$ and obtain the corresponding pseudo-labels $\{\tilde{y}_i\}_{i=1}^{2\tilde{n}}$. We denote sample sizes for each label class by $\tilde{n}_1=\sum_{i=1}^{2\tilde n} 1(\tilde y_i=1)$ and $\tilde n_2=\sum_{i=1}^{2\tilde n} 1(\tilde y_i=-1)$ respectively. Then we estimate $w$ and $b$ respectively by 
\begin{align*}
 \tilde{w}=\frac{1}{2\tilde{n}_1}\sum_{\tilde y_i=1}\vartheta(\tilde x_i)-\frac{1}{2\tilde{n}_2}\sum_{\tilde y_i=-1}^{\tilde{n}}\vartheta(\tilde{x}_i)\\
 \tilde{b} =\frac{1}{2\tilde{n}_1}\sum_{\tilde y_i=1}\vartheta(\tilde x_i)+\frac{1}{2\tilde{n}_2}\sum_{\tilde y_i=-1}\vartheta(\tilde{x}_i).
\end{align*}
Given the pseudo-labels, these two estimators only depend on the shifted domain data. They are slightly different than those in the supervised setting, since the shifted domain data is not necessarily mixed uniformly. The classifier is then constructed as $f_{\tilde w,\tilde b}(x)=sgn(\tilde w^\top(\vartheta(x)-\tilde b))$. %\james{Seems like typos in the last expression. Also good to explain this more. So the final model fitting using on the pseudo-labeled shifted domain data and not the original labeled data? } We additionally note here that
For the simplicity of theoretical analysis, we don't merge the original and out-of-domain datasets to get $\tilde w$ and $\tilde b$. However, as we show in Section~\ref{sec:exp}, merging both datasets for robust training lead to better empirical performance. 

Recall $\phi=\vartheta\circ\rho$, and the semi-supervised learning algorithm only involves $y$ and $\vartheta(x)$, we can equivalently view the input distribution as $\phi(yz)$ for $z\sim\cN(0,I_{s_1})$, $y\sim Bern(1/2)$, and the classifier class as $\cal G'=\big\{g |g(x)=sgn\big(w^\top x- b)\big), (w,b)\in \bR^d\times \bR ^d\big\}$ (such linearization is the common purpose of kernel tricks). For the simplicity of description, our later statements will use this equivalent setting and simply consider the distributions of $\vartheta(\tilde{x})$.
%consider the distribution of the out-of-domain data after feeding to $\vartheta$ (i.e. we consider distributions of $\vartheta(\tilde{x})$).} 
%\james{Not clear what this paragraph is saying.}

%==========
\begin{theorem}[Robust accuracy]\label{thm:robustaccuracy}
 Recall in Gaussian generative model, the marginal distribution of the input $x$ of labeled domain is a uniform mixture of two distributions with mean $\mu_{1}=\E[\phi(z)]$ and  $\mu_{2}=\E[\phi(-z)]$ respectively, where $z\sim\cN(0,\sigma^2I_{s_1})$. Suppose the marginal distribution of the input of unlabeled domain is a mixture of  two sub-gaussian distributions with mean $\tilde\mu_{1}$ and  $\tilde\mu_{2}$ with mixing probabilities $q$ and $1-q$ and $\left\|\E\left[\vartheta(\tilde x_i)-\E[\vartheta(\tilde x_i)]\mid a^T \vartheta(\tilde x_i)=b\right]\right\|\lesssim \sqrt d + |b|$ for fixed unit vector $a$.  Assuming the sub-gaussian norm for both labeled and unlabeled data are upper bounded by a universal quantity $\sigma_{\max}\asymp d^{1/4}$, %\james{explain what is $\sigma_{\max}$}
 $\|\tilde\mu_1-\tilde\mu_2\|_2\asymp \sqrt d$,  $c<q<1-c$ for some constant $0<c<1/2$, and 
%$$d_\nu=\max\Big\{\frac{|(\tilde\mu_1-\tilde\mu_2)^\top(\tilde\mu_1-\mu_1)|}{\|\tilde\mu_1-\tilde\mu_2\|^2},\frac{|(\tilde\mu_1-\tilde\mu_2)^\top(\tilde\mu_2-\mu_2)|}{\|\tilde\mu_1-\tilde\mu_2\|^2}\Big\}<\frac{1}{4},$$
$$
d_\nu=\max\Big\{\frac{\|\tilde\mu_1-\mu_1\|}{\|\tilde\mu_1-\tilde\mu_2\|},\frac{\|\tilde\mu_2-\mu_2\|}{\|\tilde\mu_1-\tilde\mu_2\|}\Big\}<c_0,$$
for some constant $c_0\le 1/4$,
then the robust classification error is at most $1\%$ when $d$ is suffciently large, $n\ge C$ for some constant C (not depending on $d$ and $\epsilon$) and
$$\tilde n\gtrsim \varepsilon^2\log d\sqrt d.$$
\end{theorem}
\begin{remark}
We remark here that $\sigma_{\max}$, the upper bound of the sub-gaussian norm of the Gaussian generative model, is upper bounded by $L_1L_2\sigma$. Comparing to the Theorem~\ref{thm:lowerbound}, which shows that the sample complexity of order ${\varepsilon^2\sqrt{d}}/{\log d}$ is necessary to achieve small robust error,
the above theorem shows that, by incorporating the same order of similar unlabeled data (up to a logarithm factor), which is generally cheaper, one can achieve the same robust accuracy. We further note that the sub-Gaussian assumption in Theorem \ref{thm:robustaccuracy} is quite relaxed. For example, any dataset where the feature values are bounded are automatically sub-Gaussian. This includes all image data since the pixel values are bounded. 
\end{remark}
\begin{remark}

Moreover, by using the same technique, we can extend our theoretical results to a much more general family of distributions in $\R^d$ whose tails are bounded by any strictly decreasing function $g$. Define \begin{align*}
D_g(\mu,\sigma^2)=\{& X\in\R^d: \forall v\in\R^d, \|v\|_2=1, \text{Var}(X_j)\le\sigma^2\\
 &\Pro(|v^T(X-\mu)|>\sigma\cdot t)\le g(t) \}.
    \end{align*} For example letting $g(t)=C_1 e^{-C_2t^2}$ reduces $D_g$ to the family of sub-Gaussian distributions. The in-domain distribution is now assumed to be $x_1,...,x_n\stackrel{i.i.d.}{\sim}{1}/{2}\cdot D_g(\mu_1,\sigma^2)+{1}/{2}\cdot D_g(\mu_2,\sigma^2)$ and the out-of-domain distribution is assumed to be $\tilde x_1,...,\tilde x_{\tilde n}\stackrel{i.i.d.}{\sim}q\cdot D_g(\tilde\mu_1,\tilde\sigma^2)+(1-q)\cdot D_g(\tilde\mu_2,\tilde\sigma^2)$ where $q\in(0,1/2)$. We remark here that this extension allows the in-domain and out-of-domain distributions to be very different as long as they are all in the family $D_g(\cdot,\cdot)$. 
    We present the following proposition for this extension.
    \begin{proposition}
    Under the similar assumptions to those in Theorem 3.3, that is, $\left\|\E\left[\tilde x_i-\E[\tilde x_i]\mid a^T \tilde x_i=b\right]\right\|\lesssim \sqrt d + |b|$ for fixed unit vector $a$, $\tilde\sigma\le\sigma_{\max}\asymp d^{1/4}$, %\james{explain what is $\sigma_{\max}$}
 $\|\tilde\mu_1-\tilde\mu_2\|_2\asymp \sqrt d$,  $c<q<1-c$ for some constant $0<c<1/2$, and 
%$$d_\nu=\max\Big\{\frac{|(\tilde\mu_1-\tilde\mu_2)^\top(\tilde\mu_1-\mu_1)|}{\|\tilde\mu_1-\tilde\mu_2\|^2},\frac{|(\tilde\mu_1-\tilde\mu_2)^\top(\tilde\mu_2-\mu_2)|}{\|\tilde\mu_1-\tilde\mu_2\|^2}\Big\}<\frac{1}{4},$$
$$d_\nu=\max\Big\{\frac{\|\tilde\mu_1-\mu_1\|}{\|\tilde\mu_1-\tilde\mu_2\|},\frac{\|\tilde\mu_2-\mu_2\|}{\|\tilde\mu_1-\tilde\mu_2\|}\Big\}<c_0,$$
for some constant $c_0\le 1/4$,
    then    the robust classification error is at most $1\%$ when $d$ is sufficiently large, $n\ge C$ for some constant C (not depending on $d$ and $\varepsilon$) and $$\tilde n\gtrsim \varepsilon^2\cdot(g^{-1}(1/d\log d))^2\cdot\sqrt d.$$
       \end{proposition}
\end{remark}

%as in the supervised setting.  \james{Not clear. Make comparison with Theorem~\ref{thm:lowerbound} more explicit.}%\zhun{add interpretation emphasis only $\log d$ factor to lower bound, which matches, make notation consistent, interpret $d_\nu$}

%==========================
\paragraph{Connections to statistical measures.} 
A key quantity in Theorem~\ref{thm:robustaccuracy} is $d_\nu$, which quantifies the difference between the labeled and unlabeled domain. In this section, we make connections between some commonly used statistical measures and $d_\nu$ via some more specific examples.We establish connections to Wasserstein Distance, Maximal Information and $\cH$-Divergence. Due to the limit of space, we only demonstrate the result for Wasserstein Distance here and put the other results in the appendix. Throughout this paragraph, we consider the distribution of the labeled domain with positive distribution $\cP_1$, negative distribution $\cP_2$, and $y\sim Bern(1/2)$. The marginal distributions of shifted domain is assumed to be a uniform mixtures of $\tilde{\cP}_1$ and $\tilde{\cP}_2$. 

%(a). \textit{Maximal Information}: Maximal Information between distributions $\cP_1$ and $\cP_2$ over $\bR^d$ is defined as 
%$$MI(\cP_1,\cP_2)=\sup_{O\subseteq \bR^d}\left\{\frac{\bP_{x\sim\cP_1}(x\in O)}{\bP_{x\sim\cP_2}(x\in O)}, \frac{\bP_{x\sim\cP_2}(x\in O)}{\bP_{x\sim\cP_1}(x\in O)}\right\}.$$
%============
%\begin{proposition}
%Suppose $\max\{MI(\cP_1,\tilde{\cP}_1),MI(\cP_2,\tilde{\cP}_2)\}\leq \tau$
%for $1\leq\tau\leq 1+\|\mu_1-\mu_2\|/(2\|\mu_1\|+2\|\mu_2\|)$, then we have $\|\mu_i-\tilde{\mu}_i\|\leq(\tau-1)\|\mu_i\|,~ i=1,2.$
%As a result, we have
%$$d_\nu\leq\frac{(\tau-1)\max\{ \|\mu_1\|, \|\mu_2\|\}}{\|\mu_1-\mu_2\|-2(\tau-1)(\|\mu_1\|+\|\mu_2\|)}.$$
%\end{proposition}
%As we can see, as $\tau\rightarrow 1$, $d_\nu\rightarrow 0$.

\textit{Wasserstein Distance}: the Wasserstein Distance induced by metric $\rho$ between distributions $\cP_1$ and $\cP_2$ over $\bR^d$ is defined as $$W_\rho(\cP_1,\cP_2)=\sup_{\|f\|_{\text{Lip}}\leq 1}\left[\int fd\cP_1-fd\cP_2\right],$$
where $\|f\|_{\text{Lip}}\leq 1$ indicates the class of $f:\bR^d\mapsto\bR$ such that for any $x,x'\in \bR^d$, $|f(x)-f(x)|\leq \rho(x,x').$ Let us consider $\rho(x,x')=\|x-x'\|.$ 
%============
\begin{proposition}
Under the assumption that $\max\{W_\rho(\cP_1,\tilde{\cP}_1),W_\rho(\cP_2,\tilde{\cP}_2)\}\leq \tau$, for $ \tau\geq 0$, then we have $\|\mu_i-\tilde{\mu_i}\|\leq\tau$ for $i=1,2.$ 
As a result, 
$$d_\nu\leq\frac{\tau}{\|\tilde{\mu}_1-\tilde{\mu}_2\|}.$$
If we further have $ \tau\leq \|\mu_1-\mu_2\|/2$, we will then have $d_\nu\leq\tau/(\|\mu_1-\mu_2\|-2\tau).$
\end{proposition}

%(b). $\cH$-Divergence: let $\cH$ be a class of binary classifiers, the $\cH$-divergence between distributions $\cP$ and $\cP'$ over $\bR^d$ is defined as 
%$$D_\cH(\cP,\cP')=\sup_{h\in \cH}|\bP_{x\sim\cP}(h(x)=1)-\bP_{x\sim\cP'}(h(x)=1)|.$$
%To illustrate the connection between  $d_\nu$ and $\cH$-divergence, we consider a specific hypothesis class
%\begin{equation}\label{eq:hypothesis}
%\cal H=\big\{h |h(t)=sgn(w^\top(t- b)), (w,b)\in \bR^d\times \bR^d \big\}.
%\end{equation}
%============
%\begin{proposition}
% Suppose for $X_i\sim\cP_i$ and $\tilde{X}_i\sim\tilde{\cP}_i$ $i=1,2$, the sub-gaussian norm of $\|X_i-\mu_i\|_{\psi_2}$ and $\|\tilde{X}_i-\tilde{\mu}_i\|_{\psi_2}$ are bounded by $\sigma$ and $\tilde{\sigma}$, where $X_i\sim\cP_i$, $\tilde{X}_i\sim\tilde{\cP}_i$ and $\mu_i$, $\tilde{\mu}_i$ are the corresponding means. Let $\alpha(\tau)=\zeta\sqrt{\log(4/(1-\tau))}$, where $\zeta=\max\{\sigma,\tilde{\sigma}\}$, if $\max\{D_{\cH}(\cP_1,\tilde{\cP}_1),D_{\cH}(\cP_2,\tilde{\cP}_2)\}\leq \tau,$ for $\tau\leq 1$, we have $\|\mu_i-\tilde{\mu}_i\|\leq \alpha$, $i=1,2$. As a result,
 %$$d_\nu\leq \frac{\alpha(\tau)}{\|\tilde{\mu}_1-\tilde{\mu}_2\|}. $$%\james{What happened to the dependence on $\tau$?}
% If we further have $\tau\leq 1-4\exp(-\|\mu_1-\mu_2\|^2/4\zeta^2)$, then $d_\nu\leq \alpha/(\|\mu_1-\mu_2\|-2\alpha).$
%\end{proposition}
As we can see, when the Wasserstein distance get smaller, the quantity $d_\nu$ decreases. %Although for $\cH$-Divergence, $d_\nu$ does not go to $0$ when $\tau\rightarrow 0$,  this is due to the constraint of capacity of $\cH$, which only contains the linear classifiers. Even for $\tau=0$, $\cP_i$ and $\tilde{\cP}_i$ can still be very different distributions. 
%\james{If short on space, it's fine to move the H-Divergence part to the Appendix.}

%==========================
\paragraph{Data from a shifted domain can work even better.}
Theorem \ref{thm:robustaccuracy} demonstrates the sample complexity gap in Section \ref{subsec:sg} can be bridged via out-of-domain data. Next, we  show that in certain settings,
one can achieve even better adversarial robustness when the unlabeled data comes from a shifted domain rather than the same domain as the labeled data. %where the latter setting is studied in \cite{carmon2019unlabeled,stanforth2019labels}. 
To illustrate this phenomenon,  let us analyze a specific example of our model --- the Gaussian model proposed in \cite{schmidt2018adversarially}.

%==========
% \begin{figure}[H]
%     \centering\label{enhance}
%     \includegraphics[width=0.8\linewidth]{figures/plot_enhance} 
%     \caption{Error Difference vs $\varepsilon$: we use synthetic data described in Theorem \ref{thm:enhance}, where the Error Difference = Adversarial Error when Unlabeled Data from the Same Domain - Adversarial Error when Unlabeled Data from the Shifted Domain. We can see that error difference is positive for all the $\varepsilon$ we take, which implies unlabeled data from a shifted domain works even better. In this experiment, we take $n=1000$, $\tilde{n}=300$, $d$=100. \james{Suggest removing the figure, since it looks a bit noisy, and just summarize this result in a couple of sentences.}}
% \end{figure}
\begin{theorem}\label{thm:enhance}
Suppose the distribution of the labeled domain has positive distribution $\cN(\mu,\sigma^2I_{d})$ and negative distribution $\cN(-\mu,\sigma^2I_{d})$ with $\mu\in\R^d$ and $y\sim Bern(1/2)$. Samples $(x_1,y_1),\cdots,(x_n,y_n)$ are i.i.d. drawn from the labeled domain.  Suppose we have unlabeled inputs from the same domain $ x_{n+1},\cdots, x_{n+\tilde n}$, %\james{confusing notation since $\tilde x$ is associated with shifted domain}. We
and also have unlabeled shifted domain inputs $\tilde x_1,\cdots,\tilde x_{\tilde n}$, which are drawn from a uniform mixture of $\cN(\tilde{\mu},\sigma^2I_{d})$ and $\cN(-\tilde{\mu},\sigma^2I_{d})$ with $\tilde{\mu}\in\R^d$. Denote the parameter $\theta=(w,b)$ of the classifier obtained through semi-supervised algorithm by $\hat\theta_{\text{same}}$ and  $\hat\theta_{\text{shifted}}$, when we use $\{{x}_i\}_{i=n+1}^{n+\tilde n}$ and $\{\tilde{x}_i\}$ respectively.
If we let $\mu=2\varepsilon 1_d$ and   $\tilde\mu=\varepsilon 1_d$, where $1_d$ is a $d$-dimensional vector with every entry equals to $1$, when  $(1/d+1/\tilde n)\cdot\sigma^2/\epsilon^2\to0$ as $d$, $\tilde n\to\infty$, we then have 
$$
\beta^{\varepsilon,\infty}(\hat{\theta}_{\text{shifted}})\le \beta^{\varepsilon,\infty}(\hat{\theta}_{\text{same}}).
$$
\end{theorem}
This seemingly surprising result can be explained intuitively. Heuristically, when one tries to minimize the robust error, the robust optimizer will behave similarly to a regularized version of the standard optimizer. In our semi-supervised setting, the shifted domain data also act as regularization. Such an intuition is rigorously justified in the proof.  Further, we illustrate the results in Theorem~\ref{thm:enhance} by experiments with synthetic data, the experiment set-up and results are presented in the appendix, where we find that the robust error by incorporating the out-of-domain unlabeled data is smaller than than incorporating the same amount of unlabeled data from the same domain as the labeled data. %Readers can find corresponding simulation results in the appendix.

%\james{Briefly explain the regularization intuition.}

%==========================
\paragraph{Too irrelevant unlabeled data hurts robustness.} 
In the results above, we demonstrate that incorporating unlabeled data from a shifted domain can improve robust accuracy in the original domain, if the shifted domain is not too different from the original (as measured by $d_{\nu}$). Here we show that there is no free lunch; if the shifted domain is too different from the original, then incorporating its unlabeled data through pseudo-labeling could decrease the robust accuracy in the original domain.

%In the previous paragraphs, we demonstrate the power of incorporating benign out-of-domain data in improving adversarial robustness. Another natural question is what happens if the out-of domain data is not so desirable? Such a situation indeed occurs in real applications, especially when one tries to augment images of one type with other types. In such scenarios, even visually similar images have very different distributions given the high dimensionality. The following theorem provides a theoretical view regarding what happens if we incorporate undesirable out-of-domain data by studying the case when the labeled domains follows a Gaussian model. %For lower bound, consider only Gaussian model is enough, since it is a special case of Gaussian generative model.
\begin{theorem}\label{thm:lb}
Suppose that the distribution of labeled domain's positive  and negative distribution area uniform two symmetric sub-gaussian distribution with means $\mu_{2}=-\mu_{1}$, and $y\sim Bern(1/2)$. The distribution of unlabeled domain $\cP'$ is a mixture of two sub-gaussian distributions with mean $\tilde{\mu}_{1}$ and  $\tilde{\mu}_{2}$. Let $\Xi=\{\tilde\mu_{1},\tilde\mu_{2}: d_\nu\ge\frac{1}{2}
\}$. Then for $\cP'$ with $(\tilde{\mu}_1,\tilde{\mu_2})\in \Xi$, with high probability, the worst case  robust misclassification error $\beta^{\varepsilon,\infty}(\tilde{w},\tilde{b})$ via the previous semi-supervised learning satisfies
$$
\sup_{\cP':(\tilde{\mu}_1,\tilde{\mu_2})\in  \Xi}\beta^{\varepsilon,\infty}(\tilde{w},\tilde{b})\ge 49\%.
$$
\end{theorem}

%==========================
\paragraph{Shifted domain with unknown sparsity.}
In Theorem~\ref{thm:robustaccuracy}, we show how reasonably close shifted domain data helps in improving adversarial robustness. However, sometimes, the shifted domain is not so close to the original domain in terms of $d_\nu$, but they still share some similarity. For instance, both domains can share some structural commonness. Here, we consider the case where the two distributions have common salient feature set; that is, the labeled and unlabeled domains share discriminant features, though the corresponding coefficients can be far apart. This setting is common in practice. For example, when one tries to classify images of different kinds of cats, the discriminant features include the eyes, ears, shapes etc. These discriminant features also applies when one aims to classify dogs, though the weights on these features might be very different.

Specifically, we consider the distributions of the labeled domain's positive and negative parts are $\cN(\mu_1,\sigma^2I_{d})$ and $\cN(\mu_2,\sigma^2I_{d})$, and the labels $y\sim Bern(1/2)$. The samples $(x_1,y_1),...,(x_n,y_n)$ are drawn i.i.d. from 
this labeled domain. Suppose we have unlabeled samples $\tilde x_1,...,\tilde x_{\tilde n}$, which are drawn from a uniform mixture of $\cN(\tilde{\mu},\sigma^2I_{d})$ and $\cN(-\tilde{\mu},\sigma^2I_{d})$ with $\tilde{\mu}\in\R^d$. Here, we assume the two domains share the support information, that is,  $\supp(\mu_1-\mu_2)=\supp(\tilde\mu_1-\tilde\mu_2)$, though the distance between $\mu_1-\mu_2$ and $\tilde\mu_1-\tilde\mu_2$ is not necessarily small. For such a case, we propose to use the following algorithm to help improving adversarial robustness.

\textit{Algorithm of unknown sparsity:} we first apply the high-dimensional EM algorithm \citep{cai2019chime} to estimate $\supp(\tilde\mu_1-\tilde\mu_2)$ from the unlabeled data. This high-dimensional EM algorithm is an extension of the traditional EM algorithm with the M-step being replaced by a regularized maximization. The detailed description can be found in the Appendix. After implementing the high-dimensional EM  to estimate the support $\hat S$ from the unlabeled data, we then project the labeled data to this support $\hat S$ and therefore reduce the dimension. Finally, we apply the algorithm in the supervised setting on the labeled samples with reduced dimensionality to get the estimated $\hat w_{sparse}=1/n\sum_{i=1}^n y_i[\vartheta(x_i)]_{\hat S}$ and $\hat b_{sparse}=1/n\sum_{i=n+1}^{2n}[\vartheta(x_i)]_{\hat S}$. %\james{Clarify how the unlabeled data is used.}
The following theorem provides  theoretical guarantee for the robust classification error for  $\hat{\theta}_{\text{sparse}}$ by this algorithm. %\james{Not clear what is $\hat{\theta}_{\text{sparse}}$. Is this the output of the EM?}
%$\hat{w}=1/n\sum_{i=1}^n y_i\vartheta(x_i)$ and $\hat{b}=1/n\sum_{i=n+1}^{2n}\vartheta(x_i)$.

\begin{theorem}\label{thm:sparsity}
Under the conditions of Theorem 3.3 on parameters $\mu_1, \mu_2, \tilde\mu_1,\tilde\mu_2,\sigma$ and $q$. 
Suppose $|\supp(\mu_1-\mu_2)|=|\supp(\tilde\mu_1-\tilde\mu_2)|=m$, and $\min_{\tilde{\mu}_{1,j}-\tilde{\mu}_{2j}\neq 0}|\tilde{\mu}_{1j}-\tilde{\mu}_{2,j}|\ge\sigma\sqrt{ 2m\log d/\tilde n}$, where $\tilde{\mu}_{i,j}$ is the $j$-th entry in vector $\tilde{\mu}_i$. If  $n\gtrsim \varepsilon^2\log d\sqrt m$, we have 
$$%\vspace{-0.1cm} 
 \beta^{\varepsilon,\infty}(\hat w_{sparse}, \hat b_{sparse})\le 10^{-3}+O_{\bP}(\frac{1}{ n}+\frac{1}{d}).
$$%\vspace{-0.3cm}
  \end{theorem}
Comparing to the result in Theorem~\ref{thm:lowerbound}, which shows that the sample complexity $O(\varepsilon^2\sqrt{d}/(\log d))$ is necessary to achieve small robust error, the above theorem suggests that by utilizing the shared structural information from the unlabeled domain, one can reduce the sample complexity from $O(\sqrt{d})$ to $O(\sqrt m)$. Corresponding simulation results are put in the Appendix.

%\james{Explain that we don't merge the datasets here to simplify theory. In experiments, we do.}
%\linjun{ add some sentences to address the discrepancy between the theory and simulation}

%===========================
%==========================================================

\section{Experiments}\label{sec:exp}
\begin{figure*}
    \centering
    \includegraphics[width=6.5in]{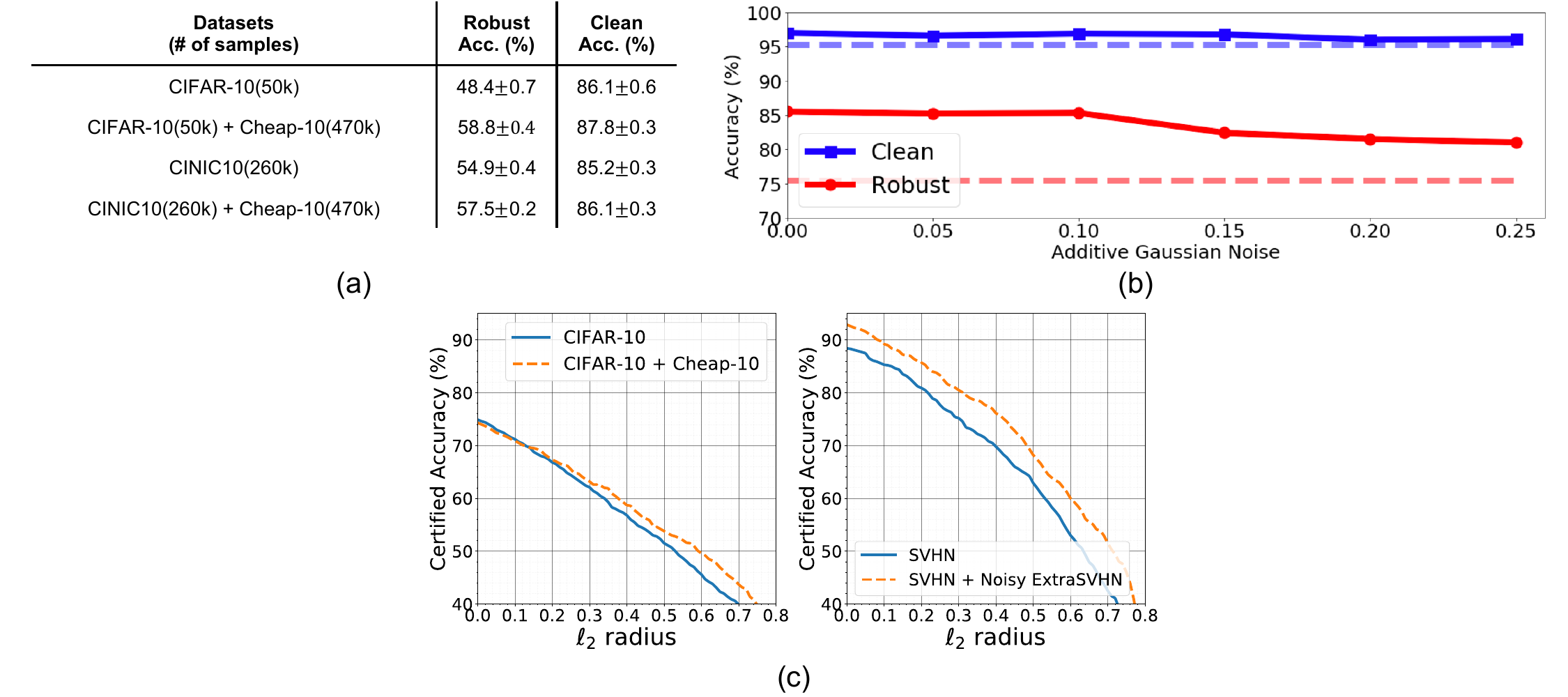} 
    \caption{(a) \textbf{$\ell_\infty$ robustness.}  Each row shows the  test accuracy on clean and adversarially perturbed images ($\epsilon = 8/255$) when the original datasets are used versus when there is additional unlabeled source of data (Cheap-10). The robustness performance when we use the out-of-domain dataset is  significantly better than the original training set, in agreement with our theory. (b) \textbf{ SVHN dataset.} Dashed lines stand for the baseline performance of using only the labeled data. Each point on the x-axis shows a different model that is robustly trained using the original data and the unlabeled set of images with additive Gaussian noise with std ($\sigma$) equal to the axis value. The dashed lines indicate the clean and robust accuracy achieved using only the SVHN data. Adding noisy data improved robust accuracy. (c) \textbf{$\ell_2$ certified robustness.}
    Each point in the plot shows the percentage of test images that are certified to be classified correctly at that $\ell_2$ radius. Adding out-of-domain datasets consistently improves the certified raidii.
    \label{fig:results}  \vspace{-0.4cm} }
\end{figure*}
In this section, we provide empirical support for our theory and show that using unlabeled data from shifted domains can consistently improve robust accuracy for three widely-used benchmark datasets: CIFAR-10~\citep{cifar10}, CINIC-10~\citep{cinic10} and SVHN~\citep{svhn}.

\paragraph{Datasets.} The CIFAR-10 dataset has a training set of $50k$ images and test set of $10k$ images. The CINIC-10 dataset is a subset of ImageNet~\cite{russakovsky2015imagenet} of objects that are similar to CIFAR-10 objects; it has $260k$ images\footnote{After removing CIFAR-10 test images that are in CINIC-10}. As our source of unlabeled data, we use the Cheap-10 dataset that we created to be a benchmark for using very cheap unlabeled out-of-domain data (avaiable at \url{ https://tinyurl.com/mere5j0x}). We created Cheap-10 by searching keywords related to CIFAR-10 objects on the Bing image search engine\footnote{\url{https://www.bing.com/images/}}. A more detailed pipeline for creating Cheap-10 is described the Appendix. The important thing to note about Cheap-10 is that it is very fast to generate (hours) and can be quite noisy due to the lack of expert curation. Therefore it is a good illustration of the power of cheap, out-of-domain data.

A model trained on original CIFAR-10 data has a $68\%$ accuracy on predicting Cheap-10 labels. The number is $75\%$ for a model trained on CINIC-10 data. Both results mean that Cheap-10 is a related out of domain datasets with respect to both Cheap-10 and CIFAR-10.
The SVHN dataset had $73k$ training and $26k$ test images. For SVHN task, the original dataset contains an extra $513k$ set of training images. We use this extra images as our source of unlabeled data and synthetically push the data out of domain by adding random Gaussian noise to it.

\paragraph{Methods.}For each task, we first train a classification model on the original labeled data using cross-entropy loss function. We then use the trained model to assign pseudo-labels to unlabeled images. We next aggregate the two datasets to train a robust model using robust training. Following~\cite{carmon2019unlabeled}, we sample half of each batch from the original data and the other half from the additional pseudo-labeled data during training. %There are different methods of robust training~\cite{madry2017towards, zhang2019theoretically}. 
We use the robustness regularization loss introduced in ~\cite{zhang2019theoretically}. For a maximum allowed $\ell_p$-norm perturbation of size $\epsilon$, we use the training loss function:
\begin{align*}
 \mathscr{L}(x, y; \theta) &= -\log p_\theta(y|x) \\
 &+ \beta \max_{\hat{x} \in \bB_p(x,\varepsilon)} D_{KL}(p_\theta(y|x)||p_\theta(\hat{y}|\hat{x}))   
\end{align*}
where the regularization parameter $\beta$ balances the loss between accurate classification and stability within the $\varepsilon$ $\ell_p$-norm ball. We approximate the  maximization in the second term as follows:
\begin{itemize}%\vspace{-0.1cm}
    \item Similar to ~\cite{madry2017towards}, for $\ell_\infty$ perturbations, we focus on empirical robustness of the models and use an inner loop of projected gradient descent for the maximization.%\vspace{-0.1cm}
    \item Following~\cite{carmon2019unlabeled}, for $\ell_2$ perturbations, we focus on certified robustness and use the idea of stability training~\cite{zheng2016improving,li2018second}. We replace the maximization with large additive noise draws: $\mathbb{E}_{\hat{x}\sim\mathscr{N(0, \sigma)}}D_{KL}(p_\theta(y|x)||p_\theta(\hat{y}|\hat{x})) $. The idea is to have a model that is robust to large random perturbations. Using Cohen et al.'s method~\cite{cohen2019certified}, in test time, we can find a safe radius of certified robust prediction for each sample.
\end{itemize}

As our first experiment, we focused on empirical robustness against $\ell_\infty$ perturbations. We used a Wide ResNet 28-10 architecture~\cite{zagoruyko2016wide}.  Following the literature, for $\ell_{\infty}$ perturbation, we set $\varepsilon = 8/255$. Results for empirical robustness against $l_\infty$ perturbations are shown in Fig.~\ref{fig:results}(a). The clean accuracy is the model's performance on non-perturbed images. The robust accuracy is the model's performance on adversarially perturbed images. We use the strongest known adversarial attack methods, iterative projected gradient descent (PGD), to create the $l_\infty$ perturbations. We fine-tuned the attack hyperparameters and found that using $40$ iterations results in the smallest robust accuracy. More details are in the Appendix. We find that using Cheap-10 consistently improves the robust accuracy.
Note that, for CIFAR-10, while Cheap 10 was created in a few hours, it produced significantly better robust accuracy ($58.8 \pm 0.4 \%$) compared to using only the original data, and similarly for CINIC-10. This strategy of  using cheap noisy data to improve robustness compares favorably to state-of-the-art existing defenses applied to CIFAR-10: TRADES method ($55.4\%$, ~\cite{zhang2019theoretically}) and Adversarial Pretraining ($57.4\%$, ~\cite{hendrycks2019using}).

As our second experiment,  we use the SVHN dataset and a Wide ResNet 16-8 as our model architecture. SVHN  has a training set of $73k$ real digit images and an extended set of $531k$ images that come with the dataset. The extra set is a synthetically generated set of digits that to mimic the original dataset closely. We use the extended set as our source of unlabeled data. The model we trained (normal training) on the original training set has an accuracy of $96.8\%$ on SVHN test set and an accuracy of $98.4\%$ on the extra training set; this means that the extra data is very similar to the original SVHN dataset. To push the unlabeled data out-of-domain, we add four different levels of additive Gaussian noise to the images. We focus on $\ell_\infty$ perturbations with $\epsilon=4/255$. Fig.~\ref{fig:results}bc) describes the results. The dashed lines are the baselines for not having any additional unlabeled data. They show clean and robust accuracies when only the original training set is used. It can be observed that adding the unlabeled data robustly improves robust accuracy. As the unlabeled data distribution gets more distant from SVHN data, the improvement achieved from adding the extra set of unlabeled images becomes smaller.

As our final experiment, we focus on certified robustness. For stability training, we used $\sigma=0.25$. Fig.~\ref{fig:results}(c) shows the percentage of images that are certified to be classified correctly at each $\ell_2$ radius. First, use the CIFAR-10 dataset as the labeled data and the Cheap-10 data set as the unlabeled data. Secondly, we use the original SVHN tranining set as the labeled data and the extra set of SVHN images with additive Gaussian noise ($std=0.15$) as the unlabeled data source.
It demonstrates that adding cheap out-of-domain data consistently improves certified robustness compared to only using the original training set.  More implementation details are described in the Appendix.  

%%We use this model to create pseudo-labels for our three unlabeled datasets. We can use the accuracy of the model on our three datasets' original labels to have a better idea of domain distances. The accuracy is $97\%$ on 500k-TI--similar to CIFAR-10 test accuracy. This validates the assumption that this dataset is from the same domain as CIFAR-10. On the other hand, the accuracies on CINIC-10 and Cheap-10 are $75\%$ and $67.9\%$ respectively; confirming that these two datasets are out of domain.

% \begin{table}
% \centering
%     \begin{tabular}{c|c|c}
%          \thead{Dataset(s)} &\thead{Robust\\Acc. (\%)} & \thead{Clean\\Acc. (\%)}\\
%          \hline
%          CIFAR-10 + 500k-TI (500k)& 62.4 & 89.7\\
%          All (50K + 1260k) & 61.7 & 88.8\\
%          CIFAR-10 + Cheap10 (470k) + CINIC-10 (260k) & 59.6 & 87.7\\
%          CIFAR-10 + Cheap-10 (470k) &  59.4 & 87.5\\
%          CIFAR-10 + CINIC-10 (260K) & 59.2 & 88.1\\
%          CIFAR-10 & 48.4 & 85.2\\
%          \hline
%          CIFAR-10 (No adversarial training) & 0.0 & 96.1\\
%     \end{tabular}
%      \caption{ \label{table:versus}}
% \end{table}

%===========================
%==========================================================
\section{Further Discussions}\label{sec:conclusion}
Incorporating cheap unlabeled data is a popular way to improve the prediction performance in machine learning. In this work, we show that this substantially improves adversarial robustness, even when the unlabeled data come from a different domain. %Specifically, we present both upper and lower bounds to show there is a sample complexity gap between standard and robust classifications, and show such gap can be alleviated by leveraging unlabeled data. Such a result is also supported by experimental demonstration on two object recognition datasets where substantial  improvement  in  the  model’s  robustness  is gained by incorporating easy-to-obtain and unlabeled out-of-domain data. %Additionally, we present settings where we achieve better adversarial robustness when the incorporated unlabeled data come from a different domain than from the same domain. Moreover, we propose a novel semi-supervised algorithm when the two distributions only share the support information. %At last, we back our theories with extensive simulations by studying two object recognition datasets with additional out-of-domain data.

We prove our theoretical results for Gaussian generative models, which are very flexible (e.g. it includes common deep generative models such as GANs and VAEs). Moreover our theory is supported by our experiments using a new  dataset Cheap-10. This suggests that the vast amount of noisy out-of-domain data is a relatively untapped resource that could substantially improve the reliability of many machine learning tasks.   

%Moreover, in this paper, we present a seemingly surprising phenomenon  that in certain settings, one can achieve even better adversarial robustness when the unlabeled data comes from a shifted domain rather than the same domain as the labeled data. It would be interesting to extend this result to other settings and investigate if this phenomenon exists in generally. In addition, this result suggests that one may deliberately perturb the unlabdeled data in a certain way and improves the adversarial robustness of semisupervised algorithms. 

In this work, we showed that, in general, the adversarial robustness of a semi-supervised algorithm will be improved when the out-of-domain distribution is similar to the labeled data, and the robustness will be hurt if the out-of-domain distribution is too different. One possible extension of our work is to use the aggregation idea in \cite{li2020transfer} to deal with the challenging setting where the similarity between the out-of-domain distribution and labeled data distribution is unknown a priori.
 Such an extension will make the results applied to more general settings. %a greater class of algorithms and improve their adversarial robustness by incorporating some easy-to-obtain out-of-domain unlabeled data. 
 Further, our theoretical results and analysis also lay the foundation of studying the adversarial robustness of other tasks, such as multi-class classification and linear/kernel regression in the semi-supervised setting when the unlabeled data come from a different domain.  

The focus of this work is on the effects of out-of-domain unlabeled data, and we use the popular and simple pseudo-labeling method to capture the key insights. An interesting direction of future work is to investigate how to improve robustness with other  semi-supervised learning methods. For example, one could apply several iteration of pseudo-labeling to improve label quality. 

%Additionally, another interesting extension is to investigate the the generalization error of our proposed semi-supervised algorithms to the unseen domain. 
%adversarial robustness of other semi-supervised algorithms. In this paper, we provide two semi-supervised algorithms to make use of the information in the unlabeled domain, and provides technical tools to analyze their adversarial robustness. It would be interesting to study  the adversarial robustness of other algorithms, see \cite{blum2020foundations} and the reference therein.

%\newpage

\bibliography{DA_cite}
\bibliographystyle{iclr2021_conference}
\onecolumn
%!TEX root = main_DA.tex
\newpage
\appendix
\noindent\textbf{\Large Appendix}

\section{Omitted Proofs}
\subsection{Notation}
%$O_P$, $\lesssim$
We begin with notations. For a random variable $X$, its sub-gaussian norm/Orlicz norm is defined as $\|X\|_{\psi_2}=\inf\{t>0:\E[e^{{X^2}/{t^2}}]\le 1 \}$. For a $d$-dimensional random vector $Y$, the sub-gaussian norm of $Y$ is defined as $\|Y\|_{\psi_2}=\sup_{v\in S^{d-1}}\|\langle Y,v\rangle\|$, where $S^{d-1}$ denotes the sphere of a unit ball in $\R^d$.
For two sequences of positive numbers $a_n$ and $b_n$, $a_n \lesssim b_n$ means that  for some constant $c>0$, $a_n \leq c b_n$ for all $n$, and $a_n \asymp b_n $ if $a_n \lesssim b_n$ and $b_n \lesssim a_n$. Further, we use the notion $o_p$ and $O_p$, where for a sequence of random variables $X_n$, $X_n=o_p(a_n)$ means $X_n/a_n\to0$ in probability, $X_n=O_p(b_n)$ means that for any $\varepsilon>0$, there is a constant $K$, such that $\mathbb P(|X_n|\le K\cdot b_n)\ge 1-\varepsilon$, and $X_n=\Omega_p(b_n)$ means that for any $\varepsilon>0$, there is a constant $K$, such that $\mathbb P(|X_n|\ge K\cdot b_n)\ge 1-\varepsilon$. Finally, we use $c_0, c_1, c_2, C_1, C_2, \ldots$ to denote generic positive constants that may vary from place to place. %{\red define the sub-gaussian norm}

Besides, let $L=L_1L_2$, then $\phi(\cdot)$ is $L_1L_2$-Lipchitz in $\ell_2$-norm.

\subsection{Proof of Theorem \ref{thm:standardaccuracy}}
%============

We firstly consider to prove a bound for
$$\bP\Big(\big\|\frac{1}{n}\sum_{i=1}^n\phi(z_i)-\bE\phi(z)\big\|\geq t \Big),$$
where $z_i\sim \cN(\mu, \sigma^2I)$. 
\begin{lemma}\label{lemma:handy}
There exists a universal constant $c$, such that
$$\bP\Big(\big\|\frac{1}{n}\sum_{i=1}^n\phi(z_i)-\bE\phi(z)\big\|\geq c\sigma \big(\frac{\sqrt{d}\tilde{L}}{\sqrt{n/2}}+L\sqrt{\frac{2\log(2/\delta)}{n}}\big) \Big)\leq \delta.$$
\end{lemma}
From the above inequality, we can immediately obtain
$$\bP\Big(\big\|\frac{1}{n}\sum_{i=1}^n\phi(z_i)\big\|\geq\big\|\bE\phi(z)\big\|+ c\sigma \big(\frac{\sqrt{d}\tilde{L}}{\sqrt{n/2}}+L\sqrt{\frac{2\log(2/\delta)}{n}}\big) \Big)\leq \delta.$$
\begin{remark}
Note that the concentration bound still holds for $y\phi(yz)$ and $\phi(yz)$ by simply applying conditional probability.
\end{remark}
\begin{proof}
Let $\vartheta_u(z)=\langle\phi(z)-\bE\phi(z),u\big\rangle$
\begin{align*}
\bP\Big(\big\|\frac{1}{n}\sum_{i=1}^n\phi(z_i)-\bE\phi(z)\big\|\geq t \Big)&=\bP\Big(\sup_{\|u\|=1}\big\langle\frac{1}{n}\sum_{i=1}^n\phi(z_i)-\bE\phi(z),u\big\rangle\geq t \Big)\\
&=\bP\Big(\sup_{\|u\|=1}\frac{1}{n}\sum_{i=1}^n\vartheta_u(z_i)\geq t\Big)\\
&\leq \bP\Big(\sup_{u,u' \in \bB(0,1)}\big|\frac{1}{n}\sum_{i=1}^n\vartheta_u(z_i)-\sum_{i=1}^n\vartheta_{u'}(z_i)\big|\geq t\Big)
%&\leq \bP\Big(\sup_{\|u\|=1}\frac{1}{n}\sum_{i=1}^n\vartheta_u(z_i)\geq t\Big)+\bP\Big(\sup_{\|u\|=1}\frac{1}{n}\sum_{i=1}^n\vartheta_u(z_i)\geq t\Big)
\end{align*}
Let us use chaining and Orlicz-processes to obtain a bound. We prove $\{\frac{1}{n}\sum_{i=1}^n\vartheta_u(z_i), u\in \bB(0,1)\}$ is a $\vartheta_2$-process with respect to a rescaled distance $\|\cdot \|/\lambda$ for some $\lambda>0$. If so, we will have 
$$\bE \exp\Big(\frac{\lambda^2|\frac{1}{n}\sum_{i=1}^n\vartheta_u(z_i)-\frac{1}{n}\sum_{i=1}^n\vartheta_{u'}(z_i)|^2}{\|u-u'\|^2}\Big)\leq 2. $$

The LHS 
\begin{align*}
\bE \exp\Big(\frac{\lambda^2|\frac{1}{n}\sum_{i=1}^n\vartheta_u(z_i)-\frac{1}{n}\sum_{i=1}^n\vartheta_{u'}(z_i)|^2}{\|u-u'\|^2}\Big)&\leq  \int_{0}^\infty e^t\bP\Big(\frac{|\frac{1}{n}\sum_{i=1}^n\vartheta_u(z_i)-\frac{1}{n}\sum_{i=1}^n\vartheta_{u'}(z_i)|}{\|u-u'\|}\geq \frac{\sqrt{t}}{\lambda}\Big)dt\\
&\leq \int_{0}^\infty e^t 2\exp\big(-\frac{tn}{\lambda^2\sigma^2L^2}\big)dt.
\end{align*}
As long as 
$$\frac{tn}{\lambda^2\sigma^2L^2}\geq 2,~~\text{i.e.}~~\lambda\leq \frac{\sqrt{n/2}}{\sigma L},$$
we would obtain the Dudley entropy integral as
$$J(D)=\frac{1}{\lambda}\int_{0}^1\sqrt{\log(1+\exp(d\log\frac{1}{\delta}))}d\delta \approx const\cdot \frac{\sqrt{d}}{\lambda}.$$
We let $\lambda = \sqrt{n/2}/(\sigma L)$, it gives us $J(D)=(\sigma\tilde{L})/\sqrt{n/2}$, where $\tilde{L}=const \cdot L$.
\end{proof}

Next, let us consider bounding 
$$\bP\big(\hat{w}^\top (\phi(z)-\hat{b})\leq 0 \big)=\bP\Big(\frac{\hat{w}^\top}{\|\hat{w}\|} (\phi(z)-\hat{b})\leq 0 \Big).\quad \text{(notice $\hat{w}=0$ is of zero probability)}$$
We further denote $\nu_w=\big[\bE\phi(z_i^+) -\bE\phi(z_i^-)\big]/2$, $\nu_b= \big[\bE\phi(z_i^+) +\bE\phi(z_i^-)\big]/2$ and 
$$\gamma =  c\sigma \big(\frac{\sqrt{d}\tilde{L}}{\sqrt{n/2}}+L\sqrt{\frac{2\log(2/\delta)}{n}} \big).$$
From Lemma \ref{lemma:handy}, we can obtain
\begin{align*}
&\bP\left(\big|\|\hat{w}-\nu_w\|\big|\geq \gamma\right) \leq \delta,\\
&\bP\left(\|\hat{b}-\nu_b\|\geq \gamma \right) \leq \delta.
\end{align*}
%Besides, we know that for $\tilde{\delta}>0$
%\begin{equation*}
%\bP(\big|\|\phi(z)-\nu_b-\nu_w\|\big|\geq t)\leq \tilde{\delta},
%\end{equation*}
%where 
%$$t=c\sigma \Big(\sqrt{2d}\tilde{L}+L\sqrt{2\log(2/\tilde{\delta})}\Big).$$
%Let $\varsigma_w=\hat{w}-\nu_w$ and $\varsigma_b=\hat{b}-\nu_b$.
%\begin{align*}
%\bP\Big(\frac{\hat{w}^\top}{\|\hat{w}\|} (\phi(z)-\hat{b})\leq 0 \Big)&=\bP\Big(\frac{(\nu_w+\varsigma_w)^\top}{\|\nu_w+\varsigma_w\|}(\phi(z)-\hat{b})\leq 0\Big)%&\leq\bP\Big(\frac{\nu_w^\top}{\|\nu_w\|+\|\varsigma_w\|}(\phi(z)-\nu_b)\leq \|\varsigma_b\|+\frac{\|\varsigma_w\|\|\phi(z)-\nu_b\|}{\|\nu_w\|+\|\varsigma_w\|} \Big)\\
%&\leq \bP\Big(\frac{\nu_w^\top\nu_w-\|\nu_w\|t}{\|\nu_w\|+\gamma}\leq \gamma+\frac{\gamma(t+\|\nu_w\|)}{\|\nu_w\|+\gamma} \Big)+2\delta+\tilde{\delta}.
%\end{align*}
Notice that for any unit vector $v$, $v^\top (\phi(z)-\hat{b})$ is a $\sigma L$-Lipschitz function of $(z^\top, z_1^\top,\cdots,z_n^\top)^\top\sim \cN(0,I_{(n+1)m})$, by standard concentration, we have the following lemma.
\begin{lemma}
For any $t>0$ and unit vector $v$
$$\bP\Big(|v^\top (\phi(z)-\hat{b})-v^\top \nu_w|\geq t\Big)\leq 2\exp(-\frac{t^2}{2\sigma^2 L^2}).$$
\end{lemma}

Next, we provide a bound for $\langle\hat{w} ,\nu_w\rangle$. 
\begin{lemma} 
For any $t>0$
\begin{equation*}
\bP\big(|\langle\hat{w},\nu_w \rangle- \|\nu_w\|^2|\geq t\big)\leq 2\exp(\frac{-nt^2}{2\sigma^2L^2\|\nu_w\|^2})
\end{equation*}
Taking $\delta=2\exp(\frac{-nt^2}{2\sigma^2L^2\|\nu_w\|^2}) $, we have
$$\bP\big(|\langle\hat{w},\nu_w \rangle- \|\nu_w\|^2|\geq \sigma L \|\nu_w\|\sqrt{\frac{2\log(2/\delta)}{n}}\big)\leq \delta$$
\end{lemma}
\begin{proof}
LHS is equivalent to
$$\bP\big(|\langle\hat{w}-\nu_w,\nu_w \rangle|\geq t\big).$$
Besides,  we have $\langle\hat{w} -\nu_w,\nu_w/\|\nu_w\|\rangle=\frac{1}{n}\sum_{i=1}^n\langle y_i\phi(y_iz_i) -\nu_w,\nu_w/\|\nu_w\|\rangle$ is a sum of sub-gaussian variables with constant $\sigma L$, then by sub-gaussian tail bound we have 
$$\bP\big(|\langle\hat{w}-\nu_w,\nu_w \rangle|\geq t\big)\leq 2\exp(\frac{-nt^2}{2\sigma^2L^2\|\nu_w\|^2}).$$
\end{proof}
%==========================================
\paragraph{[Proof of Theorem \ref{thm:standardaccuracy}] } If we denote $E=A\cup B$, where $A=\Big\{|\langle\hat{w},\nu_w \rangle- \|\nu_w\|^2|\leq \sigma L \|\nu_w\|\sqrt{\frac{2\log(2/\delta_1)}{n}}\Big\}$, $B=\{|\|\hat{w}-\nu_w\|\big|\leq \gamma\}$, then with probability $\bP(E)$, we have for $t>0$
\begin{align*}
\bP\Big(\frac{\hat{w}^\top}{\|\hat{w}\|} (\phi(z)-\hat{b})\leq 0\Big|E \Big)%&\leq \bP\Big(\frac{\hat{w}^\top}{\|\hat{w}\|}(\phi(z)-\hat{b})\leq 0\Big|E\Big)\bP(E)+\bP(E^c)\\
%&\leq \bP\Big(\frac{\hat{w}^\top}{\|\nu_w+\varsigma_w\|}\nu_w\leq \|\varsigma_b\|+t\Big|E\Big)\bP(E)+2\exp(-\frac{t^2}{2\sigma^2 L^2})\bP(E)+\bP(E^c)\\
&\leq  \bP\Big(\frac{\hat{w}^\top\nu_w}{\|\nu_w\|+\gamma}\leq t\Big|E\Big)+2\exp(-\frac{t^2}{2\sigma^2 L^2}).
\end{align*}
As long as we choose $\gamma$ and $t$ such that 
$$\frac{\hat{w}^\top\nu_w}{\|\nu_w\|+\gamma}> t$$
we have
$$\bP\Big(\frac{\hat{w}^\top}{\|\hat{w}\|} (\phi(z)-\hat{b})\leq 0\Big|E \Big)\leq 2\exp(-\frac{t^2}{2\sigma^2 L^2}).$$
%where $\tilde{\delta}=2\exp(-\frac{t^2}{2\sigma^2 L^2}).$
%Let $\tilde{\delta}=\delta$, and 
%$$L\sqrt{2\log(2/\delta)}=\sqrt{2d}\tilde{L}\Rightarrow \delta = 2\exp\big(-(\tilde{L}/L)^2d\big),$$
%it gives us
%$$\gamma=2c\sigma\frac{\sqrt{2d}\tilde{L}}{n},~~ t =2c\sigma\sqrt{2d}\tilde{L}.$$
We take 
$$\delta =2\exp(-(\tilde{L}/L)^2d),~~\delta_1=2\exp(-\frac{d}{8\sigma^2L^2})$$
then 
$$\gamma=2\sqrt{2}c\sigma\tilde{L}\sqrt{\frac{d}{n}}.$$
As a result, we obtain
$$|\langle\hat{w},\nu_w \rangle- \|\nu_w\|^2|\leq \frac{d}{2\sqrt{n}},~~ |\|\hat{w}-\nu_w\|\big|\leq \gamma$$
with probability at least $1-\delta_1-\delta$.

We can choose 
$$t=\frac{\sqrt{d}(\sqrt{n}-1/2)}{\sqrt{n}+2\sqrt{2}c\sigma\tilde{L}},~~$$
so that
$$\bP\Big(\frac{\hat{w}^\top}{\|\hat{w}\|} (\phi(z)-\hat{b})\leq 0 \Big|E\Big)\leq 2\exp\left(\frac{-d(\sqrt{n}-1/2)^2}{2\sigma^2L^2(\sqrt{n}+2\sqrt{2}c\sigma\tilde{L})^2}\right).$$
Thusly, 
$$\bP\left(\beta(\hat{w},\hat{b})\geq 2\exp\left(\frac{-d(\sqrt{n}-1/2)^2}{2\sigma^2L^2(\sqrt{n}+2\sqrt{2}c\sigma\tilde{L})^2}\right)\right)\leq \bP(E^c),$$
which gives us the final result stated in the theorem.

%===========================================

\subsection{Proof of Theorem 2}
Since we know $\vartheta$ is $L'_2$-Lipchitz continuous in $\ell_\infty$-norm. Then, we know $$\bP_{(x,y)\sim \cP} [\exists u \in \bB_p(x,\varepsilon):\big\langle \hat{w}, y\cdot(\vartheta(u)-\hat{b})\big\rangle\leq 0]\geq \bP_{(zy,y)\sim \cP'} [\exists u \in \bB_p(yz,\varepsilon/L'_2):\big\langle \hat{w}, y\cdot(\phi(u)-\hat{b})\big\rangle\leq 0]$$
since the pre-image of  $bB_p(x,\varepsilon)$ via $\vartheta$ includes the set $\bB_p(yz,\varepsilon/L'_2)$. Then following the argument in \cite{schmidt2018adversarially}, the result follows.

\begin{remark}
As a side interest, we also provide an analysis to show the lower bound result in Theorem 3.2 is achievable up to a logarithm factor, by purely using labeled data.  This scale matches the result in \cite{schmidt2018adversarially}, but under a more general model considered in our paper.
% which means in order to reach a small robust error, a sample complexity matching the lower bound up to a factor of $\log d$ is sufficient in our model. 
\begin{align*}
\bP_{(x,y)\sim \cP} [\exists u \in \bB_p(x,\varepsilon): f_{\hat{w},\hat{b}}(u)\neq y]&=\bP_{(x,y)\sim \cP} [\exists u \in \bB_p(x,\varepsilon):\big\langle \hat{w}, y\cdot(\vartheta(u)-\hat{b})\big\rangle\leq 0]\\
&=\bP_{(x,y)\sim \cP} [\exists \eta \in \bB_p(0,\varepsilon):\big\langle \hat{w}, y\cdot(\vartheta(x+\eta)-\hat{b})\big\rangle\leq 0]\\
&\leq\bP_{(x,y)\sim \cP} [\big\langle \hat{w}, (\phi(z)-\hat{b})\big\rangle+\min_{\eta \in \bB_p(0,\varepsilon)}\langle\eta L_1,\hat{w}\rangle\leq 0]\\
&=\bP_{(x,y)\sim \cP} [\big\langle \hat{w}, (\phi(z)-\hat{b})\big\rangle\leq \varepsilon L_1\|\hat{w}\|_q]
\end{align*}
where $1/p+1/q=1$.

When $p=\infty$, $q=1$, it leads to $\|\hat{w}\|_1 /\|\hat{w}\|\leq \sqrt{d}$. Recall in Theorem \ref{thm:standardaccuracy}, $E=A\cup B$, where $A=\Big\{|\langle\hat{w},\nu_w \rangle- \|\nu_w\|^2|\leq \sigma L \|\nu_w\|\sqrt{\frac{2\log(2/\delta_1)}{n}}\Big\}$, $B=\{|\|\hat{w}-\nu_w\|\big|\leq \gamma\}$, then with probability $\bP(E)$, we have
\begin{align*}
\bP\Big(\frac{\hat{w}^\top}{\|\hat{w}\|} (\phi(z)-\hat{b})\leq  \varepsilon L_1\frac{\|\hat{w}\|_q}{\|\hat{w}\|}\Big|E \Big)%&\leq \bP\Big(\frac{\hat{w}^\top}{\|\hat{w}\|}(\phi(z)-\hat{b})\leq 0\Big|E\Big)\bP(E)+\bP(E^c)\\
%&\leq \bP\Big(\frac{\hat{w}^\top}{\|\nu_w+\varsigma_w\|}\nu_w\leq \|\varsigma_b\|+t\Big|E\Big)\bP(E)+2\exp(-\frac{t^2}{2\sigma^2 L^2})\bP(E)+\bP(E^c)\\
&\leq \bP\Big(\frac{\hat{w}^\top\nu_w}{\|\nu_w\|+\gamma}\leq t+ \varepsilon L_1\frac{\|\hat{w}\|_q}{\|\hat{w}\|}\Big|E\Big)+2\exp(-\frac{t^2}{2\sigma^2 L^2})\\
&\leq \bP\Big(\frac{\hat{w}^\top\nu_w}{\|\nu_w\|+\gamma}\leq t+ \varepsilon L_1\sqrt{d}\Big|E\Big)+2\exp(-\frac{t^2}{2\sigma^2 L^2}).
\end{align*}
We still choose 
$$\delta =2\exp(-(\tilde{L}/L)^2d),~~\delta_1=2\exp(-\frac{d}{8\sigma^2L^2})$$
such that
$$\gamma=2\sqrt{2}c\sigma\tilde{L}\sqrt{\frac{d}{n}}.$$
As a result, we obtain
$$|\langle\hat{w},\nu_w \rangle- \|\nu_w\|^2|\leq \frac{d}{2\sqrt{n}},~~ |\|\hat{w}-\nu_w\|\big|\leq 2\sqrt{2}c\sigma\tilde{L}\sqrt{\frac{d}{n}}$$
with probability at least $1-\delta_1-\delta$. We choose $t$ such that
$$\frac{\hat{w}^\top\nu_w}{\|\nu_w\|+\gamma}> t+\varepsilon L_1\sqrt{d}.$$

We let 
$$t=\frac{(\sqrt{d})^2-d/(2\sqrt{n})}{\sqrt{d}/2+2\sqrt{2}c\sigma\tilde{L}\sqrt{\frac{d}{n}}}-\varepsilon L_1\sqrt{d}=\frac{\sqrt{d}(\sqrt{n}-1/2)}{\sqrt{n}/2+2\sqrt{2}c\sigma\tilde{L}}-\varepsilon L_1\sqrt{d}.$$
As long as 
$$\varepsilon\leq (\frac{\sqrt{n}-1/2}{\sqrt{n}/2+2\sqrt{2}c\sigma\tilde{L}}-\frac{\sigma L \sqrt{2\log(1/\beta)}}{\sqrt{d}})/L_1,$$
we have 
$$\beta^{\cR}(\hat{w},\hat{b})\leq 2\exp(-\frac{t^2}{2\sigma^2 L^2})\leq \beta$$

\end{remark}

\subsection{Statistical Measures}
Recall the definition of 
$$d_\nu=\max\Big\{\frac{\|\tilde\mu_1-\mu_1\|}{\|\tilde\mu_1-\tilde\mu_2\|},\frac{\|\tilde\mu_2-\mu_2\|}{\|\tilde\mu_1-\tilde\mu_2\|}\Big\}.$$
We now make connections to commonly used statistical measures and provide a sketch of proof.

(a). \textit{Wasserstein Distance}: the Wasserstein Distance induced by metric $\rho$ between distributions $\cP_1$ and $\cP_2$ over $\bR^d$ is defined as $$W_\rho(\cP_1,\cP_2)=\sup_{\|f\|_{\text{Lip}}\leq 1}[\int fd\cP_1-fd\cP_2],$$
where $\|f\|_{\text{Lip}}\leq 1$ indicates the class of $f:\bR^d\mapsto\bR$ such that for any $x,x'\in \bR^d$, $|f(x)-f(x)|\leq \rho(x,x').$ Let us consider $\rho(x,x')=\|x-x'\|.$ 
%============
\begin{proposition}
Suppose $\max\{W_\rho(\cP_1,\tilde{\cP}_1),W_\rho(\cP_2,\tilde{\cP}_2)\}\leq \tau$, for $ \tau\geq 0$, then we have $\|\mu_i-\tilde{\mu_i}\|\leq\tau,\quad i=1,2.$
As a result, 
$$d_\nu\leq\frac{\tau}{\|\tilde{\mu}_1-\tilde{\mu}_2\|}.$$
If we further have $ \tau\leq \|\mu_1-\mu_2\|/2$, we have $d_\nu\leq\tau/(\|\mu_1-\mu_2\|-2\tau).$
\end{proposition}
\begin{proof}
Notice $f(x)=x$ also satisfies $\|f\|_{Lip}\leq 1$, then we know  $\|\mu_i-\tilde{\mu_i}\|\leq\tau,\quad i=1,2.$ If we further have $ \tau\leq \|\mu_1-\mu_2\|/2$, plugging into the denominator, the result follows. 
\end{proof}

(b). \textit{Maximal Information}: Maximal Information between distributions $\cP_1$ and $\cP_2$ over $\bR^d$ is defined as 
$$MI(\cP_1,\cP_2)=\sup_{O\subseteq \bR^d}\max\left\{\frac{\bP_{x\sim\cP_1}(x\in O)}{\bP_{x\sim\cP_2}(x\in O)}, \frac{\bP_{x\sim\cP_2}(x\in O)}{\bP_{x\sim\cP_1}(x\in O)}\right\}.$$
%============
\begin{proposition}
Suppose $\max\{MI(\cP_1,\tilde{\cP}_1),MI(\cP_2,\tilde{\cP}_2)\}\leq \tau$
for $1\leq\tau\leq 1+\|\mu_1-\mu_2\|/(2\|\mu_1\|+2\|\mu_2\|)$, then we have $\|\mu_i-\tilde{\mu}_i\|\leq(\tau-1)\|\mu_i\|,~ i=1,2.$
As a result, we have
$$d_\nu\leq\frac{(\tau-1)\max\{ \|\mu_1\|, \|\mu_2\|\}}{\|\mu_1-\mu_2\|-2(\tau-1)(\|\mu_1\|+\|\mu_2\|)}.$$
\end{proposition}
As we can see, as $\tau\rightarrow 1$, $d_\nu\rightarrow 0$.
\begin{proof}
Let $X_1\sim \cP_1$, $X_2\sim\cP_2$. By the definition of Maximal Information, 
$$\sup_{x\in\bR^d}\max\left\{\frac{\bP(X_1=x)}{\bP(X_2=x)},\frac{\bP(X_2=x)}{\bP(X_1=x)}\right\}\leq \tau.$$
Then, we know $\|\mu_i-\tilde{\mu}_i\|\leq(\tau-1)\|\mu_i\|,~ i=1,2$ once we notice for all corresponding entries of the vector of $X_1$ and $X_2$, their maximal information is bounded by $\tau.$
So,
$$d_\nu\leq\frac{(\tau-1)\max\{ \|\mu_1\|, \|\mu_2\|\}}{\|\tilde{\mu}_1-\tilde{\mu}_2\|}.$$
If we further have $\tau\leq 1+\|\mu_1-\mu_2\|/(2\|\mu_1\|+2\|\mu_2\|)$, plugging into the denominator, the result follows.
\end{proof}

(c). $\cH$-Divergence: let $\cH$ be a class of binary classifiers, then $\cH$-divergence between distributions $\cP$ and $\cP'$ over $\bR^d$ is defined as 
$$D_\cH(\cP,\cP')=\sup_{h\in \cH}|\bP_{x\sim\cP}(h(x)=1)-\bP_{x\sim\cP'}(h(x)=1)|.$$
To illustrate the connection between Theorem \ref{thm:robustaccuracy} and $\cH$-divergence, we consider a specific hypothesis class
\begin{equation}\label{eq:hypothesis}
\cal H=\big\{h |h(t)=sgn(w^\top(t- b)), (w,b)\in \bR^d\times \bR^d \big\}.
\end{equation}
%============
\begin{proposition}
 Suppose for $X_i\sim\cP_i$ and $\tilde{X}_i\sim\tilde{\cP}_i$ $i=1,2$, the sub-gaussian norm of $\|X_i-\mu_i\|_{\psi_2}$ and $\|\tilde{X}_i-\tilde{\mu}_i\|_{\psi_2}$ are bounded by $\sigma$ and $\tilde{\sigma}$, where $X_i\sim\cP_i$, $\tilde{X}_i\sim\tilde{\cP}_i$ and $\mu_i$, $\tilde{\mu}_i$ are the corresponding means. Let $\alpha=\zeta\sqrt{\log(4/(1-\tau))}$, where $\zeta=\max\{\sigma,\tilde{\sigma}\}$, if $\max\{D_{\cH}(\cP_1,\tilde{\cP}_1),D_{\cH}(\cP_2,\tilde{\cP}_2)\}\leq \tau,$ for $\tau\leq 1$, we have $\|\mu_i-\tilde{\mu}_i\|\leq \alpha$, $i=1,2$. As a result,
 $$d_\nu\leq \frac{\alpha}{\|\tilde{\mu}_1-\tilde{\mu}_2\|}.$$
 If we further have $\tau\leq 1-4\exp(-\|\mu_1-\mu_2\|^2/4\zeta^2)$, then $d_\nu\leq \alpha/(\|\mu_1-\mu_2\|-2\alpha).$
\end{proposition}
\begin{proof}
It follows a simple geometric argument -- a hyperplane cannot distinguish the two distributions too well. 
Recall if $\|X_i-\mu_i\|_{\psi_2}\leq \sigma_i$ and $\|\tilde{X}_i-\tilde{\mu}_i\|_{\psi_2}\leq \tilde{\sigma}_i$, then for $i=1,2$
$$\bP(\|X_i-\mu_i\|\geq t)\leq 2\exp(-\frac{t^2}{\sigma^2}),~~\bP(\|\tilde{X}_i-\tilde{\mu}_i\|\geq t)\leq 2\exp(-\frac{t^2}{\tilde{\sigma}^2})$$ 
Consider $t^*$ such that 
$$2\exp(-\frac{t^{*2}}{\zeta^2})=(1-\tau)/2,~~i.e. ~t^*=\alpha/2.$$
It is easy to see the distance $\|(\mu_1-\mu_2)/2-(\tilde{\mu}_1-\tilde{\mu}_2)/2\|$ should be upper bounded by $2t^*$, otherwise, there exists a hyperplane such that the probability mass of $X_1\sim\cP_1$ and $\tilde{X}_1\sim\tilde{\cP}_1$ has high probability mass on difference side of the hyperplane.

\end{proof}

As we can see in the case for Wasserstein Distance, as $\tau\rightarrow 0$, $d_\nu\rightarrow 0$. However, for $\cH$-Divergence when $\tau\rightarrow 0$, $d_\nu$ will not go to $0$.  That is due to the constraint of capacity of $\cH$. Even if $\tau=0$, $\cP_i$ and $\tilde{\cP}_i$ can still be quite different.

%===========================================
\subsection{Proof of Theorem 3 and Proposition 1}
Let us recall the statement of Theorem 3 with some specified constants.% Proposition 1 can be proved using the same technique by changing the tail probabilities.
\begin{theorem}[Robust accuracy]%\label{thm:robustaccuracy}
 Consider the Gaussian generative model, where the marginal distribution of the input $x$ of labeled domain is a uniform mixture of two distributions with mean $\mu_{1}=\E[\phi(z)]$ and  $\mu_{2}=\E[\phi(-z)]$ respectively, where $z\sim\cN(0,\sigma^2I_{s_1})$. Suppose the marginal distribution of the input of unlabeled domain is a mixture of  two sub-gaussian distributions with mean $\tilde\mu_{1}$ and  $\tilde\mu_{2}$ with mixing probabilities $q$ and $1-q$ and $\left\|\E\left[\vartheta(\tilde x_i)-\E[\vartheta(\tilde x_i)]\mid a^T \vartheta(\tilde x_i)=b\right]\right\|\le c_\varepsilon\cdot (\sqrt d + |b|)$ for fixed unit vector $a$.  Assuming the sub-gaussian norm for both labeled and unlabeled data are upper bounded by a universal quantity $\sigma_{\max}:=C_\sigma d^{1/4}$, %\james{explain what is $\sigma_{\max}$}
  $c_q<q<1-c_q$, $\|\tilde\mu_1-\tilde\mu_2\|_2=C_\mu \sqrt d$, for some constants $C_\sigma>0$,$0<c_q<1/2$, $C_\mu>0$ sufficiently large, and 
%$$d_\nu=\max\Big\{\frac{|(\tilde\mu_1-\tilde\mu_2)^\top(\tilde\mu_1-\mu_1)|}{\|\tilde\mu_1-\tilde\mu_2\|^2},\frac{|(\tilde\mu_1-\tilde\mu_2)^\top(\tilde\mu_2-\mu_2)|}{\|\tilde\mu_1-\tilde\mu_2\|^2}\Big\}<\frac{1}{4},$$
$$
d_\nu=\max\Big\{\frac{\|\tilde\mu_1-\mu_1\|}{\|\tilde\mu_1-\tilde\mu_2\|},\frac{\|\tilde\mu_2-\mu_2\|}{\|\tilde\mu_1-\tilde\mu_2\|}\Big\}<c_0,$$
for some constant $c_0\le 1/4$,
then the robust classification error is at most $1\%$ when $d$ is sufficiently large, $n\ge C$ for some constant C (not depending on $d$ and $\varepsilon$) and
$$\tilde n\gtrsim \varepsilon^2\log d\sqrt d.$$
\end{theorem}

Now let us proceed to the proof.

For simplicity of presentation. We first denote the distributions for the two classes of labeled data as ${subGaussian}(\mu_1,\sigma_{\max}^2)$, and ${subGaussian}(\mu_2,\sigma_{\max}^2)$ respectively. Similarly, we also denote the distributions for the two classes of unlabeled data as ${subGaussian}(\bm\tilde\mu_1,\sigma_{\max}^2)$, and ${subGaussian}(\bm\tilde\mu_2,\sigma_{\max}^2)$ respectively. Also, to avoid the visual similarity and emphasize the estimates constructed by the labeled and unlabeled data respectively, we write $\hat w$ as $\hat w_{\text{intermediate}}$, $\hat b$ as $\hat b_{\text{intermediate}}$, $\tilde w$ as $\tilde w_{\text{final}}$ and $\tilde b$ as $\tilde b_{\text{final}}$.

%subgaussian distribution %class in $\R^d$ by $$
%{subGaussian}(\bm\mu,\sigma^2)=\{\bm X\in\R^d: \E[\bm X]=\bm\mu, \|\bm X-\bm\mu\|_{\psi_2}\le\sigma^2\}
%$$

Then, let us write out the robust error of misclassifying class 1 against the $\ell_\infty$ attack  (the robust error of misclassifying class 2 can be bounded similarly) as
%\begin{align}\label{eq:misclassification}
%\eta_1=&\bP\Big(\frac{\finalest^\top}{\|\finalest\|} ( x-\hat{b})\leq 0 \mid  x\sim subGaussian(\mu_1,\sigma^2)\Big)
%=\bP\Big(\frac{\finalest^\top}{\|\finalest\|} \varepsilon\leq -\frac{\finalest^\top}{\|\finalest\|}(\mu_1-\hat b) \Big)
%%=&\bP\Big(\frac{\finalest^\top}{\|\finalest\|} \varepsilon\leq -\frac{\tilde\mu^{\top}\finalest}{\left\Vert \finalest\right\Vert }+ d_\nu\Big).
%%\le&  \exp\left(d_\nu-\left(\frac{\sigma^2}{\left\Vert \mu\right\Vert ^{2}}+\frac{2\sigma^{4}d\sqrt{\log d}/\nunlab+2\sigma^2(1-r_\mu)\|\mu\|^2+d\sigma^2({\tilde n}^{-1/2}+d_{\nu})}{\left\Vert \mu\right\Vert ^{4}\left(\frac{1}{6}r_\mu-\frac{\sqrt{2}\sigma}{\left\Vert \mu\right\Vert }-\sigma\cdot(1/{\tilde n}^{1/4}+d_{\nu})\right)^{2}}\right)^{-1/2}\right)
%\end{align}
%Similarly, the robust error against the $\ell_\infty$ attack is then 
\begin{align*}
&\max_{\|\delta\|_\infty\le L_1'\varepsilon}\bP\Big(\frac{\finalest^\top}{\|\finalest\|} (\tilde x+\delta-\tilde{b}_{\text{final}})\leq 0 \mid \tilde x\sim subGaussian(\tilde\mu_1,\sigma_{\max}^2)\Big)\\
=&\bP\Big(\frac{\finalest^\top}{\|\finalest\|} \varepsilon\leq -\frac{\finalest^\top}{\|\finalest\|}(\mu_1-  \tilde b_{\text{final}}) +\varepsilon\frac{\|\finalest\|_1}{\|\finalest\|}\Big)\\
=&\bP\Big(\frac{\finalest^\top}{\|\finalest\|} \varepsilon\leq -\frac{\finalest^\top}{\|\finalest\|}(\mu_1-  \tilde b_{\text{final}}) +L_1'\cdot\varepsilon\sqrt d\Big)
\end{align*}

Denote $  \tilde b_{\text{final}}:=\frac{\tilde\mu_1+\tilde\mu_2}{2}+d_b$, we then have 
\begin{align*}
&\bP\Big(\frac{\finalest^\top}{\|\finalest\|} \varepsilon\leq -\frac{\finalest^\top}{\|\finalest\|}(\mu_1-  \tilde b_{\text{final}}) +L_1'\cdot\varepsilon\sqrt d\Big)\\
=&\bP\Big(\frac{\finalest^\top}{\|\finalest\|} \varepsilon\leq -\frac{\finalest^\top}{\|\finalest\|}((\mu_1-\tilde\mu_1)+\frac{\tilde\mu_1-\tilde\mu_2}{2}-d_b)+L_1'\cdot\varepsilon\sqrt d \Big)
\end{align*}

We are going to bound
$|\frac{\finalest^\top}{\|\finalest\|}(\mu_1-\tilde\mu_1)|$, $|\frac{\finalest^\top}{\|\finalest\|}(\tilde\mu_1-\tilde\mu_2)/2|$, and $|\frac{\finalest^\top}{\|\finalest\|}d_b|$ respectively.

Let $\tilde\mu=(\tilde\mu_1-\tilde\mu_2)/2$, $\tilde\nu=(\tilde\mu_1+\tilde\mu_2)/2$, $\mu=(\mu_1-\mu_2)/2$, $\nu=(\mu_1+\mu_2)/2$, and $b_{i}$ be the indicator that the $i$th pseudo-label $\tilde y_i$ is incorrect, so that $\xun_{i}\sim\tilde\nu+{subGaussian}\left(\left(1-2b_{i}\right)\yun_{i}\tilde\mu,\sigma^{2}\right)$

Let $\tilde n_1=\sum_{i=1}^{\tilde n}1\{\tilde y_i=1\}$, $\tilde n_2=\sum_{i=1}^{\tilde n}1\{\tilde y_i=-1\}$. 
We recall  the final direction estimator as
%\[
%\semisupest=\frac{1}{\nunlab}\sum_{i=1}^{\nunlab}\yun_{i}\xun_{i}=(\frac{1}{\tilde n}\sum_{i=1}^{\tilde n} \tilde y_i)\tilde\nu+\gamma\tilde\mu+\frac{1}{\nunlab}\sum_{i=1}^{\nunlab}\yun_{i}\eps_{i}
%\]
\begin{align*}
\semisupest=&\frac{1}{2\nunlab_1}\sum_{\tilde y_i=1}^{}\xun_{i}-\frac{1}{2\nunlab_2}\sum_{\tilde y_i=-1}^{}\xun_{i}\\
=&\frac{1}{2\nunlab_1}\sum_{\tilde y_i=1}^{}\left(1-2b_{i}\right)\tilde\mu+\frac{1}{2\nunlab_1}\sum_{\tilde y_i=1}\varepsilon_i+\frac{1}{2\nunlab_2}\sum_{\tilde y_i=-1}^{}\left(1-2b_{i}\right)\tilde\mu-\frac{1}{2\nunlab_2}\sum_{\tilde y_i=-1}\varepsilon_i,
%&{\color{red}=(\frac{1}{\tilde n}\sum_{i=1}^{\tilde n} \tilde y_i)\tilde\nu+\gamma\tilde\mu+\frac{1}{\nunlab}\sum_{i=1}^{\nunlab}\yun_{i}\eps_{i}}
\end{align*}
where $\eps_{i}\sim{subGaussian}\left(0,\sigma^{2}\right)$ independent
of each other. 

Now let  
\[
\gamma\defeq\frac{1}{2\nunlab_1}\sum_{\tilde y_i=1}^{}\left(1-2b_{i}\right)+\frac{1}{2\nunlab_2}\sum_{\tilde y_i=-1}^{}\left(1-2b_{i}\right),
\] and define
\[
\tilde{\delta}\defeq\semisupest-\gamma\tilde\mu=\frac{1}{2\nunlab_1}\sum_{\tilde y_i=1}\varepsilon_i-\frac{1}{2\nunlab_2}\sum_{\tilde y_i=-1}\varepsilon_i.%+\frac{\|\mu\|}{\|\tilde\mu\|}(\frac{1}{\tilde n}\sum_{i=1}^{\tilde n}\tilde y_i)\tilde\nu
\]
%\[
%\tilde{\delta}\defeq\frac{\|\mu\|}{\|\tilde\mu\|}\semisupest-\gamma\mu=\gamma(\frac{\|\mu\|}{\|\tilde\mu\|}\tilde\mu-\mu)+\frac{\|\mu\|}{\|\tilde\mu\|}\frac{1}{\nunlab}\sum_{i=1}^{\nunlab}\yun_{i}\eps_{i}+\frac{\|\mu\|}{\|\tilde\mu\|}(\frac{1}{\tilde n}\sum_{i=1}^{\tilde n}\tilde y_i)\tilde\nu
%\]
We then have the decomposition and bound
\begin{align}
\frac{\left\Vert \semisupest\right\Vert ^{2}}{\left(\tilde\mu^{\top}\semisupest\right)^{2}} & =\frac{\left\Vert \tilde{\delta}+\gamma\tilde\mu\right\Vert ^{2}}{\left(\gamma\left\Vert \tilde\mu\right\Vert ^{2}+\tilde\mu^{\top}\tilde{\delta}\right)^{2}}=\frac{1}{\left\Vert \tilde\mu\right\Vert ^{2}}+\frac{\left\Vert \tilde{\delta}+\gamma\tilde\mu\right\Vert ^{2}-\frac{1}{\left\Vert \tilde\mu\right\Vert ^{2}}\left(\gamma\left\Vert \tilde\mu\right\Vert ^{2}+\tilde\mu^{\top}\tilde{\delta}\right)^{2}}{\left(\gamma\left\Vert \tilde\mu\right\Vert ^{2}+\tilde\mu^{\top}\tilde{\delta}\right)^{2}}\nonumber \\
 & =\frac{1}{\left\Vert \tilde\mu\right\Vert ^{2}}+\frac{\|\tilde{\delta}\|^{2}-\frac{1}{\left\Vert \tilde\mu\right\Vert ^{2}}\left(\tilde\mu^{\top}\tilde{\delta}\right)^{2}}{\left(\gamma\left\Vert\tilde \mu\right\Vert ^{2}+\tilde\mu^{\top}\tilde{\delta}\right)^2}\le\frac{1}{\left\Vert \tilde\mu\right\Vert ^{2}}+\frac{\|\tilde{\delta}\|^{2}}{\left\Vert \tilde\mu\right\Vert ^{4}\left(\gamma+\frac{1}{\left\Vert \tilde\mu\right\Vert ^{2}}\tilde\mu^{\top}\tilde{\delta}\right)^{2}}.\label{eq:semisup-deomp}
\end{align}

To write down concentration bounds for $\|\tilde{\delta}\|^{2}$ and
$\tilde\mu^{\top}\tilde{\delta}$ we must address their sub-Gaussianity.
To do so, write
\[
\tilde{\delta}=\frac{1}{2\sum_{i=1}^{\tilde n}1(\tilde y_i=1)}\sum_{i=1}^{\tilde n}1(\tilde y_i=1)\varepsilon_i-\frac{1}{2\sum_{i=1}^{\tilde n}1(\tilde y_i=-1)}\sum_{i=1}^{\tilde n}1(\tilde y_i=-1)\varepsilon_i,%+\frac{\|\mu\|}{\|\tilde\mu\|}(\frac{1}{\tilde n}\sum_{i=1}^{\tilde n}\tilde y_i)\tilde\nu
\]
and
$$
\yun_{i}\stackrel{i.i.d.}{\sim}  \mathrm{sign}\left((z_i\tilde\mu+\tilde\nu- \hat b_{intermediate}+\varepsilon_i)^{\top}\labest\right),
$$
$$
\yun_{i}\eps_{i}\stackrel{i.i.d.}{\sim}  \mathrm{sign}\left((z_i\tilde\mu+\tilde\nu- \tilde b_{intermediate}+\varepsilon_i)^{\top}\labest\right)\cdot\eps_i,
$$
where $z_i$ is the true label of $\xun_i$ (taken value from $\pm 1$).

We then have 
\begin{align*}
\E[1(\tilde y_i=1)]=&\bP((z_i\tilde\mu+\tilde\nu- \tilde b_{intermediate}+\varepsilon_i)^{\top}\labest>0)\\
=&\frac{1}{2}\bP((\tilde\mu_1- \tilde b_{intermediate}+\varepsilon_i)^{\top}\labest>0)+\frac{1}{2}\bP((\tilde\mu_2- \tilde b_{intermediate}+\varepsilon_i)^{\top}\labest>0)\\
=&\frac{1}{2}\bP((\tilde\mu_1-\frac{\mu_1+\mu_2}{2}+e_b+\varepsilon_i)^{\top}\labest>0)+\frac{1}{2}\bP((\tilde\mu_2-\frac{\mu_1+\mu_2}{2}+e_b+\varepsilon_i)^{\top}\labest>0)\\
\ge&\frac{1}{2}\bP((\tilde\mu_1-\frac{\mu_1+\mu_2}{2}+e_b)^{\top}\labest+\varepsilon_i^{\top}\labest>0)\\
=& \frac{1}{2}\bP(\varepsilon_i^{\top}\labest>-(\tilde\mu_1-\mu_1+\frac{\mu_1-\mu_2}{2}+e_b)^{\top}\labest)
%\ge& \frac{1}{2}\bP(\varepsilon_i^{\top}\frac{\mu_1-\mu_2}{2}>0)\ge \frac{c}{2}.
\end{align*}

The term in the last line can be bounded as follows. Let us recall $
d_\nu=\max\Big\{\frac{\|\tilde\mu_1-\mu_1\|}{\|\tilde\mu_1-\tilde\mu_2\|},\frac{\|\tilde\mu_2-\mu_2\|}{\|\tilde\mu_1-\tilde\mu_2\|}\Big\}<c_0$ implies that $
\|\tilde\mu-\mu\|< c_0\|\tilde\mu\|,
$ and therefore $\|\mu\|\ge\|\tilde\mu\|-c_0\|\tilde\mu\|=(1-c_0)\|\tilde\mu\|$. We then obtain 
$$
|(\tilde\mu_1-\mu_1)^\top\mu|\le\|\tilde\mu_1-\mu_1\|\cdot\|\mu\|\le c_0 \|\tilde\mu\|\cdot\|\mu\|\le\frac{c_0}{1-c_0}\|\mu\|^2.
$$

As a result, we have \begin{align*}
|(\tilde\mu_1-\mu_1+\frac{\mu_1-\mu_2}{2}+e_b)^{\top}\labest|=&|(\tilde\mu_1-\mu_1+\frac{\mu_1-\mu_2}{2}+e_b)^{\top}(\mu+e_w)|\\
\ge&\|\mu\|^2-|(\tilde\mu_1-\mu_1)^\top\mu|-|e_b^\top\mu|-|(\tilde\mu_1-\mu_1+\frac{\mu_1-\mu_2}{2}+e_b)^{\top}e_w|\\
\gtrsim& \Omega_p(\sqrt d).
\end{align*}

We then have
\begin{align}
\E[1(\tilde y_i=1)]
\notag\ge& \frac{1}{2}\bP(\varepsilon_i^{\top}\labest>-(\tilde\mu_1-\mu_1+\frac{\mu_1-\mu_2}{2}+e_b)^{\top}\labest)\\
\ge& \frac{1}{2}\bP(\varepsilon_i^{\top}\frac{\mu_1-\mu_2}{2}>0)\ge \frac{c}{2},\label{ineq:gamma}
\end{align}
for some constant $c$ close to $1$ when $d$ is sufficiently large. 

Therefore, we have
\begin{align*}
\frac{1}{\tilde n}\sum_{i=1}^{\tilde n}1(\tilde y_i=1)\ge c+o_p(1),
\end{align*}
and
$$
\|\frac{1}{\sum_{i=1}^{\tilde n}1(\tilde y_i=1)}\sum_{i=1}^{\tilde n}1(\tilde y_i=1)\varepsilon_i\|\lesssim\|\frac{1}{\tilde n}\sum_{i=1}^{\tilde n}1(\tilde y_i=1)\varepsilon_i\|
$$

In addition, we have
%$$
%|\E[\yun_{i}\eps_{i}^{(j)}]|\le|\E[(\mathring y_{i}-\yun_i)\eps_{i}^{(j)}]|\le\sqrt{\E[(\mathring y_{i}-\yun_i)^2]\E[\eps_i^2]}\lesssim |(\frac{1}{n})^{1/4}+\frac{\finalest^\top(\tilde\nu-\nu)}{\|\finalest\|}|\sigma:=\xi_n
%$$
$$
\|\E[1(\tilde y_i=1)\eps_{i}]\|=\|\E[\E[1(\tilde y_i=1)\eps_{i}\mid \labest^\top\varepsilon_i]]\|\le\E[\sqrt{d}+| \labest^\top\varepsilon_i]|]\lesssim\sqrt{d}.%|(\frac{1}{n})^{1/4}+\frac{\finalest^\top(\tilde\nu-\nu)}{\|\finalest\|}|\sigma:=\xi_n
$$

Since $\|1(\tilde y_i=1)\eps_{i}^{(j)}-\E[1(\tilde y_i=1)\eps_{i}^{(j)}]\|_{\psi_2}\le 2\|1(\tilde y_i=1)\eps_{i}^{(j)} \|_{\psi_2}\le C\|\eps_{i}^{(j)}\|_{\psi_2}\le C\sigma_{\max}$, we have 
$$
\bP\left(\left(\frac{1}{\nunlab}\sum_{i=1}^{\nunlab}1(\tilde y_i=1)\eps_{i}\supind j\right)^{2}\ge (\E[1(\tilde y_i=1)\eps_{i}^{(j)}])^2+t^2\cdot\sigma_{\max}^2\right)\le e^{-C\nunlab t^2}.
$$

Therefore, by union bound, with probability at least $1-d^{-1}$,
 $$
\|\frac{1}{\sum_{i=1}^{\tilde n}1(\tilde y_i=1)}\sum_{i=1}^{\tilde n}1(\tilde y_i=1)\varepsilon_i\|^2\lesssim\|\frac{1}{\nunlab}\sum_{i=1}^{\nunlab}1(\tilde y_i=1)\eps_{i}\|^2=\sum_{j=1}^{d}\left(\frac{1}{\nunlab}\sum_{i=1}^{\nunlab}1(\tilde y_i=1)\eps_{i}\supind j\right)^{2}\lesssim d+d\cdot\frac{\log d}{\tilde n}\sigma_{\max}^2.
$$
Similarly, we have 
 $$
\|\frac{1}{\sum_{i=1}^{\tilde n}1(\tilde y_i=-1)}\sum_{i=1}^{\tilde n}1(\tilde y_i=-1)\varepsilon_i\|^2\lesssim\|\frac{1}{\nunlab}\sum_{i=1}^{\nunlab}1(\tilde y_i=-1)\eps_{i}\|^2=\sum_{j=1}^{d}\left(\frac{1}{\nunlab}\sum_{i=1}^{\nunlab}1(\tilde y_i=-1)\eps_{i}\supind j\right)^{2}\lesssim d+d\cdot\frac{\log d}{\tilde n}\sigma_{\max}^2.
$$

Then, since $\|\tilde{\delta}\|^{2}\le \|\frac{1}{2\sum_{i=1}^{\tilde n}1(\tilde y_i=1)}\sum_{i=1}^{\tilde n}1(\tilde y_i=1)\varepsilon_i\|^2+\|\frac{1}{2\sum_{i=1}^{\tilde n}1(\tilde y_i=-1)}\sum_{i=1}^{\tilde n}1(\tilde y_i=-1)\varepsilon_i\|^2$, we have
$$
\|\tilde{\delta}\|^{2}=O_p(d\cdot(1+\frac{\log d}{\tilde n}\sigma^2)).
$$
%we have by the union bound
%\begin{equation*}
%\bP\left(\sum_{j=1}^{d}\left(\frac{1}{\nunlab}\sum_{i=1}^{\nunlab}\yun_{i}\eps_{i}\supind j\right)^{2}\ge d\xi_n^2+\frac{d\cdot\sigma^{2}}{\nunlab}t^2\right)\le\sum_{j=1}^d \bP\left(\left(\frac{1}{\nunlab}\sum_{i=1}^{\nunlab}\yun_{i}\eps_{i}\supind j\right)^{2}\ge d\xi_n^2+t^2\cdot\frac{\sigma^2}{\nunlab}\right) \le de^{-C t^2}.%\label{eq:delta-til-norm-bound}
%\end{equation*}
%Therefore
%\begin{equation}
%\bP\left(\|\tilde{\delta}\|^{2}\le\|\gamma(\frac{\|\mu\|}{\|\tilde\mu\|}\tilde\mu-\mu)\|^2+ \frac{d\cdot\sigma^{2}}{\nunlab}t^2+d\xi_n^2\right)\ge1- de^{-C t^2}.\label{eq:delta-til-norm-bound}
%\end{equation}

The same technique also yields a crude bound on $\tilde\mu^{\top}\tilde{\delta}=\frac{1}{2\nunlab_1}\sum_{i=1}^{\tilde n}1(\tilde y_i=1)\tilde\mu^{\top}\varepsilon_i-\frac{1}{2\nunlab_2}\sum_{i=1}1(\tilde y_i=-1)\tilde\mu^{\top}\varepsilon_i$. We can write%=\gamma\mu^\top(\frac{\|\mu\|}{\|\tilde\mu\|}\tilde\mu-\mu)+\frac{1}{\nunlab}\sum_{i=1}^{\nunlab}\yun_{i}\mu^{\top}\eps_{i}$, as we have
$$
1(\yun_{i}=1)\tilde\mu^\top\eps_{i}\stackrel{i.i.d.}{\sim}   1\left((z_i\tilde\mu+\tilde\nu-\hat b_{intermediate}+\varepsilon_i
)^{\top}\labest>0\right)\cdot\tilde\mu^\top\eps_i. %1\left((z_i\mu+\varepsilon_i)^{\top}\labest>0\right)\cdot\tilde\mu^\top\eps_i.
$$

Since $\|1(\yun_{i}=1)\tilde\mu^\top\eps_{i}\|_{\vartheta_2}\le C\|\tilde\mu^\top\eps_{i}^{(j)}\|_{\vartheta_2}\le C\|\tilde\mu\|_2\sigma$, we have 
$$
\bP\left(\left(\frac{1}{\nunlab}\sum_{i=1}^{\nunlab}1(\yun_{i}=1)\tilde\mu^\top\eps_{i} \right)^{2}\ge t^2\cdot\|\tilde\mu\|^2\sigma^2+\|\tilde\mu\|^2\sigma^2\right)\le e^{-C\nunlab t^2}.
$$
and by the fact that $\left(\frac{1}{\nunlab_1}\sum_{i=1}^{\nunlab}1(\yun_{i}=1)\tilde\mu^\top\eps_{i} \right)^2\lesssim\left(\frac{1}{\nunlab}\sum_{i=1}^{\nunlab}1(\yun_{i}=1)\tilde\mu^\top\eps_{i} \right)^2$, we have 
\[
\bP\left(\left|\tilde\mu^{\top}\tilde{\delta}\right|\ge\sqrt{2}\sigma\left\Vert \tilde\mu\right\Vert+\|\tilde\mu\|\sigma \right)=\bP\left(\left|\tilde\mu^{\top}\tilde{\delta}\right|^{2}\ge C\sigma^{2}\left\Vert \tilde\mu\right\Vert ^{2}\right)\le e^{-\nunlab/8}.
\]
Finally, we need to argue that $\gamma$ is not too small. Recall
that $\gamma=\frac{1}{2\nunlab_1}\sum_{\tilde y_i=1}^{}\left(1-2b_{i}\right)+\frac{1}{2\nunlab_2}\sum_{\tilde y_i=-1}^{}\left(1-2b_{i}\right)$
 where
$b_{i}$ is the indicator that $\yun_{i}$ is incorrect and therefore
\begin{align*}
\E\left[1-2b_i\mid\labest,\tilde y_i=1\right]&=1-2\bP(f_{\labest}\mid \tilde x\sim subGaussian(\tilde\mu_1,\sigma^2))\\
&=2\bP(\varepsilon_i^{\top}\labest>-(\tilde\mu_1-\mu_1+\frac{\mu_1-\mu_2}{2}+e_b)^{\top}\labest)-1.
\end{align*}
This term can be lower bounded similarly as \eqref{ineq:gamma}, which satisfies 
\begin{align*}
&\E\left[1-2b_i\mid\labest,\tilde y_i=1\right]&\\
=&2\bP(\varepsilon_i^{\top}\labest>-(\tilde\mu_1-\mu_1+\frac{\mu_1-\mu_2}{2}+e_b)^{\top}\labest)-1\ge\frac{4}{5},
\end{align*}
with high probability when $d$ is sufficiently large.

Similarly, we have
\[
\E\left[1-2b_i\mid\labest,\tilde y_i=-1\right]\ge\frac{4}{5},
\]
with high probability when $d$ is sufficiently large.

%\[
%\E\left[\gamma\mid\labest,\tilde y_i=1\right]=1-2\err(f_{\labest}),
%\]
Therefore we expect $\gamma$ to be reasonably large as long as $\E[\gamma]\ge\frac{4}{5}$.
Indeed, define $$
\tilde\gamma=\frac{1}{\tilde n}\sum_{i=1}^n(1-2b_i).
$$
We then have \begin{align*}
\E[\tilde\gamma]\ge&\E[\frac{1}{\tilde n} \sum_{y_i=1}(1-2b_i)+\frac{1}{\tilde n} \sum_{y_i=-1}(1-2b_i)]\\
\ge&\E[\frac{1}{\tilde n}\cdot \frac{4}{5}\tilde n_1+\frac{1}{\tilde n} \cdot \frac{4}{5}\tilde n_2]\ge\frac{4}{5}.
\end{align*}

By using $\gamma\ge\frac{1}{2}\tilde\gamma$, we have
\begin{align*}
&\bP(\gamma\ge\frac{1}{5})\ge\bP(\tilde\gamma\ge\frac{2}{5})=1-\bP(\tilde\gamma<\frac{2}{5})\\
\ge& 1-\bP(|\tilde\gamma-\E[\tilde\gamma]|>\frac{2}{5})\ge 1- e^{-c\tilde n},
\end{align*}
where the last inequality is due to Hoeffding's inequality.

As a result, we have $\gamma \ge \frac{2}{5}$ with high probability.

%\begin{align*}
%\bP\left(\gamma<\frac{1}{6}\right) & =\bP\left(\frac{1}{\nunlab}\sum_{i=1}^{\nunlab}\left(1-2b_{i}\right)<\frac{1}{6}\right)\\
% & \le\bP\left(\err(f_{\labest})>\frac{1}{3}\right)+\bP\left(\frac{1}{\nunlab}\sum_{i=1}^{\nunlab}b_{i}<\frac{5}{12}\mid\err(f_{\labest})\le\frac{1}{3}\right).
%\end{align*}
%Note that
%\[
%\frac{1}{3}\ge Q\left(\frac{1}{2}\right)\ge Q\left(\left[2\left(1+\sqrt{\nn/d}\right)\right]^{-1/2}\right)
%\]
%Therefore, by Lemma \ref{lem:supervised-perf}, for sufficiently large
%$d/\nn$,
%\[
%\bP\left(\err(f_{\labest})>\frac{1}{3}\right)\le e^{-c\|\mu\|^2/8\sigma^2}+2e^{-cn\|\mu\|/2\sigma}%e^{-c\cdot\min\left\{ \sqrt{d/\nn},\nn\left(d/\nn\right)^{1/4}\right\} }
%\]
%for some constant $c$. Moreover, by Bernoulli concentration (Hoeffding's
%inequality) we have that
%\[
%\bP\left(\frac{1}{\nunlab}\sum_{i=1}^{\nunlab}b_{i}<\frac{5}{12}\mid\err(f_{\labest})\le\frac{1}{3}\right)\le e^{-2\nunlab\left(\frac{5}{12}-\frac{1}{3}\right)^{2}}=e^{-\nunlab/72}.
%\]
Define the event,
\[
\mathcal{E}=\left\{ \|\tilde{\delta}\|^{2}\le  \|\gamma(\frac{\|\mu\|}{\|\tilde\mu\|}\tilde\mu-\mu)\|^2+ \frac{d\cdot\sigma_{\max}^{2}}{\nunlab}{\log d}+d\xi_n^2,\ \left|\mu^{\top}\tilde{\delta}\right|\le\sqrt{2}\sigma_{\max}\left\Vert \mu\right\Vert+\gamma\mu^\top(\frac{\|\mu\|}{\|\tilde\mu\|}\tilde\mu-\mu)+\xi_n\|\mu\| \ \text{ and }\gamma\ge\frac{2}{5}\right\} ;
\]
by the preceding discussion,
\[
\bP\left(\mathcal{E}^{C}\right)\le \frac{1}{d}+e^{-\nunlab/8}+e^{-c\|\mu\|^2/8\sigma_{\max}^2}+2e^{-cn\|\mu\|/2\sigma_{\max}}+e^{-c\tilde n}%2e^{-\nunlab/8}+e^{-\left(d-1\right)/8}+e^{-c\cdot\min\left\{ \sqrt{d/\nn},\nn\left(d/\nn\right)^{1/4}\right\} }+e^{-\nunlab/72}\le e^{-\tilde{c}_{2}\min\left\{ \nunlab,\sqrt{d/\nn},\nn\left(d/\nn\right)^{1/4}\right\} }.
\]
%
%$$
% \|\gamma(\frac{\|\mu\|}{\|\tilde\mu\|}\tilde\mu-\mu)\|^2\le 2\|\mu\|^2-2\frac{\|\mu\|}{\|\tilde\mu\|}\tilde\mu^\top\mu=2\|\mu\|^2(1-\frac{\tilde\mu^\top\mu}{\|\mu\|\|\tilde\mu\|})\le2(1-r_{\mu})\|\mu\|^2.
%$$
%$$
%|\gamma\mu^\top(\frac{\|\mu\|}{\|\tilde\mu\|}\tilde\mu-\mu) |\le \gamma (1-r_{\mu})\|\mu\|^2
%$$
%$$
%\xi_n=\sigma\cdot(1/{\tilde n}^{1/4}+d_{\nu})
%$$

Moreover, by the bound (\ref{eq:semisup-deomp}), $\mathcal{E}$ implies
\begin{align*}
\frac{\left\Vert \semisupest\right\Vert ^{2}}{\left(\tilde\mu^{\top}\semisupest\right)^{2}} & \le\frac{1}{\left\Vert \tilde\mu\right\Vert ^{2}}+\frac{\|\tilde{\delta}\|^{2}}{\left\Vert \tilde\mu\right\Vert ^{4}\left(\gamma+\frac{1}{\left\Vert \tilde\mu\right\Vert ^{2}}\tilde\mu^{\top}\tilde{\delta}\right)^{2}}\le\frac{1}{\left\Vert \tilde\mu\right\Vert ^{2}}+\frac{d\cdot(1+\frac{\log d}{\tilde n}\sigma_{\max}^2)}{\left\Vert \tilde\mu\right\Vert ^{4}\left(\frac{2}{5}+ \frac{1}{\left\Vert \tilde\mu\right\Vert ^{2}}\cdot\|\tilde\mu\|\sigma_{\max}\right)^{2}}.
\end{align*}

Therefore, \[
\frac{\tilde\mu^{\top}\finalest}{\sigma_{\max}\left\Vert \finalest\right\Vert }\ge\left(\frac{\sigma_{\max}^2}{\left\Vert \tilde\mu\right\Vert ^{2}}+\frac{d\cdot(1+\frac{\log d}{\tilde n}\sigma_{\max}^2)}{\left\Vert \tilde\mu\right\Vert ^{4}\left(\frac{2}{5}-\frac{\sigma_{\max}}{\|\tilde\mu\|_2}\right)^{2}}\right)^{-1/2}
\]
with probability $\ge1-(\frac{1}{d}+e^{-\nunlab/8}+e^{-c\|\mu\|^2/8\sigma_{\max}^2}+2e^{-cn\|\mu\|/2\sigma_{\max}})$.

Recall that we take  $\sigma_{\max}:=C_\sigma d^{1/4}$ and $\|\tilde\mu_1-\tilde\mu_2\|_2=C_\mu \sqrt d$ for sufficiently large $C_\mu$, we than have when $\tilde n\gtrsim \varepsilon^2 d\log d$, $$
\frac{\tilde\mu^{\top}\finalest}{\left\Vert \finalest\right\Vert }=\Omega_P(\sqrt d).
$$

Then let us consider
\begin{align*}
\hat b_{\text{final}}=&\frac{1}{2\nunlab_1}\sum_{\tilde y_i=1}^{}\xun_{i}+\frac{1}{2\nunlab_2}\sum_{\tilde y_i=-1}^{}\xun_{i}\\
=&\tilde\nu+\frac{1}{2\nunlab_1}\sum_{\tilde y_i=1}^{}\left(1-2b_{i}\right)\tilde\mu-\frac{1}{2\nunlab_2}\sum_{\tilde y_i=-1}^{}\left(1-2b_{i}\right)\tilde\mu+\frac{1}{2\nunlab_1}\sum_{\tilde y_i=1}\varepsilon_i+\frac{1}{2\nunlab_2}\sum_{\tilde y_i=-1}\varepsilon_i\\
=&\tilde\nu+\left[\frac{1}{2\nunlab_1}\sum_{\tilde y_i=1}^{}\left(1-2b_{i}\right)-\frac{1}{2\nunlab_2}\sum_{\tilde y_i=-1}^{}\left(1-2b_{i}\right)\right]\tilde\mu+\frac{1}{2\nunlab_1}\sum_{\tilde y_i=1}\varepsilon_i+\frac{1}{2\nunlab_2}\sum_{\tilde y_i=-1}\varepsilon_i,
%&{\color{red}=(\frac{1}{\tilde n}\sum_{i=1}^{\tilde n} \tilde y_i)\tilde\nu+\gamma\tilde\mu+\frac{1}{\nunlab}\sum_{i=1}^{\nunlab}\yun_{i}\eps_{i}}
\end{align*}

Let
\[
\lambda\defeq\frac{1}{\nunlab_1}\sum_{\tilde y_i=1}^{}\left(1-2b_{i}\right)-\frac{1}{\nunlab_2}\sum_{\tilde y_i=-1}^{}\left(1-2b_{i}\right)
\] 
When $n>C$ for sufficiently large $C$, we have $\lambda\le0.01$.

Also, let us denote $\tilde\delta_2=\frac{1}{2\nunlab_1}\sum_{\tilde y_i=1}\varepsilon_i+\frac{1}{2\nunlab_2}\sum_{\tilde y_i=-1}\varepsilon_i$, we then have $$
|\tilde\delta^\top\tilde\delta_2|=\|\frac{1}{2\nunlab_1}\sum_{\tilde y_i=1}\varepsilon_i\|^2-\|\frac{1}{2\nunlab_2}\sum_{\tilde y_i=-1}\varepsilon_i\|^2\le\|\frac{1}{2\nunlab_1}\sum_{\tilde y_i=1}\varepsilon_i\|^2\le c_\varepsilon(d+d\cdot\frac{\log d}{\tilde n}\sigma_{\max}^2)
$$

We also have 
\begin{align*}
|\frac{\finalest^\top}{\|\finalest\|}d_b|\le& \lambda|\frac{\finalest^\top}{\|\finalest\|}\tilde\mu|+|\frac{\finalest^\top}{\|\finalest\|}\tilde\delta_2|\\
\le&\lambda\|\tilde\mu\|+|\frac{( \tilde{\delta}+\gamma\tilde\mu)^\top}{\| \tilde{\delta}+\gamma\tilde\mu\|}\tilde\delta_2|\\
\le&\lambda\|\tilde\mu\|+\frac{| \tilde{\delta}^\top \tilde\delta_2+\gamma\tilde\mu^\top\tilde\delta_2|}{\|\gamma\tilde\mu\|-\|\tilde{\delta}\|}\\
\le&\lambda\|\tilde\mu\|+\frac{c_\varepsilon(d+d\cdot\frac{\log d}{\tilde n}\sigma^2)+O_P(\|\tilde\mu\|\sigma)}{\gamma\|\tilde\mu\|-c_\varepsilon(d+d\cdot\frac{\log d}{\tilde n}\sigma^2)}\\
\le& (\lambda C_\mu+\frac{c_\varepsilon}{\gamma C_\mu-c_\varepsilon})\cdot\sqrt d
\end{align*}

Therefore, when the constant $C_\mu$ is sufficiently large,
\begin{align*}
&\frac{\finalest^\top}{\|\finalest\|}((\mu_1-\tilde\mu_1)+\frac{\tilde\mu_1-\tilde\mu_2}{2}-d_b)\\
\ge& |\frac{\finalest^\top}{\|\finalest\|}\tilde\mu|-|\frac{\finalest^\top}{\|\finalest\|}(\mu_1-\tilde\mu_1)|-|\frac{\finalest^\top}{\|\finalest\|}d_b|\\
\ge&\left(\frac{\sigma^2}{\left\Vert \tilde\mu\right\Vert ^{2}}+\frac{d\cdot(1+\frac{\log d}{\tilde n}\sigma_{\max}^2)}{\left\Vert \tilde\mu\right\Vert ^{4}\left(\frac{1}{6}-\frac{\sigma_{\max}}{\|\tilde\mu\|_2}\right)^{2}}\right)^{-1/2}-2c_0\|\tilde\mu\|-(\lambda C_\mu+\frac{c_\varepsilon}{\gamma C_\mu-c_\varepsilon})\cdot\sqrt d\\
=&\Omega_P(\sqrt d).
\end{align*}

%$$\bP\Big(\big\|\hat b-\bE[\hat b]\big\|\geq c\sigma \big(\frac{\sqrt{d}\tilde{L}}{\sqrt{n/2}}+L\sqrt{\frac{2\log(2/\delta)}{n}}\big) \Big)\leq \delta.$$

%Recall that
%$r_\mu=|\frac{(\mu_1-\mu_2)^\top(\tilde\mu_1-\tilde\mu_2)}{\|\mu_1-\mu_2\|\|\tilde\mu_1-\tilde\mu_2\|}|$, $d_\nu=\frac{(\mu_1+\mu_2)^\top(\tilde\mu_1+\tilde\mu_2-\mu_1-\mu_2)/2}{\|\mu_1+\mu_2\|}$

%{\color{red}
%$$
%\bP\Big(\frac{\finalest^\top}{\|\finalest\|} \varepsilon\leq -\frac{\finalest^\top}{\|\finalest\|}(\mu_1-\hat b) \Big)=\bP\Big(\frac{\finalest^\top}{\|\finalest\|} \varepsilon\leq -\frac{\finalest^\top}{\|\finalest\|}((\mu_1-\tilde\mu_1)+\frac{\tilde\mu_1-\tilde\mu_2}{2}-d_b) \Big)
%$$}

The robust error is then
\begin{align*}
&\bP\Big(\frac{\finalest^\top}{\|\finalest\|} \varepsilon\leq -\frac{\finalest^\top}{\|\finalest\|}(\mu_1-\hat b) +L_1'\cdot\varepsilon\sqrt d\Big)\\
=&\bP\Big(\frac{\finalest^\top}{\|\finalest\|} \varepsilon\leq -\frac{\finalest^\top}{\|\finalest\|}((\mu_1-\tilde\mu_1)+\frac{\tilde\mu_1-\tilde\mu_2}{2}-d_b)+L_1'\cdot\varepsilon\sqrt d \Big)\\
\le& \exp(-C\sqrt{d})\le0.01,
\end{align*}
when $d$ is sufficiently large. 

{\subsection{Proof of Proposition 1}
The proof of Proposition 1 is very similar to those of Theorem 3 except for the tail probabilities changed from subgaussian to $g(\cdot)$. For completeness, we present the proof below.
 
 We first recall the definition of $D_g$: \begin{align*}
D_g(\mu,\sigma^2)=\{& X\in\R^d: \forall v\in\R^d, \|v\|_2=1, \text{Var}(X_j)\le\sigma^2\\
 &\Pro(|v^T(X-\mu)|>\sigma\cdot t)\le g(t) \},
    \end{align*}
    and restate Proposition 1.
    
   \textbf{Proposition 1}
     Suppose $D_g$ is closed under independent summation, and assume $\left\|\E\left[\tilde x_i-\E[\tilde x_i]\mid a^T \tilde x_i=b\right]\right\|\lesssim \sqrt d + |b|$ for fixed unit vector $a$, $\tilde\sigma\le\sigma_{\max}\asymp d^{1/4}$, %\james{explain what is $\sigma_{\max}$}
  $\|\tilde\mu_1-\tilde\mu_2\|_2\asymp \sqrt d$,  $c<q<1-c$ for some constant $0<c<1/2$, and 
 $$d_\nu=\max\Big\{\frac{\|\tilde\mu_1-\mu_1\|}{\|\tilde\mu_1-\tilde\mu_2\|},\frac{\|\tilde\mu_2-\mu_2\|}{\|\tilde\mu_1-\tilde\mu_2\|}\Big\}<c_0,$$
 for some constant $c_0\le 1/4$,
     then    the robust classification error is at most $1\%$ when $d$ is sufficiently large, $n\ge C$ for some constant C (not depending on $d$ and $\varepsilon$) and $$\tilde n\gtrsim \varepsilon^2\cdot(g^{-1}(1/d\log d))^2\cdot\sqrt d.$$

Now let us proceed to the proof. 

 We first recall the distributions for the two classes of labeled data as ${D_g}(\mu_1,\sigma_{\max}^2)$, and ${D_g}(\mu_2,\sigma_{\max}^2)$ respectively. Similarly, we also denote the distributions for the two classes of unlabeled data as ${D_g}(\bm\tilde\mu_1,\sigma_{\max}^2)$, and ${D_g}(\bm\tilde\mu_2,\sigma_{\max}^2)$ respectively. Also, to avoid the visual similarity and emphasize the estimates constructed by the labeled and unlabeled data respectively, we write $\hat w$ as $\hat w_{\text{intermediate}}$, $\hat b$ as $\hat b_{\text{intermediate}}$, $\tilde w$ as $\tilde w_{\text{final}}$ and $\tilde b$ as $\tilde b_{\text{final}}$.

%subgaussian distribution %class in $\R^d$ by $$
%{subGaussian}(\bm\mu,\sigma^2)=\{\bm X\in\R^d: \E[\bm X]=\bm\mu, \|\bm X-\bm\mu\|_{\psi_2}\le\sigma^2\}
%$$

Then, let us write out the robust error of misclassifying class 1 against the $\ell_\infty$ attack  (the robust error of misclassifying class 2 can be bounded similarly) as
%\begin{align}\label{eq:misclassification}
%\eta_1=&\bP\Big(\frac{\finalest^\top}{\|\finalest\|} ( x-\hat{b})\leq 0 \mid  x\sim subGaussian(\mu_1,\sigma^2)\Big)
%=\bP\Big(\frac{\finalest^\top}{\|\finalest\|} \varepsilon\leq -\frac{\finalest^\top}{\|\finalest\|}(\mu_1-\hat b) \Big)
%%=&\bP\Big(\frac{\finalest^\top}{\|\finalest\|} \varepsilon\leq -\frac{\tilde\mu^{\top}\finalest}{\left\Vert \finalest\right\Vert }+ d_\nu\Big).
%%\le&  \exp\left(d_\nu-\left(\frac{\sigma^2}{\left\Vert \mu\right\Vert ^{2}}+\frac{2\sigma^{4}d\sqrt{\log d}/\nunlab+2\sigma^2(1-r_\mu)\|\mu\|^2+d\sigma^2({\tilde n}^{-1/2}+d_{\nu})}{\left\Vert \mu\right\Vert ^{4}\left(\frac{1}{6}r_\mu-\frac{\sqrt{2}\sigma}{\left\Vert \mu\right\Vert }-\sigma\cdot(1/{\tilde n}^{1/4}+d_{\nu})\right)^{2}}\right)^{-1/2}\right)
%\end{align}
%Similarly, the robust error against the $\ell_\infty$ attack is then 
\begin{align*}
&\max_{\|\delta\|_\infty\le L_1'\varepsilon}\bP\Big(\frac{\finalest^\top}{\|\finalest\|} (\tilde x+\delta-\tilde{b}_{\text{final}})\leq 0 \mid \tilde x\sim D_g(\tilde\mu_1,\sigma_{\max}^2)\Big)\\
=&\bP\Big(\frac{\finalest^\top}{\|\finalest\|} \varepsilon\leq -\frac{\finalest^\top}{\|\finalest\|}(\mu_1-  \tilde b_{\text{final}}) +\varepsilon\frac{\|\finalest\|_1}{\|\finalest\|}\Big)\\
=&\bP\Big(\frac{\finalest^\top}{\|\finalest\|} \varepsilon\leq -\frac{\finalest^\top}{\|\finalest\|}(\mu_1-  \tilde b_{\text{final}}) +L_1'\cdot\varepsilon\sqrt d\Big)
\end{align*}

Denote $  \tilde b_{\text{final}}:=\frac{\tilde\mu_1+\tilde\mu_2}{2}+d_b$, we then have 
\begin{align*}
&\bP\Big(\frac{\finalest^\top}{\|\finalest\|} \varepsilon\leq -\frac{\finalest^\top}{\|\finalest\|}(\mu_1-  \tilde b_{\text{final}}) +L_1'\cdot\varepsilon\sqrt d\Big)\\
=&\bP\Big(\frac{\finalest^\top}{\|\finalest\|} \varepsilon\leq -\frac{\finalest^\top}{\|\finalest\|}((\mu_1-\tilde\mu_1)+\frac{\tilde\mu_1-\tilde\mu_2}{2}-d_b)+L_1'\cdot\varepsilon\sqrt d \Big)
\end{align*}

We are going to bound
$|\frac{\finalest^\top}{\|\finalest\|}(\mu_1-\tilde\mu_1)|$, $|\frac{\finalest^\top}{\|\finalest\|}(\tilde\mu_1-\tilde\mu_2)/2|$, and $|\frac{\finalest^\top}{\|\finalest\|}d_b|$ respectively.

Let $\tilde\mu=(\tilde\mu_1-\tilde\mu_2)/2$, $\tilde\nu=(\tilde\mu_1+\tilde\mu_2)/2$, $\mu=(\mu_1-\mu_2)/2$, $\nu=(\mu_1+\mu_2)/2$, and $b_{i}$ be the indicator that the $i$th pseudo-label $\tilde y_i$ is incorrect, so that $\xun_{i}\sim\tilde\nu+{D_g}\left(\left(1-2b_{i}\right)\yun_{i}\tilde\mu,\sigma^{2}\right)$

Let $\tilde n_1=\sum_{i=1}^{\tilde n}1\{\tilde y_i=1\}$, $\tilde n_2=\sum_{i=1}^{\tilde n}1\{\tilde y_i=-1\}$. 
We recall  the final direction estimator as
%\[
%\semisupest=\frac{1}{\nunlab}\sum_{i=1}^{\nunlab}\yun_{i}\xun_{i}=(\frac{1}{\tilde n}\sum_{i=1}^{\tilde n} \tilde y_i)\tilde\nu+\gamma\tilde\mu+\frac{1}{\nunlab}\sum_{i=1}^{\nunlab}\yun_{i}\eps_{i}
%\]
\begin{align*}
\semisupest=&\frac{1}{2\nunlab_1}\sum_{\tilde y_i=1}^{}\xun_{i}-\frac{1}{2\nunlab_2}\sum_{\tilde y_i=-1}^{}\xun_{i}\\
=&\frac{1}{2\nunlab_1}\sum_{\tilde y_i=1}^{}\left(1-2b_{i}\right)\tilde\mu+\frac{1}{2\nunlab_1}\sum_{\tilde y_i=1}\varepsilon_i+\frac{1}{2\nunlab_2}\sum_{\tilde y_i=-1}^{}\left(1-2b_{i}\right)\tilde\mu-\frac{1}{2\nunlab_2}\sum_{\tilde y_i=-1}\varepsilon_i,
%&{\color{red}=(\frac{1}{\tilde n}\sum_{i=1}^{\tilde n} \tilde y_i)\tilde\nu+\gamma\tilde\mu+\frac{1}{\nunlab}\sum_{i=1}^{\nunlab}\yun_{i}\eps_{i}}
\end{align*}
where $\eps_{i}\sim{D_g}\left(0,\sigma^{2}\right)$ independent
of each other. 

Now let  
\[
\gamma\defeq\frac{1}{2\nunlab_1}\sum_{\tilde y_i=1}^{}\left(1-2b_{i}\right)+\frac{1}{2\nunlab_2}\sum_{\tilde y_i=-1}^{}\left(1-2b_{i}\right),
\] and define
\[
\tilde{\delta}\defeq\semisupest-\gamma\tilde\mu=\frac{1}{2\nunlab_1}\sum_{\tilde y_i=1}\varepsilon_i-\frac{1}{2\nunlab_2}\sum_{\tilde y_i=-1}\varepsilon_i.%+\frac{\|\mu\|}{\|\tilde\mu\|}(\frac{1}{\tilde n}\sum_{i=1}^{\tilde n}\tilde y_i)\tilde\nu
\]
%\[
%\tilde{\delta}\defeq\frac{\|\mu\|}{\|\tilde\mu\|}\semisupest-\gamma\mu=\gamma(\frac{\|\mu\|}{\|\tilde\mu\|}\tilde\mu-\mu)+\frac{\|\mu\|}{\|\tilde\mu\|}\frac{1}{\nunlab}\sum_{i=1}^{\nunlab}\yun_{i}\eps_{i}+\frac{\|\mu\|}{\|\tilde\mu\|}(\frac{1}{\tilde n}\sum_{i=1}^{\tilde n}\tilde y_i)\tilde\nu
%\]
We then have the decomposition and bound
\begin{align}
\frac{\left\Vert \semisupest\right\Vert ^{2}}{\left(\tilde\mu^{\top}\semisupest\right)^{2}} & =\frac{\left\Vert \tilde{\delta}+\gamma\tilde\mu\right\Vert ^{2}}{\left(\gamma\left\Vert \tilde\mu\right\Vert ^{2}+\tilde\mu^{\top}\tilde{\delta}\right)^{2}}=\frac{1}{\left\Vert \tilde\mu\right\Vert ^{2}}+\frac{\left\Vert \tilde{\delta}+\gamma\tilde\mu\right\Vert ^{2}-\frac{1}{\left\Vert \tilde\mu\right\Vert ^{2}}\left(\gamma\left\Vert \tilde\mu\right\Vert ^{2}+\tilde\mu^{\top}\tilde{\delta}\right)^{2}}{\left(\gamma\left\Vert \tilde\mu\right\Vert ^{2}+\tilde\mu^{\top}\tilde{\delta}\right)^{2}}\nonumber \\
 & =\frac{1}{\left\Vert \tilde\mu\right\Vert ^{2}}+\frac{\|\tilde{\delta}\|^{2}-\frac{1}{\left\Vert \tilde\mu\right\Vert ^{2}}\left(\tilde\mu^{\top}\tilde{\delta}\right)^{2}}{\left(\gamma\left\Vert\tilde \mu\right\Vert ^{2}+\tilde\mu^{\top}\tilde{\delta}\right)^2}\le\frac{1}{\left\Vert \tilde\mu\right\Vert ^{2}}+\frac{\|\tilde{\delta}\|^{2}}{\left\Vert \tilde\mu\right\Vert ^{4}\left(\gamma+\frac{1}{\left\Vert \tilde\mu\right\Vert ^{2}}\tilde\mu^{\top}\tilde{\delta}\right)^{2}}.\label{eq:semisup-deomp}
\end{align}

To write down concentration bounds for $\|\tilde{\delta}\|^{2}$ and
$\tilde\mu^{\top}\tilde{\delta}$ we must control their tail bound.
To do so, write
\[
\tilde{\delta}=\frac{1}{2\sum_{i=1}^{\tilde n}1(\tilde y_i=1)}\sum_{i=1}^{\tilde n}1(\tilde y_i=1)\varepsilon_i-\frac{1}{2\sum_{i=1}^{\tilde n}1(\tilde y_i=-1)}\sum_{i=1}^{\tilde n}1(\tilde y_i=-1)\varepsilon_i,%+\frac{\|\mu\|}{\|\tilde\mu\|}(\frac{1}{\tilde n}\sum_{i=1}^{\tilde n}\tilde y_i)\tilde\nu
\]
and
$$
\yun_{i}\stackrel{i.i.d.}{\sim}  \mathrm{sign}\left((z_i\tilde\mu+\tilde\nu- \hat b_{intermediate}+\varepsilon_i)^{\top}\labest\right),
$$
$$
\yun_{i}\eps_{i}\stackrel{i.i.d.}{\sim}  \mathrm{sign}\left((z_i\tilde\mu+\tilde\nu- \tilde b_{intermediate}+\varepsilon_i)^{\top}\labest\right)\cdot\eps_i,
$$
where $z_i$ is the true label of $\xun_i$ (taken value from $\pm 1$).

We then have 
\begin{align*}
\E[1(\tilde y_i=1)]=&\bP((z_i\tilde\mu+\tilde\nu- \tilde b_{intermediate}+\varepsilon_i)^{\top}\labest>0)\\
=&\frac{1}{2}\bP((\tilde\mu_1- \tilde b_{intermediate}+\varepsilon_i)^{\top}\labest>0)+\frac{1}{2}\bP((\tilde\mu_2- \tilde b_{intermediate}+\varepsilon_i)^{\top}\labest>0)\\
=&\frac{1}{2}\bP((\tilde\mu_1-\frac{\mu_1+\mu_2}{2}+e_b+\varepsilon_i)^{\top}\labest>0)+\frac{1}{2}\bP((\tilde\mu_2-\frac{\mu_1+\mu_2}{2}+e_b+\varepsilon_i)^{\top}\labest>0)\\
\ge&\frac{1}{2}\bP((\tilde\mu_1-\frac{\mu_1+\mu_2}{2}+e_b)^{\top}\labest+\varepsilon_i^{\top}\labest>0)\\
=& \frac{1}{2}\bP(\varepsilon_i^{\top}\labest>-(\tilde\mu_1-\mu_1+\frac{\mu_1-\mu_2}{2}+e_b)^{\top}\labest)
%\ge& \frac{1}{2}\bP(\varepsilon_i^{\top}\frac{\mu_1-\mu_2}{2}>0)\ge \frac{c}{2}.
\end{align*}

The term in the last line can be bounded as follows. Let us recall $
d_\nu=\max\Big\{\frac{\|\tilde\mu_1-\mu_1\|}{\|\tilde\mu_1-\tilde\mu_2\|},\frac{\|\tilde\mu_2-\mu_2\|}{\|\tilde\mu_1-\tilde\mu_2\|}\Big\}<c_0$ implies that $
\|\tilde\mu-\mu\|< c_0\|\tilde\mu\|,
$ and therefore $\|\mu\|\ge\|\tilde\mu\|-c_0\|\tilde\mu\|=(1-c_0)\|\tilde\mu\|$. We then obtain 
$$
|(\tilde\mu_1-\mu_1)^\top\mu|\le\|\tilde\mu_1-\mu_1\|\cdot\|\mu\|\le c_0 \|\tilde\mu\|\cdot\|\mu\|\le\frac{c_0}{1-c_0}\|\mu\|^2.
$$

As a result, we have \begin{align*}
|(\tilde\mu_1-\mu_1+\frac{\mu_1-\mu_2}{2}+e_b)^{\top}\labest|=&|(\tilde\mu_1-\mu_1+\frac{\mu_1-\mu_2}{2}+e_b)^{\top}(\mu+e_w)|\\
\ge&\|\mu\|^2-|(\tilde\mu_1-\mu_1)^\top\mu|-|e_b^\top\mu|-|(\tilde\mu_1-\mu_1+\frac{\mu_1-\mu_2}{2}+e_b)^{\top}e_w|\\
\gtrsim& \Omega_p(\sqrt d).
\end{align*}

We then have
\begin{align}
\E[1(\tilde y_i=1)]
\notag\ge& \frac{1}{2}\bP(\varepsilon_i^{\top}\labest>-(\tilde\mu_1-\mu_1+\frac{\mu_1-\mu_2}{2}+e_b)^{\top}\labest)\\
\ge& \frac{1}{2}\bP(\varepsilon_i^{\top}\frac{\mu_1-\mu_2}{2}>0)\ge \frac{c}{2},\label{ineq:gamma}
\end{align}
for some constant $c$ close to $1$ when $d$ is sufficiently large. 

Therefore, we have
\begin{align*}
\frac{1}{\tilde n}\sum_{i=1}^{\tilde n}1(\tilde y_i=1)\ge c+o_p(1),
\end{align*}
and
$$
\|\frac{1}{\sum_{i=1}^{\tilde n}1(\tilde y_i=1)}\sum_{i=1}^{\tilde n}1(\tilde y_i=1)\varepsilon_i\|\lesssim\|\frac{1}{\tilde n}\sum_{i=1}^{\tilde n}1(\tilde y_i=1)\varepsilon_i\|
$$

In addition, we have
%$$
%|\E[\yun_{i}\eps_{i}^{(j)}]|\le|\E[(\mathring y_{i}-\yun_i)\eps_{i}^{(j)}]|\le\sqrt{\E[(\mathring y_{i}-\yun_i)^2]\E[\eps_i^2]}\lesssim |(\frac{1}{n})^{1/4}+\frac{\finalest^\top(\tilde\nu-\nu)}{\|\finalest\|}|\sigma:=\xi_n
%$$
$$
\|\E[1(\tilde y_i=1)\eps_{i}]\|=\|\E[\E[1(\tilde y_i=1)\eps_{i}\mid \labest^\top\varepsilon_i]]\|\le\E[\sqrt{d}+| \labest^\top\varepsilon_i]|]\lesssim\sqrt{d}.%|(\frac{1}{n})^{1/4}+\frac{\finalest^\top(\tilde\nu-\nu)}{\|\finalest\|}|\sigma:=\xi_n
$$

%Since $\|1(\tilde y_i=1)\eps_{i}^{(j)}-\E[1(\tilde y_i=1)\eps_{i}^{(j)}]\|_{\psi_2}\le 2\|1(\tilde y_i=1)\eps_{i}^{(j)} \|_{\psi_2}\le C\|\eps_{i}^{(j)}\|_{\psi_2}\le C\sigma_{\max}$
By the definition of $D_g$, we have 
$$
\bP\left(\left(\frac{1}{\nunlab}\sum_{i=1}^{\nunlab}1(\tilde y_i=1)\eps_{i}\supind j\right)^{2}\ge (\E[1(\tilde y_i=1)\eps_{i}^{(j)}])^2+t^2\cdot\sigma_{\max}^2\right)\le g(C\sqrt{\nunlab} t).
$$

Therefore, by union bound, with probability at least $1-(\log d)^{-1}$,
 $$
\|\frac{1}{\sum_{i=1}^{\tilde n}1(\tilde y_i=1)}\sum_{i=1}^{\tilde n}1(\tilde y_i=1)\varepsilon_i\|^2\lesssim\|\frac{1}{\nunlab}\sum_{i=1}^{\nunlab}1(\tilde y_i=1)\eps_{i}\|^2=\sum_{j=1}^{d}\left(\frac{1}{\nunlab}\sum_{i=1}^{\nunlab}1(\tilde y_i=1)\eps_{i}\supind j\right)^{2}\lesssim d+d\cdot\frac{(g^{-1}(1/d\log d))^2}{\tilde n}\sigma_{\max}^2.
$$
% $$
% (g^{-1}(1/d\log d))^2\cdot\sqrt d
% $$
Similarly, we have 
 $$
\|\frac{1}{\sum_{i=1}^{\tilde n}1(\tilde y_i=-1)}\sum_{i=1}^{\tilde n}1(\tilde y_i=-1)\varepsilon_i\|^2\lesssim\|\frac{1}{\nunlab}\sum_{i=1}^{\nunlab}1(\tilde y_i=-1)\eps_{i}\|^2=\sum_{j=1}^{d}\left(\frac{1}{\nunlab}\sum_{i=1}^{\nunlab}1(\tilde y_i=-1)\eps_{i}\supind j\right)^{2}\lesssim d+d\cdot\frac{(g^{-1}(1/d\log d))^2}{\tilde n}\sigma_{\max}^2.
$$

Then, since $\|\tilde{\delta}\|^{2}\le \|\frac{1}{2\sum_{i=1}^{\tilde n}1(\tilde y_i=1)}\sum_{i=1}^{\tilde n}1(\tilde y_i=1)\varepsilon_i\|^2+\|\frac{1}{2\sum_{i=1}^{\tilde n}1(\tilde y_i=-1)}\sum_{i=1}^{\tilde n}1(\tilde y_i=-1)\varepsilon_i\|^2$, we have
$$
\|\tilde{\delta}\|^{2}=O_p(d\cdot(1+\frac{(g^{-1}(1/d\log d))^2}{\tilde n}\sigma^2)).
$$
%we have by the union bound
%\begin{equation*}
%\bP\left(\sum_{j=1}^{d}\left(\frac{1}{\nunlab}\sum_{i=1}^{\nunlab}\yun_{i}\eps_{i}\supind j\right)^{2}\ge d\xi_n^2+\frac{d\cdot\sigma^{2}}{\nunlab}t^2\right)\le\sum_{j=1}^d \bP\left(\left(\frac{1}{\nunlab}\sum_{i=1}^{\nunlab}\yun_{i}\eps_{i}\supind j\right)^{2}\ge d\xi_n^2+t^2\cdot\frac{\sigma^2}{\nunlab}\right) \le de^{-C t^2}.%\label{eq:delta-til-norm-bound}
%\end{equation*}
%Therefore
%\begin{equation}
%\bP\left(\|\tilde{\delta}\|^{2}\le\|\gamma(\frac{\|\mu\|}{\|\tilde\mu\|}\tilde\mu-\mu)\|^2+ \frac{d\cdot\sigma^{2}}{\nunlab}t^2+d\xi_n^2\right)\ge1- de^{-C t^2}.\label{eq:delta-til-norm-bound}
%\end{equation}

The same technique also yields a crude bound on $\tilde\mu^{\top}\tilde{\delta}=\frac{1}{2\nunlab_1}\sum_{i=1}^{\tilde n}1(\tilde y_i=1)\tilde\mu^{\top}\varepsilon_i-\frac{1}{2\nunlab_2}\sum_{i=1}1(\tilde y_i=-1)\tilde\mu^{\top}\varepsilon_i$. We can write%=\gamma\mu^\top(\frac{\|\mu\|}{\|\tilde\mu\|}\tilde\mu-\mu)+\frac{1}{\nunlab}\sum_{i=1}^{\nunlab}\yun_{i}\mu^{\top}\eps_{i}$, as we have
$$
1(\yun_{i}=1)\tilde\mu^\top\eps_{i}\stackrel{i.i.d.}{\sim}   1\left((z_i\tilde\mu+\tilde\nu-\hat b_{intermediate}+\varepsilon_i
)^{\top}\labest>0\right)\cdot\tilde\mu^\top\eps_i. %1\left((z_i\mu+\varepsilon_i)^{\top}\labest>0\right)\cdot\tilde\mu^\top\eps_i.
$$

%Since $\|1(\yun_{i}=1)\tilde\mu^\top\eps_{i}\|_{\vartheta_2}\le C\|\tilde\mu^\top\eps_{i}^{(j)}\|_{\vartheta_2}\le C\|\tilde\mu\|_2\sigma$
By definition of $D_g$, we have 
$$
\bP\left(\left(\frac{1}{\nunlab}\sum_{i=1}^{\nunlab}1(\yun_{i}=1)\tilde\mu^\top\eps_{i} \right)^{2}\ge t^2\cdot\|\tilde\mu\|^2\sigma^2+\|\tilde\mu\|^2\sigma^2\right)\le g(C\sqrt{\nunlab} t).
$$
and by the fact that $\left(\frac{1}{\nunlab_1}\sum_{i=1}^{\nunlab}1(\yun_{i}=1)\tilde\mu^\top\eps_{i} \right)^2\lesssim\left(\frac{1}{\nunlab}\sum_{i=1}^{\nunlab}1(\yun_{i}=1)\tilde\mu^\top\eps_{i} \right)^2$, we have 
\[
\bP\left(\left|\tilde\mu^{\top}\tilde{\delta}\right|\ge\sqrt{2}\sigma\left\Vert \tilde\mu\right\Vert+\|\tilde\mu\|\sigma \right)=\bP\left(\left|\tilde\mu^{\top}\tilde{\delta}\right|^{2}\ge C\sigma^{2}\left\Vert \tilde\mu\right\Vert ^{2}\right)\le g(C\sqrt{\nunlab}).
\]
Finally, we need to argue that $\gamma$ is not too small. Recall
that $\gamma=\frac{1}{2\nunlab_1}\sum_{\tilde y_i=1}^{}\left(1-2b_{i}\right)+\frac{1}{2\nunlab_2}\sum_{\tilde y_i=-1}^{}\left(1-2b_{i}\right)$
 where
$b_{i}$ is the indicator that $\yun_{i}$ is incorrect and therefore
\begin{align*}
\E\left[1-2b_i\mid\labest,\tilde y_i=1\right]&=1-2\bP(f_{\labest}\mid \tilde x\sim subGaussian(\tilde\mu_1,\sigma^2))\\
&=2\bP(\varepsilon_i^{\top}\labest>-(\tilde\mu_1-\mu_1+\frac{\mu_1-\mu_2}{2}+e_b)^{\top}\labest)-1.
\end{align*}
This term can be lower bounded similarly as \eqref{ineq:gamma}, which satisfies 
\begin{align*}
&\E\left[1-2b_i\mid\labest,\tilde y_i=1\right]&\\
=&2\bP(\varepsilon_i^{\top}\labest>-(\tilde\mu_1-\mu_1+\frac{\mu_1-\mu_2}{2}+e_b)^{\top}\labest)-1\ge\frac{4}{5},
\end{align*}
with high probability when $d$ is sufficiently large.

Similarly, we have
\[
\E\left[1-2b_i\mid\labest,\tilde y_i=-1\right]\ge\frac{4}{5},
\]
with high probability when $d$ is sufficiently large.

%\[
%\E\left[\gamma\mid\labest,\tilde y_i=1\right]=1-2\err(f_{\labest}),
%\]
Therefore we expect $\gamma$ to be reasonably large as long as $\E[\gamma]\ge\frac{4}{5}$.
Indeed, define $$
\tilde\gamma=\frac{1}{\tilde n}\sum_{i=1}^n(1-2b_i).
$$
We then have \begin{align*}
\E[\tilde\gamma]\ge&\E[\frac{1}{\tilde n} \sum_{y_i=1}(1-2b_i)+\frac{1}{\tilde n} \sum_{y_i=-1}(1-2b_i)]\\
\ge&\E[\frac{1}{\tilde n}\cdot \frac{4}{5}\tilde n_1+\frac{1}{\tilde n} \cdot \frac{4}{5}\tilde n_2]\ge\frac{4}{5}.
\end{align*}

By using $\gamma\ge\frac{1}{2}\tilde\gamma$, we have
\begin{align*}
&\bP(\gamma\ge\frac{1}{5})\ge\bP(\tilde\gamma\ge\frac{2}{5})=1-\bP(\tilde\gamma<\frac{2}{5})\\
\ge& 1-\bP(|\tilde\gamma-\E[\tilde\gamma]|>\frac{2}{5})\ge 1- e^{-c\tilde n},
\end{align*}
where the last inequality is due to Hoeffding's inequality.

As a result, we have $\gamma \ge \frac{2}{5}$ with high probability.

%\begin{align*}
%\bP\left(\gamma<\frac{1}{6}\right) & =\bP\left(\frac{1}{\nunlab}\sum_{i=1}^{\nunlab}\left(1-2b_{i}\right)<\frac{1}{6}\right)\\
% & \le\bP\left(\err(f_{\labest})>\frac{1}{3}\right)+\bP\left(\frac{1}{\nunlab}\sum_{i=1}^{\nunlab}b_{i}<\frac{5}{12}\mid\err(f_{\labest})\le\frac{1}{3}\right).
%\end{align*}
%Note that
%\[
%\frac{1}{3}\ge Q\left(\frac{1}{2}\right)\ge Q\left(\left[2\left(1+\sqrt{\nn/d}\right)\right]^{-1/2}\right)
%\]
%Therefore, by Lemma \ref{lem:supervised-perf}, for sufficiently large
%$d/\nn$,
%\[
%\bP\left(\err(f_{\labest})>\frac{1}{3}\right)\le e^{-c\|\mu\|^2/8\sigma^2}+2e^{-cn\|\mu\|/2\sigma}%e^{-c\cdot\min\left\{ \sqrt{d/\nn},\nn\left(d/\nn\right)^{1/4}\right\} }
%\]
%for some constant $c$. Moreover, by Bernoulli concentration (Hoeffding's
%inequality) we have that
%\[
%\bP\left(\frac{1}{\nunlab}\sum_{i=1}^{\nunlab}b_{i}<\frac{5}{12}\mid\err(f_{\labest})\le\frac{1}{3}\right)\le e^{-2\nunlab\left(\frac{5}{12}-\frac{1}{3}\right)^{2}}=e^{-\nunlab/72}.
%\]
Define the event,
\[
\mathcal{E}=\left\{ \|\tilde{\delta}\|^{2}\le  \|\gamma(\frac{\|\mu\|}{\|\tilde\mu\|}\tilde\mu-\mu)\|^2+ \frac{d\cdot\sigma_{\max}^{2}}{\nunlab}{\log d}+d\xi_n^2,\ \left|\mu^{\top}\tilde{\delta}\right|\le\sqrt{2}\sigma_{\max}\left\Vert \mu\right\Vert+\gamma\mu^\top(\frac{\|\mu\|}{\|\tilde\mu\|}\tilde\mu-\mu)+\xi_n\|\mu\| \ \text{ and }\gamma\ge\frac{2}{5}\right\} ;
\]
by the preceding discussion,
\[
\bP\left(\mathcal{E}^{C}\right)\le \frac{1}{\log d}+g(C\sqrt{\nunlab})+g(C\|\mu\|/\sigma_{\max})+2g(C\sqrt{\nunlab \|\mu\|/\sigma_{\max}} )+e^{-c\tilde n}%2e^{-\nunlab/8}+e^{-\left(d-1\right)/8}+e^{-c\cdot\min\left\{ \sqrt{d/\nn},\nn\left(d/\nn\right)^{1/4}\right\} }+e^{-\nunlab/72}\le e^{-\tilde{c}_{2}\min\left\{ \nunlab,\sqrt{d/\nn},\nn\left(d/\nn\right)^{1/4}\right\} }.
\]
%
%$$
% \|\gamma(\frac{\|\mu\|}{\|\tilde\mu\|}\tilde\mu-\mu)\|^2\le 2\|\mu\|^2-2\frac{\|\mu\|}{\|\tilde\mu\|}\tilde\mu^\top\mu=2\|\mu\|^2(1-\frac{\tilde\mu^\top\mu}{\|\mu\|\|\tilde\mu\|})\le2(1-r_{\mu})\|\mu\|^2.
%$$
%$$
%|\gamma\mu^\top(\frac{\|\mu\|}{\|\tilde\mu\|}\tilde\mu-\mu) |\le \gamma (1-r_{\mu})\|\mu\|^2
%$$
%$$
%\xi_n=\sigma\cdot(1/{\tilde n}^{1/4}+d_{\nu})
%$$

Moreover, by the bound (\ref{eq:semisup-deomp}), $\mathcal{E}$ implies
\begin{align*}
\frac{\left\Vert \semisupest\right\Vert ^{2}}{\left(\tilde\mu^{\top}\semisupest\right)^{2}} & \le\frac{1}{\left\Vert \tilde\mu\right\Vert ^{2}}+\frac{\|\tilde{\delta}\|^{2}}{\left\Vert \tilde\mu\right\Vert ^{4}\left(\gamma+\frac{1}{\left\Vert \tilde\mu\right\Vert ^{2}}\tilde\mu^{\top}\tilde{\delta}\right)^{2}}\le\frac{1}{\left\Vert \tilde\mu\right\Vert ^{2}}+\frac{d\cdot(1+\frac{(g^{-1}(1/d\log d))^2}{\tilde n}\sigma_{\max}^2)}{\left\Vert \tilde\mu\right\Vert ^{4}\left(\frac{2}{5}+ \frac{1}{\left\Vert \tilde\mu\right\Vert ^{2}}\cdot\|\tilde\mu\|\sigma_{\max}\right)^{2}}.
\end{align*}

Therefore, \[
\frac{\tilde\mu^{\top}\finalest}{\sigma_{\max}\left\Vert \finalest\right\Vert }\ge\left(\frac{\sigma_{\max}^2}{\left\Vert \tilde\mu\right\Vert ^{2}}+\frac{d\cdot(1+\frac{(g^{-1}(1/d\log d))^2}{\tilde n}\sigma_{\max}^2)}{\left\Vert \tilde\mu\right\Vert ^{4}\left(\frac{2}{5}-\frac{\sigma_{\max}}{\|\tilde\mu\|_2}\right)^{2}}\right)^{-1/2}
\]
with probability $\ge1-(\frac{1}{\log d}+g(C\sqrt{\nunlab})+g(C\|\mu\|/\sigma_{\max})+2g(C\sqrt{\nunlab \|\mu\|/\sigma_{\max}} )+e^{-c\tilde n})$.

Recall that we take  $\sigma_{\max}:=C_\sigma d^{1/4}$ and $\|\tilde\mu_1-\tilde\mu_2\|_2=C_\mu \sqrt d$ for sufficiently large $C_\mu$, we than have when $\tilde n\gtrsim \varepsilon^2 d(g^{-1}(1/d\log d))^2 $, $$
\frac{\tilde\mu^{\top}\finalest}{\left\Vert \finalest\right\Vert }=\Omega_P(\sqrt d).
$$

Then let us consider
\begin{align*}
\hat b_{\text{final}}=&\frac{1}{2\nunlab_1}\sum_{\tilde y_i=1}^{}\xun_{i}+\frac{1}{2\nunlab_2}\sum_{\tilde y_i=-1}^{}\xun_{i}\\
=&\tilde\nu+\frac{1}{2\nunlab_1}\sum_{\tilde y_i=1}^{}\left(1-2b_{i}\right)\tilde\mu-\frac{1}{2\nunlab_2}\sum_{\tilde y_i=-1}^{}\left(1-2b_{i}\right)\tilde\mu+\frac{1}{2\nunlab_1}\sum_{\tilde y_i=1}\varepsilon_i+\frac{1}{2\nunlab_2}\sum_{\tilde y_i=-1}\varepsilon_i\\
=&\tilde\nu+\left[\frac{1}{2\nunlab_1}\sum_{\tilde y_i=1}^{}\left(1-2b_{i}\right)-\frac{1}{2\nunlab_2}\sum_{\tilde y_i=-1}^{}\left(1-2b_{i}\right)\right]\tilde\mu+\frac{1}{2\nunlab_1}\sum_{\tilde y_i=1}\varepsilon_i+\frac{1}{2\nunlab_2}\sum_{\tilde y_i=-1}\varepsilon_i,
%&{\color{red}=(\frac{1}{\tilde n}\sum_{i=1}^{\tilde n} \tilde y_i)\tilde\nu+\gamma\tilde\mu+\frac{1}{\nunlab}\sum_{i=1}^{\nunlab}\yun_{i}\eps_{i}}
\end{align*}

Let
\[
\lambda\defeq\frac{1}{\nunlab_1}\sum_{\tilde y_i=1}^{}\left(1-2b_{i}\right)-\frac{1}{\nunlab_2}\sum_{\tilde y_i=-1}^{}\left(1-2b_{i}\right)
\] 
When $n>C$ for sufficiently large $C$, we have $\lambda\le0.01$.

Also, let us denote $\tilde\delta_2=\frac{1}{2\nunlab_1}\sum_{\tilde y_i=1}\varepsilon_i+\frac{1}{2\nunlab_2}\sum_{\tilde y_i=-1}\varepsilon_i$, we then have $$
|\tilde\delta^\top\tilde\delta_2|=\|\frac{1}{2\nunlab_1}\sum_{\tilde y_i=1}\varepsilon_i\|^2-\|\frac{1}{2\nunlab_2}\sum_{\tilde y_i=-1}\varepsilon_i\|^2\le\|\frac{1}{2\nunlab_1}\sum_{\tilde y_i=1}\varepsilon_i\|^2\le c_\varepsilon(d+d\cdot\frac{(g^{-1}(1/d\log d))^2}{\tilde n}\sigma_{\max}^2)
$$

We also have 
\begin{align*}
|\frac{\finalest^\top}{\|\finalest\|}d_b|\le& \lambda|\frac{\finalest^\top}{\|\finalest\|}\tilde\mu|+|\frac{\finalest^\top}{\|\finalest\|}\tilde\delta_2|\\
\le&\lambda\|\tilde\mu\|+|\frac{( \tilde{\delta}+\gamma\tilde\mu)^\top}{\| \tilde{\delta}+\gamma\tilde\mu\|}\tilde\delta_2|\\
\le&\lambda\|\tilde\mu\|+\frac{| \tilde{\delta}^\top \tilde\delta_2+\gamma\tilde\mu^\top\tilde\delta_2|}{\|\gamma\tilde\mu\|-\|\tilde{\delta}\|}\\
\le&\lambda\|\tilde\mu\|+\frac{c_\varepsilon(d+d\cdot\frac{(g^{-1}(1/d\log d))^2}{\tilde n}\sigma^2)+O_P(\|\tilde\mu\|\sigma)}{\gamma\|\tilde\mu\|-c_\varepsilon(d+d\cdot\frac{(g^{-1}(1/d\log d))^2}{\tilde n}\sigma^2)}\\
\le& (\lambda C_\mu+\frac{c_\varepsilon}{\gamma C_\mu-c_\varepsilon})\cdot\sqrt d
\end{align*}

Therefore, when the constant $C_\mu$ is sufficiently large,
\begin{align*}
&\frac{\finalest^\top}{\|\finalest\|}((\mu_1-\tilde\mu_1)+\frac{\tilde\mu_1-\tilde\mu_2}{2}-d_b)\\
\ge& |\frac{\finalest^\top}{\|\finalest\|}\tilde\mu|-|\frac{\finalest^\top}{\|\finalest\|}(\mu_1-\tilde\mu_1)|-|\frac{\finalest^\top}{\|\finalest\|}d_b|\\
\ge&\left(\frac{\sigma^2}{\left\Vert \tilde\mu\right\Vert ^{2}}+\frac{d\cdot(1+\frac{(g^{-1}(1/d\log d))^2}{\tilde n}\sigma_{\max}^2)}{\left\Vert \tilde\mu\right\Vert ^{4}\left(\frac{1}{6}-\frac{\sigma_{\max}}{\|\tilde\mu\|_2}\right)^{2}}\right)^{-1/2}-2c_0\|\tilde\mu\|-(\lambda C_\mu+\frac{c_\varepsilon}{\gamma C_\mu-c_\varepsilon})\cdot\sqrt d\\
=&\Omega_P(\sqrt d).
\end{align*}

%$$\bP\Big(\big\|\hat b-\bE[\hat b]\big\|\geq c\sigma \big(\frac{\sqrt{d}\tilde{L}}{\sqrt{n/2}}+L\sqrt{\frac{2\log(2/\delta)}{n}}\big) \Big)\leq \delta.$$

%Recall that
%$r_\mu=|\frac{(\mu_1-\mu_2)^\top(\tilde\mu_1-\tilde\mu_2)}{\|\mu_1-\mu_2\|\|\tilde\mu_1-\tilde\mu_2\|}|$, $d_\nu=\frac{(\mu_1+\mu_2)^\top(\tilde\mu_1+\tilde\mu_2-\mu_1-\mu_2)/2}{\|\mu_1+\mu_2\|}$

%{\color{red}
%$$
%\bP\Big(\frac{\finalest^\top}{\|\finalest\|} \varepsilon\leq -\frac{\finalest^\top}{\|\finalest\|}(\mu_1-\hat b) \Big)=\bP\Big(\frac{\finalest^\top}{\|\finalest\|} \varepsilon\leq -\frac{\finalest^\top}{\|\finalest\|}((\mu_1-\tilde\mu_1)+\frac{\tilde\mu_1-\tilde\mu_2}{2}-d_b) \Big)
%$$}

The robust error is then
\begin{align*}
&\bP\Big(\frac{\finalest^\top}{\|\finalest\|} \varepsilon\leq -\frac{\finalest^\top}{\|\finalest\|}(\mu_1-\hat b) +L_1'\cdot\varepsilon\sqrt d\Big)\\
=&\bP\Big(\frac{\finalest^\top}{\|\finalest\|} \varepsilon\leq -\frac{\finalest^\top}{\|\finalest\|}((\mu_1-\tilde\mu_1)+\frac{\tilde\mu_1-\tilde\mu_2}{2}-d_b)+L_1'\cdot\varepsilon\sqrt d \Big)\\
\le& \exp(-C\sqrt{d})\le0.01,
\end{align*}
when $d$ is sufficiently large. 
       }
%==============================================================
\subsection{Proof of Theorem 4}%5 (Adding data from shifted domain enhances adversarial robustness)}
Let us consider the following modelr:  $x\sim N(y\mu,\sigma^2 I)$ with $y$ uniform on $\{-1,1\}$ and $\mu\in\R^d$. Consider a linear classifier $f_w(x)=sgn(x^\top w)$.%f(x) = sign �x>✓�. Then the standard error probability is

It's easy to see that the robust error probability is $$
err_{robust}^{\infty}(f)=Q(\frac{\mu^\top w}{\sigma\|w\|}-\frac{\varepsilon\|w\|_1}{\sigma\|w\|}),
$$
where $Q=\frac{1}{\sqrt{2\pi}}\int_x^\infty e^{-t^2/2}\;dt$.

Therefore \begin{align*}
\argmin_{\|w\|=1} err_{robust}^{\infty}(f_w)=&\argmax_{\|w\|=1} \frac{\mu^\top w}{\sigma\|w\|}-\frac{\varepsilon\|w\|_1}{\sigma\|w\|}\\
=&\argmax_{\|w\|=1}\mu^\top w-\varepsilon\|w\|_1\\
=&\argmax_{\|w\|=1}\sum_{j=1}^d\mu_jw_j-\varepsilon|w_j|
\end{align*}
By observation, when reaching maximum, we have to have $sgn(w_j)=sgn(\mu_j)$, therefore
\begin{align*}
&\argmax_{\|w\|=1}\sum_{j=1}^d\mu_jw_j-\varepsilon|w_j|\\
=&\argmax_{\|w\|=1}\sum_{j=1}^d(\mu_j-\varepsilon\cdot sgn(\mu_j))w_j\\
=&\frac{T_\varepsilon(\mu)}{\|T_\varepsilon(\mu)\|},
\end{align*}
where $T_\varepsilon(\mu)$ is the hard-thresholding operator with $(T_\varepsilon(\mu))_j=sgn(\mu_j)\cdot\max\{ |\mu_j|-\varepsilon,0\}$.

Now let us consider the example: $\mu$ with $\mu_j>\varepsilon$ for all $j=1,2,...d$. For the shifted domain, we let $\tilde\mu_1=-\tilde\mu_2=\tilde\mu=\mu-\varepsilon\cdot 1_p$, and the mixing proportion is half-half. 

Let $b_{i}$ be the indicator that the $i$th pseudo-label $\tilde y_i=sgn(\tilde x_i^\top \hat w_{intermediate})$ is incorrect, so that $\xun_{i}\sim{N}\left(\left(1-2b_{i}\right)\yun_{i}\tilde\mu,\sigma^{2}I\right)$,
and let 
\[
\gamma\defeq\frac{1}{\nunlab}\sum_{i=1}^{\nunlab}\left(1-2b_{i}\right)\in[-1,1].
\]

%Let $\tilde n_1=\sum_{i=1}^{\tilde n}1\{\tilde y_i=1\}$, $\tilde n_2=\sum_{i=1}^{\tilde n}1\{\tilde y_i=-1\}$. 
We may write the final direction estimator as
\[
\semisupest=\frac{1}{\nunlab}\sum_{i=1}^{\nunlab}\yun_{i}\xun_{i}=\gamma\tilde\mu+\frac{1}{\nunlab}\sum_{i=1}^{\nunlab}\yun_{i}\eps_{i}
\]
%\begin{align*}
%&\semisupest=\frac{1}{\nunlab_1}\sum_{\tilde y_i=1}^{\nunlab}\xun_{i}-\frac{1}{\nunlab_2}\sum_{\tilde y_i=-1}^{\nunlab}\xun_{i}=\frac{1}{\nunlab_1}\sum_{\tilde y_i=1}^{}\left(1-2b_{i}\right)\tilde\mu+\frac{1}{\nunlab_1}\sum_{\tilde y_i=1}\varepsilon_i-\frac{1}{\nunlab_2}\sum_{\tilde y_i=-1}^{}\left(1-2b_{i}\right)\tilde\mu-\frac{1}{\nunlab_2}\sum_{\tilde y_i=-1}\varepsilon_i\\
%=& \\
%&{\color{red}=(\frac{1}{\tilde n}\sum_{i=1}^{\tilde n} \tilde y_i)\tilde\nu+\gamma\tilde\mu+\frac{1}{\nunlab}\sum_{i=1}^{\nunlab}\yun_{i}\eps_{i}}
%\end{align*}
where $\eps_{i}\sim{N}\left(0,\sigma^{2}I\right)$ independent
of each other. 

By orthogonal invariance of Gaussianality, we choose a coordinate system such that the first coordinate is in the direction of $\hat w_{intermediate}$, we then have 
\begin{align*}
\|\frac{1}{\nunlab}\sum_{i=1}^{\nunlab}\yun_{i}\eps_{i}\|_2^2=&|\frac{1}{\nunlab}\sum_{i=1}^{\nunlab}\yun_{i}\eps_{i1}|_2^2+\sum_{j=2}^d|\frac{1}{\nunlab}\sum_{i=1}^{\nunlab}\yun_{i}\eps_{ij}|_2^2\\
\le&\frac{\sigma^2}{\tilde n}\chi^2_{\tilde n}+\frac{\sigma^2}{\tilde n}\chi^2_{d-1}.
\end{align*}

In addition, we have $\|\tilde\mu\|_2^2=d\varepsilon^2$. Therefore, if $(1/d+1/\tilde n)\cdot\frac{\sigma^2}{\varepsilon^2}\to0$, we will then have $$
\frac{\semisupest}{\|\semisupest\|}\to\tilde\mu,
$$ 
and therefore
$$
err_{robust}^{\infty}(f_{\hat w_{final}})\le err_{robust}^{\infty}(f_{\mu}).
$$

%\begin{theorem}
%Let us consider the labeled domain $x_1,...,x_n\sim \frac{1}{2}N(\mu,\sigma^2I)+\frac{1}{2}N(-\mu,\sigma^2I)$ with $\mu\in\R^d$. Suppose we have unlabeled same domain $\tilde x_1^{(1)},...,\tilde x_{\tilde n}^{(1)}\sim \frac{1}{2}N(\mu,\sigma^2I)+\frac{1}{2}N(-\mu,\sigma^2I)$, and unlabeled shifted domain $\tilde x_1^{(2)},...,\tilde x^{(2)}_{\tilde n}\sim \frac{1}{2}N(\tilde\mu,\sigma^2I)+\frac{1}{2}N(-\tilde\mu,\sigma^2I)$. Suppose the classifier direction constructed by combing those two different unlabeled domains are $\hat\theta_{same}$ and  $\hat\theta_{shifted}$ respectively.
%If we let $\mu=2\varepsilon 1_p$ and   $\tilde\mu=\varepsilon 1_p$, when  $(1/d+1/\tilde n)\cdot\frac{\sigma^2}{\varepsilon^2}\to0$, we then have 
%$$
%err_{robust}^{\infty}(f_{\hat\theta_{shifted}})\le err_{robust}^{\infty}(f_{\hat\theta_{same}})+o(1).
%$$
%\end{theorem}

\subsection{Proof of Theorem 5}
Suppose the labeled domain distribution is $\frac{1}{2} N(\mu,\sigma^2 I_p)+\frac{1}{2} N(-\mu,\sigma^2 I_p)$. Let $v\in\R^p$ be a vector such that $v^\top\mu=0$, $\|\mu\|=a\|v\|$ for some $a>0$, and let the unlabeled domain distribution be $\frac{1}{2} N(v,\sigma^2 I_p)+\frac{1}{2} N(-v,\sigma^2 I_p)$. That is, $\mu_1=-\mu_2=\mu, \tilde\mu_1=-\tilde\mu_2=v$.

We then have 

$$d_\nu=\max\{|\frac{\|\tilde\mu_1-\mu_1\|}{\|\tilde\mu_1-\tilde\mu_2\|}|,\frac{\|\tilde\mu_2-\mu_2\|}{\|\tilde\mu_1-\tilde\mu_2\|}|\}=\frac{\sqrt{1+a^2}\cdot\|v\|}{2\|v\|}=\frac{\sqrt{1+a^2}\cdot\|\tilde\mu_1-\tilde\mu_2\|}{2},%<\frac{1}{2}\|\tilde\mu_1-\tilde\mu_2\|_2^2
$$
which falls into the specified class. 

Now let us consider the case where $\mu=e_1, v=a^{-1}e_2$, where $e_1,e_2$ are the canonical basis, and study the performance of the classifier $\text{sgn}(\semisupest^\top(z-\hat b))$, where
%$r_\mu=|\frac{(\mu_1-\mu_2)^\top(\tilde\mu_1-\tilde\mu_2)}{\|\mu_1-\mu_2\|\|\tilde\mu_1-\tilde\mu_2\|}|$, 
%$d_\nu=\frac{(\tilde\mu_1+\tilde\mu_2)^\top(\tilde\mu_1+\tilde\mu_2-\mu_1-\mu_2)/2}{\|\tilde\mu_1+\tilde\mu_2\|}$
%$$
%\hat b=\frac{1}{\tilde n}\sum_{i=1}^{\tilde n}\tilde x_i; \quad \semisupest=\frac{1}{\nunlab}\sum_{i=1}^{\nunlab}\yun_{i}\xun_{i}
%$$
$$
\hat b_{\text{final}}=\frac{1}{2\nunlab_1}\sum_{\tilde y_i=1}^{}\xun_{i}+\frac{1}{2\nunlab_2}\sum_{\tilde y_i=-1}^{}\xun_{i};  \quad \semisupest=\frac{1}{2\nunlab_1}\sum_{\tilde y_i=1}^{}\xun_{i}-\frac{1}{2\nunlab_2}\sum_{\tilde y_i=-1}^{}\xun_{i}.
$$

Similar to the proof in the last section, let $b_{i}$ be the indicator that $\yun_{i}$ is incorrect  and
we decompose $\semisupest$ and $\hat b$ into

\begin{align*}
\semisupest=&\frac{1}{2\nunlab_1}\sum_{\tilde y_i=1}^{}\xun_{i}-\frac{1}{2\nunlab_2}\sum_{\tilde y_i=-1}^{}\xun_{i}\\
=&\frac{1}{2\nunlab_1}\sum_{\tilde y_i=1}^{}\left(1-2b_{i}\right)v+\frac{1}{2\nunlab_1}\sum_{\tilde y_i=1}\varepsilon_i+\frac{1}{2\nunlab_2}\sum_{\tilde y_i=-1}^{}\left(1-2b_{i}\right)v-\frac{1}{2\nunlab_2}\sum_{\tilde y_i=-1}\varepsilon_i,
%&{\color{red}=(\frac{1}{\tilde n}\sum_{i=1}^{\tilde n} \tilde y_i)\tilde\nu+\gamma\tilde\mu+\frac{1}{\nunlab}\sum_{i=1}^{\nunlab}\yun_{i}\eps_{i}}
\end{align*}

\begin{align*}
\hat b_{\text{final}}=&\frac{1}{2\nunlab_1}\sum_{\tilde y_i=1}^{}\xun_{i}+\frac{1}{2\nunlab_2}\sum_{\tilde y_i=-1}^{}\xun_{i}\\
=&\frac{1}{2\nunlab_1}\sum_{\tilde y_i=1}^{}\left(1-2b_{i}\right)\tilde\mu-\frac{1}{2\nunlab_2}\sum_{\tilde y_i=-1}^{}\left(1-2b_{i}\right)\tilde\mu+\frac{1}{2\nunlab_1}\sum_{\tilde y_i=1}\varepsilon_i+\frac{1}{2\nunlab_2}\sum_{\tilde y_i=-1}\varepsilon_i.
%&{\color{red}=(\frac{1}{\tilde n}\sum_{i=1}^{\tilde n} \tilde y_i)\tilde\nu+\gamma\tilde\mu+\frac{1}{\nunlab}\sum_{i=1}^{\nunlab}\yun_{i}\eps_{i}}
\end{align*}

Now let us investigate $b_i$ carefully. When $\tilde y_i=1$, we have \begin{align*}
&\E[b_i\mid \tilde y_i=1]=\bP( \tilde x\sim subGaussian(\tilde\mu_1,\sigma^2)\mid(\tilde x-\hat b_{\text{intermediate}})^{\top}\labest>0)\\
=&\frac{\bP((\tilde x-\hat b_{\text{intermediate}})^{\top}\labest>0)\mid  \tilde x\sim subGaussian(\tilde\mu_1,\sigma^2)}{\bP((\tilde x-\hat b_{\text{intermediate}})^{\top}\labest>0)\mid  \tilde x\sim (\tilde\mu_1,\sigma^2))+\bP((\tilde x-\hat b_{\text{intermediate}})^{\top}\labest>0)\mid  \tilde x\sim (\tilde\mu_2,\sigma^2))}\\
=&\frac{1}{2}+O_p(\sqrt\frac{d}{n}).
\end{align*}

As a result, we have $$
\tilde w_{\text{final}}=(O_p(\sqrt\frac{d}{n})+O_p(\sqrt\frac{1}{\tilde n})) v+\frac{1}{2\nunlab_1}\sum_{\tilde y_i=1}\varepsilon_i-\frac{1}{2\nunlab_2}\sum_{\tilde y_i=-1}\varepsilon_i;
$$
 $$
\hat b_{\text{final}}=(O_p(\sqrt\frac{d}{n})+O_p(\sqrt\frac{1}{\tilde n})) v+\frac{1}{2\nunlab_1}\sum_{\tilde y_i=1}\varepsilon_i+\frac{1}{2\nunlab_2}\sum_{\tilde y_i=-1}\varepsilon_i.
$$

Then let us study $\varepsilon_i \mid \tilde y_i=1$. When $\tilde y_i=1$, we have $$
(\tilde x-\hat b_{\text{intermediate}})^{\top}\labest>0.
$$
Recall that $\mu=e_1, v=a^{-1}e_2$ the inequality is equivalent to $$
\langle e_1, \varepsilon_i\rangle+O_P(\sqrt{\frac{d}{n}})>0.
$$
and put no constraint on other coordinates. 
Similarly, when $\tilde y_i=-1$, we have
 $$
\langle e_1, \varepsilon_i\rangle+O_P(\sqrt{\frac{d}{n}})<0,
$$
and put no constraint on other coordinates. 

As a result, we have $$
\frac{\tilde w_{\text{final}}}{\|\tilde w_{\text{final}}\|_2}=(O_p(\sqrt\frac{d}{n})+O_p(\sqrt\frac{1}{\tilde n})) v+ e_1;
$$
 $$
\hat b_{\text{final}}=(O_p(\sqrt\frac{d}{n})+O_p(\sqrt\frac{1}{\tilde n})) v+O_p(\sqrt\frac{d}{\tilde n}).
$$

Then we write out the misclassification error
\begin{align*}
&\bP\Big(\frac{\finalest^\top}{\|\finalest\|} (\tilde x-\hat{b})\leq 0 \mid \tilde x\sim subGaussian(v,\sigma^2)\Big)
=\bP\Big(\frac{\finalest^\top}{\|\finalest\|} \varepsilon\leq O_p(\sqrt\frac{d}{\tilde n})+O_p(\sqrt\frac{d}{n})\Big)\\
=& 1/2+O_p(\sqrt\frac{d}{\tilde n})+O_p(\sqrt\frac{d}{n}).
%=&\bP\Big(\frac{\finalest^\top}{\|\finalest\|} \varepsilon\leq -\frac{\tilde\mu^{\top}\finalest}{\left\Vert \finalest\right\Vert }+ d_\nu\Big).
%\le&  \exp\left(d_\nu-\left(\frac{\sigma^2}{\left\Vert \mu\right\Vert ^{2}}+\frac{2\sigma^{4}d\sqrt{\log d}/\nunlab+2\sigma^2(1-r_\mu)\|\mu\|^2+d\sigma^2({\tilde n}^{-1/2}+d_{\nu})}{\left\Vert \mu\right\Vert ^{4}\left(\frac{1}{6}r_\mu-\frac{\sqrt{2}\sigma}{\left\Vert \mu\right\Vert }-\sigma\cdot(1/{\tilde n}^{1/4}+d_{\nu})\right)^{2}}\right)^{-1/2}\right)
\end{align*}
Therefore, when $\frac{d}{n}$ and $\frac{d}{\tilde n}$ sufficiently small, we then have $$
\beta^{\varepsilon,\infty}(\tilde{w},\tilde{b})\ge\beta^{\infty}(\tilde{w},\tilde{b})\ge 49\%.
$$

\subsection{The high-dimensional EM algorithm mentioned in the main paper}

The algorithm used in the main paper to extract the support information from the unlabeled domain is presented in the following in Algorithm \ref{alg:hd}, which is adapted from \cite{cai2019chime}.

\begin{algorithm}
\caption{{\bf C}lustering of {\bf HI}gh-dimensional Gaussian {\bf M}ixtures with the {\bf E}M (CHIME)}\label{alg:hd}
\begin{algorithmic}[1]
\STATE {\bf Inputs:} Initializations $\hat\omega^{(0)}, \hat\mu_1^{(0)}$, and $\hat\mu_2^{(0)}$, maximum number of iterations $T_0$, and a constant $\kappa\in(0,1)$. Set 
\[
\hat\bbeta^{(0)}=\argmin_{ \bbeta \in \mathbb R^p}  \left\{ \frac{1}{2}\|\bbeta\|^2-\bbeta^\top(\hat\mu_1 ^{(0)}-  \hat\mu_2^{(0)})+\lambda^{(0)}_n\|\bbeta\|_1 \right\},
\] 
where the tuning parameter $\lambda_n^{(0)}=C_1\cdot( |\hat\omega|\vee\|\hat\mu_1^{(0)}-\hat\mu_2^{(0)}\|_{2,s})/\sqrt s+C_{\lambda}\sqrt{\log p / n}$.
\FOR {$t=0,1,\ldots,T_0-1$}
%\State {\bf E-Step:} Evaluate $Q_n(\bth \mid \hat\bth^{(t)})$ with $\gamma_{\hat\bth^{(t)}}(\tilde x_i)$ defined in \eqref{gammai-def2}.% conditioning on $\omega^{(t)}, \mu_1^{(t)}, \mu_2^{(t)}$ and $\bbeta^{(t)}$.
\STATE Let \begin{align}\label{gammai-def}
\notag\gamma_{\hat\bth^{(t)}}(\tilde x_i) =\frac{\hat\omega^{(t)}}{\hat\omega^{(t)} + (1-\hat\omega^{(t)}) \exp\big\{ ((\hat\mu_2^{(t)}-\hat\mu_1^{(t)}))^\top \big(\tilde x_i- \frac{\hat\mu_1^{(t)} + \hat\mu_2 ^{(t)}}{2}\big) \big\}}.
%\frac{\hat\omega^{(t)}}{\hat\omega^{(t)} + (1-\hat\omega^{(t)}) \exp\big\{ (\hat\bbeta^{(t)})\trp \big(\tilde x_i - \frac{\hat\mu_1^{(t)} + \hat\mu_2 ^{(t)}}{2}\big) \big\}}.
\end{align}

\STATE Update $ \hat\omega^{(t+1)},\hat \mu_1^{(t+1)}$, and $\hat{\mu}_2 ^{(t+1)},$ by
\begin{align*} 
\hat{\omega}^{(t+1)}&=\hat\omega(\hat\bth^{(t)})= \frac{1}{n} \sum_{i=1}^n \gamma_{ \hat\bth^{(t)}}(\tilde x_i),\\
\hat{\mu}_1^{(t+1)}&=\hat{\mu}_1(\hat\bth^{(t)})= \Big\{n-\sum_{i=1}^n\gamma_{ \hat\bth^{(t)}}(\tilde x_i)\Big\}^{-1} \Big\{\sum_{i=1}^n(1 - \gamma_{\hat \bth^{(t)}}(\tilde x_i) )\tilde x_i \Big\},\\
 \hat{\mu}_2^{(t+1)}&=\hat{\mu}_2(\hat\bth^{(t)})= \Big\{\sum_{i=1}^n \gamma_{ \hat\bth^{(t)}}(\tilde x_i) \Big\}^{-1} \Big\{\sum_{i=1}^n \gamma_{\hat \bth^{(t)}}(\tilde x_i)\tilde x_i \Big\},
\end{align*}
 and update $\hat\bbeta^{(t+1)}$ via 
\[
\hat{\bbeta}^{(t+1)} = \argmin_{ \bbeta \in \mathbb R^p} \left\{ \frac{1}{2}\|\bbeta\|^2-\bbeta^\top(\hat\mu_1 ^{(t+1)}-  \hat\mu_2^{(t+1)})+\lambda^{(t+1)}_n\|\bbeta\|_1 \right\}, 
\]
with 
\begin{equation*}\label{lambda:alg1}
\lambda_n^{(t+1)}=\kappa \lambda_n^{(t)}+C_{\lambda}\sqrt\frac{\log p}{n}.
\end{equation*} 
\ENDFOR
\STATE Output the support of $\hat\bbeta^{(T_0)}$, $\hat S=Supp(\hat\bbeta^{(T_0)})$.
\STATE Project $x_i$'s to  $\hat S=Supp(\hat\bbeta^{(T_0)})$, estimate $\hat w$ and $\hat b$ on these projected samples
\STATE Construct the classifier $sgn(\hat w^\top (z_{\hat S})-\hat b)$
%\State Construct the clustering rule  
%\begin{equation*}%\label{e:cluster:rule2}
%\hat G_{\rm CHIME}(\blz)=
%\begin{cases}
%1, &\big\{ \blz-\frac{\hat\mu_1^{(T_0)}+\hat\mu_2^{(T_0)}}{2}\big\}\trp\hat\bbeta^{(T_0)}\ge\log(\frac{\hat\omega^{(T_0)}}{1-\hat\omega^{(T_0)}}),\cr 
%2, &\big\{\blz-\frac{\hat\mu_1^{(T_0)}+\hat\mu_2^{(T_0)}}{2}\big\}\trp\hat\bbeta^{(T_0)}<\log(\frac{\hat\omega^{(T_0)}}{1-\hat\omega^{(T_0)}}).
% \end{cases}
%\end{equation*} 
\end{algorithmic}
\end{algorithm}

\subsection{Proof of Theorem 6}
Let us first adapt the Theorem 3.1 in \cite{cai2019chime}, which states the convergence rate of Algorithm \ref{alg:hd}
\begin{lemma}[adapted from Theorem 3.1 in  \cite{cai2019chime}]\label{lem:chime}
Under the same conditions of Theorem 3.6, if we choose the initializations of Algorithm \ref{alg:hd} according to \cite{hardt2015tight}. Then there is a constant $\kappa\in(0,1)$, such that the estimator $\hat\bbeta^{(T_0)}$ satisfies $$
\|\hat\bbeta^{(T_0)}-(\tilde\mu_1-\tilde\mu_2)\|_2\lesssim \kappa^{T_0}\cdot (\|\hat\bbeta^{(0)}-(\tilde\mu_1-\tilde\mu_2)\|_2+|\hat\omega^{(0)}-q|)+\sigma\sqrt{\frac{m\log d}{\tilde n}}.
$$
In particular, if we let $T_0\gtrsim (-\log(\kappa))^{-1}\log(n\cdot  (\|\hat\bbeta^{(0)}-(\tilde\mu_1-\tilde\mu_2)\|_2+|\hat\omega^{(0)}-q|))$, we have $$
\|\hat\bbeta^{(T_0)}-(\tilde\mu_1-\tilde\mu_2)\|_2\lesssim \sigma\sqrt{\frac{m\log d}{\tilde n}}.
$$
\end{lemma}

As a direct consequence of Lemma~\ref{lem:chime}, we have $$
\|\hat\bbeta^{(T_0)}-(\tilde\mu_1-\tilde\mu_2)\|_\infty\lesssim \sigma\sqrt{\frac{m\log d}{\tilde n}}.
$$
Using the condition that $\min_{\tilde{\mu}_{1,j}-\tilde{\mu}_{2j}\neq 0}|\tilde{\mu}_{1j}-\tilde{\mu}_{2,j}|\ge C\sigma\sqrt{ 2m\log d/\tilde n}$ for sufficiently large $C$, we then have, with high probability, 
$$
\hat S=Supp(\hat\bbeta^{(T_0)})=Supp(\tilde\mu_1-\tilde\mu_2).
$$
Therefore, when we project the labeled data to this support $\hat S$, it reduce the model to the previous setting considered in Theorem 3.1 and 3.2 with the dimension of $\nu(\tilde x)$ reduced to $m$. Combing the proofs of Theorem 3.1, 3.2, and 3,3m we then have the desired result that if  $n\gtrsim \varepsilon^2\log d\sqrt m$, we have 
$$
 \beta^{\varepsilon,\infty}(\hat w_{sparse}, \hat b_{sparse})\le 10^{-3}+O_P(\frac{1}{ n}+\frac{1}{d}).
$$%consider the labeled samples with reduced dimensionality to get the estimatedˆwsparse= 1/n∑ni=1yi[θ(xi)]ˆS217andˆbsparse= 1/n∑2ni=n+1[θ(xi)]ˆS

\section{Experimental Implementation on Synthetic Data}
We here complement our theory result in Theorem \ref{thm:enhance} and Theorem \ref{thm:sparsity} by experiments with synthetic data.

\begin{figure}[H]
    \centering
    \includegraphics[width=0.5\linewidth]{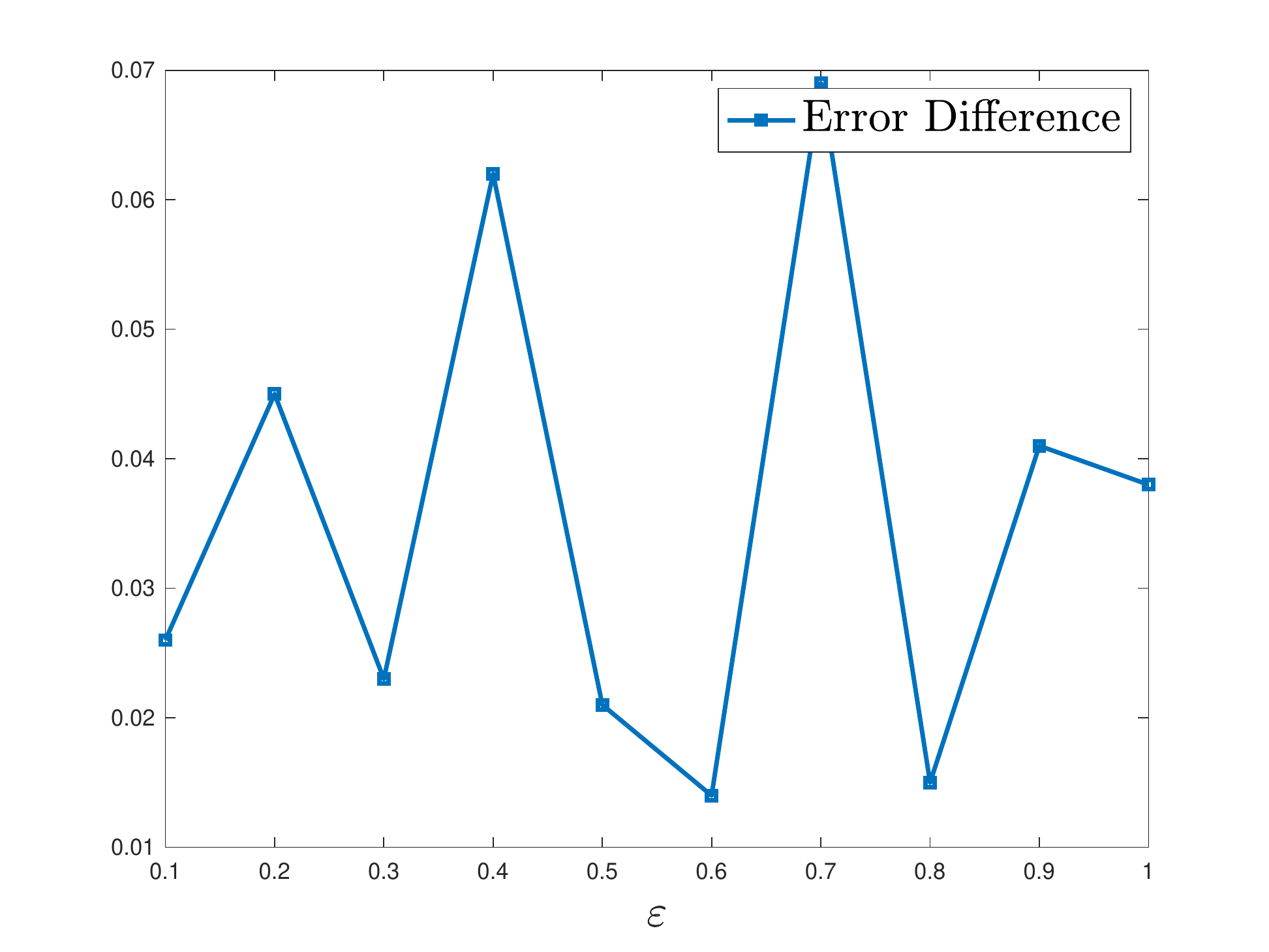} 
    \caption{Error Difference vs $\varepsilon$: we use synthetic data described in Theorem \ref{thm:enhance}, where the Error Difference = Adversarial Error when Unlabeled Data from the Same Domain - Adversarial Error when Unlabeled Data from the Shifted Domain. We can see that error difference is positive for all the $\varepsilon$ we take, which implies unlabeled data from a shifted domain works even better. }
\end{figure}

\begin{figure}[H]
    \centering
    \includegraphics[width=0.5\linewidth]{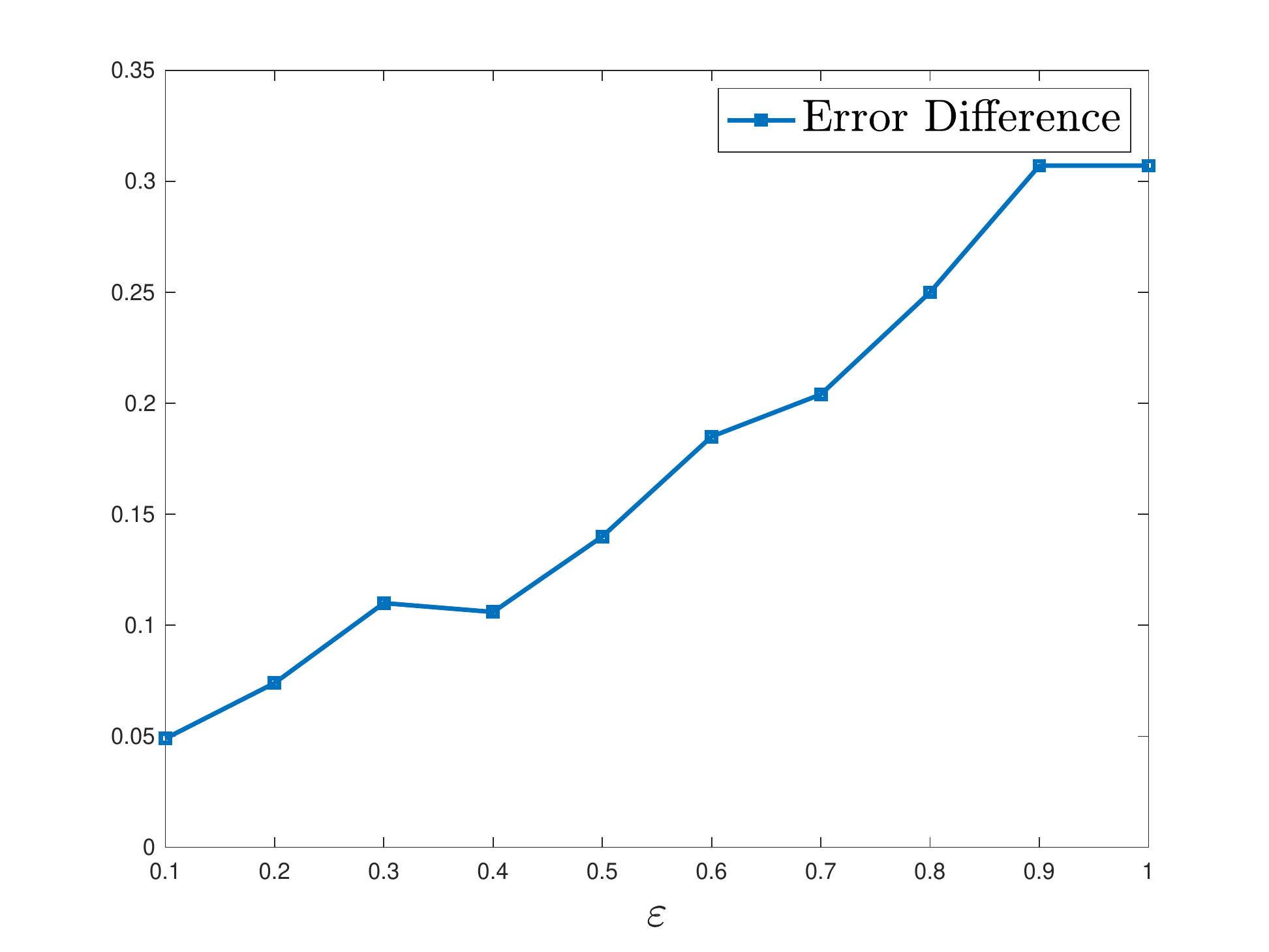} 
    \caption{Error Difference vs $\varepsilon$: we use gaussian synthetic data with sparsity structure, where we create $100$ dimensional Gaussian data for both original domain and shifted domain, and only the first $10$ coordinates matters (mean difference of positive and negative distribution of data from each domain is non-zero for only the first $10$ coordinates). The Error Difference = Adversarial Error with Semi-supervised Learning Algorithm - Adversarial Error with Algorithm with Unknown Sparsity. We can see that error difference is positive for all the $\varepsilon$ we take, which implies our algorithm with unknown sparsity works better when there is some common sparsity structure. }
\end{figure}

\section{Experimental Implementation on Real Data}

We use the implementation from ~\cite{carmon2019unlabeled} that can be accessed from \url{https://github.com/yaircarmon/semisup-adv}

\subsection{Experimental setup}
We follow the implementation in ~\cite{carmon2019unlabeled} for our experiments:
\subsubsection{CIFAR10/CINIC-10}
\paragraph{Architecture} We use a Wide ResNet 28-10 ~\cite{zagoruyko2016wide} architecture.
\paragraph{Training hyperparemeters} We use a batch size of 256 with SGD optimizer (along with Nesterov momentum of $0.1$). We use cosine learning rate annealing~\cite{loshchilov2016sgdr} with initial rate of $0.01$ and no restarts. The weight decay parameter is set to $0.0005$. We define an epoch to be a pass over 50000 training points. For normal training we run $200$ epochs. For adversarial training and stability training we run $100$ and $400$ epochs respectively.
\paragraph{Data Augmentation} We do a $4$-pixel random cropping and a random horizontal flip.
\paragraph{Adversarial attacks} We use the recommended parameters in ~\cite{carmon2019unlabeled}. In test time, we run $40$ iterations of projected gradient descent with step-size of $0.01$ and do $5$ restarts.
\paragraph{Stability training} We et noise variance to $\sigma=0.25$. In test time, we set $N_0=100$ and $N=10000$ with $\alpha=0.001$.
\subsubsection{SVHN}
We use similar parameters to CIFAR-10/CINIC-10 except the following. 
\paragraph{architecture} We use a Wide ResNet 16-8.
\paragraph{Training hyper-parameters} The same as CIFAR-10/CINIC-10 except we use batch size of $128$ and run $98k$ gradient steps for all models.
\paragraph{Data Augmentation} We do not perfom any augmentation.

\subsection{Cheap-10 dataset creation pipeline}

To create the Cheap-10 dataset, for each CIFAR-10 class, we create $50$ related keywords to search for on Bing image search engine. Using an existing image downloding API implementation~\footnote{\url{https://github.com/hardikvasa/google-images-download}}, we were able to download $\sim1000$ images for each key-word search.
%We use Bing image search enginer 
CIFAR-10 dataset is made of 10 classes. For animal classes (bird, cat, deer, dog, frog, horse), our keyboards were made of names of different breeds and different colors or adjectives known to accompany the specific animal. For instance, Parasitic Jaeger, Scottish Fold cat, Pygmy Brocket Deer, Spinone Italiano Dog, Northern Leopard Frog, and Belgian Horse. For other classes (airplane, automobile, ship, truck), we search for different brands or classes. For example, Lockheed Martin F-22 Raptor, Renault automobile, Tanker ship, and Citroën truck. We then downsize images to the original CIFAR-10 size of 32x32. We show example images of the dataset compared to CIFAR-10 images in Fig.~\ref{fig:cheap10}.

%{\red need to polish this}
\begin{figure}[H]
    \centering
    \includegraphics[width=0.8\linewidth]{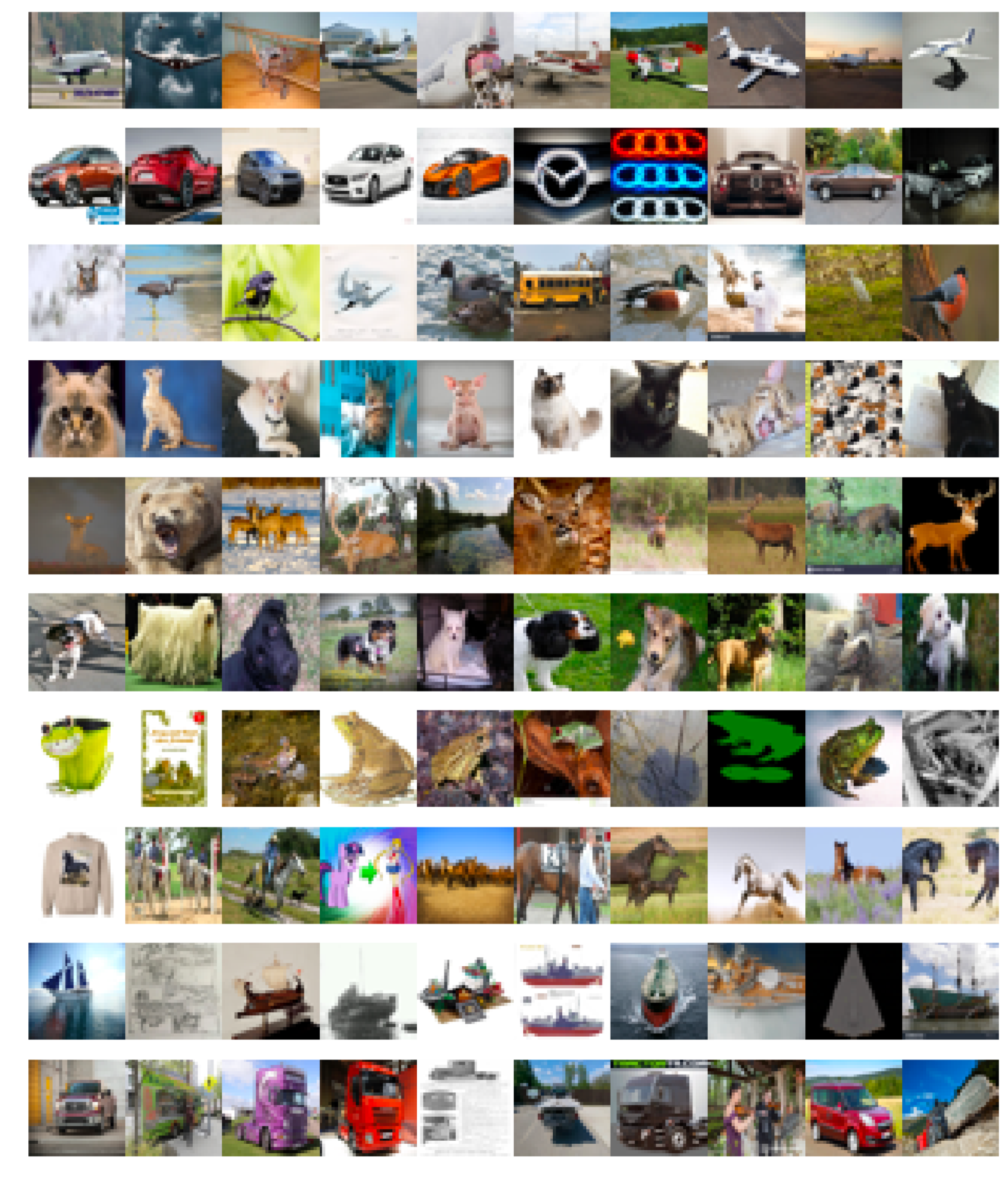} 
    \caption{\textbf{Cheap-10 examples} Each row shows 10 examples of Cheap-10 dataset. \label{fig:cheap10}}
\end{figure}

\end{document}